\documentclass[]{ipart}

\usepackage{amsthm,amsmath,amsfonts}
\usepackage{graphicx}

\pubyear{2014}
\volume{0}
\issue{0}
\firstpage{1}
\lastpage{1}

\startlocaldefs
\endlocaldefs

\usepackage{tensor}
\newtheorem{thm}{Theorem}[section]
\newtheorem{prop}[thm]{Proposition}
\newtheorem{cor}[thm]{Corollary}
\newtheorem{lem}[thm]{Lemma}
\newtheorem{defn}[thm]{Definition}

\usepackage{subfigure}
\usepackage{url}

\newenvironment{example}[1][Example]{\begin{trivlist}
  \item[\hskip \labelsep {\bfseries #1}]}{ $\lozenge$ \end{trivlist}}

\DeclareMathOperator{\Diff}{Diff}
\DeclareMathOperator{\GL}{GL}
\DeclareMathOperator{\ad}{ad}
\DeclareMathOperator{\Fr}{Fr}
\DeclareMathOperator{\evol}{evol}

\newcommand{\pder}[2]{\ensuremath{\frac{ \partial #1}{\partial #2}}}

\def \Jet {\mathcal{J}}
\newcommand{\R}{\ensuremath{\mathbb{R}}}

\begin{document}

\begin{frontmatter}

\title{Higher-order Spatial Accuracy in Diffeomorphic Image Registration}
\runtitle{Higher-order Spatial Accuracy in Diff. Image Reg.}


\author{\fnms{Henry O.} \snm{Jacobs}\ead[label=e1]{h.jacobs@imperial.ac.uk}}
\address{Department of Mathematics\\ Imperial College\\ London SW7 2AZ, UK\\ \printead{e1}}
\and
\author{\fnms{Stefan} \snm{Sommer}\ead[label=e2]{sommer@di.ku.dk}}
\address{Department of Computer Science\\ University of Copenhagen\\
  Universitetsparken 1, DK-2100\\ Copenhagen E, Denmark\\ \printead{e2}}

\runauthor{Henry O. Jacobs and Stefan Sommer}

\begin{abstract}
  We discretize a cost functional for image registration
  problems by deriving Taylor expansions for the matching term. Minima of the
  discretized cost functionals can be computed with no spatial discretization error,
  and the optimal solutions are equivalent to minimal energy curves in the
  space of $k$-jets. We show that the solutions convergence to 
  optimal solutions of the original cost functional as the number of particles
  increases with a convergence rate of $O(h^{d+k})$
  where $h$ is a resolution parameter. The effect of this approach over traditional
  particle methods is illustrated on synthetic examples and real images.
\end{abstract}

\end{frontmatter}

\section{Introduction}
The goal of image registration is to place differing images
of the same object (e.g. MRI scans)
into a shared coordinate system so that they may be compared.
One common means of doing this is to deform one image until
it matches the other.
Typical numerical schemes for implementing this task are
particle methods, where particles are used as a finite
dimensional representation of a diffeomorphism.
If the particles are initialized on a regular grid of resolution $h$,
then the solutions can be $O(h^d)$ accurate at best where $d$ is the
dimension of the image domain.
Improving this order of accuracy is non-trivial because traditional
higher-order numerical schemes are designed on fixed meshes
(e.g. higher order finite differences).

In this paper, we seek to improve this order of accuracy by considering
a more sophisticated class of particles.  We will find that by equipping
the particles with jet-data, one can achieve registrations with
higher orders of accuracy.
One impact of the use of higher-order particles is that the improved
accuracy per particle permits the use of fewer particles for a desired total accuracy.
For sufficiently smooth initial data, this implies the storage requirements are improved as well.

\subsection{Organization of the paper}
We will introduce the higher-order accurate image registration framework through
the following steps:
\begin{enumerate}
	\item We will introduce the hierarchy of jet-particles.
	\item We will pose an image registration problem as an optimal control problem on an infinite dimensional space.
        \item We will pose a sequence of deformed problems which are easier to solve.
        \item We will reduce the deformed optimization problems to optimization
          problems involving computation of finite dimensional ODEs (i.e. an infinite dimensional reduction).
	\item We will find necessary conditions for sequences of computed solutions to the deformed problems to converge to the solution of the original problem
      at a rate $O(h^{d+k})$, where $k \geq 0$ depends on the order of the jet-particles used.
\end{enumerate}
Finally, we will display the results of numerical experiments comparing the use of zeroth, first, and second order jet-particles.

\section{Previous work}
In this section, we attempt a brief overview of the large deformation diffeomorphic metric mapping (LDDMM)
framework from its origins in the 1990s, to its recent marriage with geometric mechanics (2000s-present).
\subsection{Matching with LDDMM}
The notion of seeking deformations for the sake of image registration goes back
a long way, see \cite{sotiras_deformable_2013,Younes2010} and references therein.
One of the first attempts was to consider diffeomorphisms of the form $\varphi(x) = x + f(x)$ for some map $f:\R^d \to \R^d$. The map $f$ is often denoted a displacement field.
When $f$ is ``small'', $\varphi$ is a diffeomorphism, but this can fail when $f$ is ``large'' \cite[Chapter 7]{Younes2010}. Many algorithms outside the LDDMM context employ a small deformation approach with a displacement field $f$. The displacement can be represented for example with B-spline basis functions. Particle methods, as considered in this paper, can loosely be interpreted as large deformation equivalents to representing the displacement $f$ with finite linear combinations of basis functions. In particular, the kernel $K$ can be thought of as taking the role of e.g. B-spline basis functions in small deformation approaches.

The breakdown for large $f$ is a result of the fact that the space of diffeomorphisms is a nonlinear space.
One of the early obstacles in diffeomorphic image registration entailed dealing with this nonlinearity.
A key insight in getting a handle on the nonlinearity of the diffeomorphisms was to consider the linear space of vector fields.
Given a time-dependent vector field $v(t)$, one can integrate it to obtain a diffeomorphism $\varphi_t$, which is called the \emph{flow of $v$} \cite{christensen_deformable_1996}.
This insight was used to obtain diffeomorphisms for imaging applications by posing an optimal control
problem on the space of vector-fields, and then integrating the flow of the optimal vector field to obtain a diffeomorphism.
The well-posedness of this approach was studied in \cite{trouve_infinite_1995,DupuisGrenanderMiller1998},
 where the cost functional (i.e. the norm) was identified as a fundamental choice in ensuring well-posedness
 and controlling properties of the resulting diffeomorphisms.
 A particle method based upon \cite{DupuisGrenanderMiller1998} was implemented for the purpose of medical imaging in \cite{JoshiMiller2000}.
 The completeness of the Euler-Lagrange equations in \cite{DupuisGrenanderMiller1998} was studied thoroughly in \cite{TrouveYounes2005},
 where the image data was allowed to be of a fairly general type (i.e. any entity upon which diffeomorphisms act smoothly).
 The analytic safe-guards provided by \cite{DupuisGrenanderMiller1998} and \cite{TrouveYounes2005} where then excercised in \cite{Beg2005},
 where a number of examples were numerically investigated.

\subsection{Connections with geometric mechanics}
 Following these early investigations, connections with geometric mechanics began to form.
 The cost functional chosen in \cite{JoshiMiller2000} was the $H^1$-norm of the vector-fields.
  Coincidentally, this is the cost functional of the $n$-dimensional Camassa-Holm equation (see \cite{HolmMarsden2005} and references therein).
  In 1-dimension, the particle solutions in \cite{JoshiMiller2000} are identical to the peakon solutions discovered in \cite{CamassaHolm1993},
  and the numerical scheme reduces to that of \cite{HoldenRaynaud2006}.
  The convergence of \cite{JoshiMiller2000} was proven using geometric techniques in \cite{HoldenRaynaud2006} in the one-dimensional case.
  The same proof was used in \cite{ChertockDuToitMarsden2012} for arbitrary dimensions.
  As images appear as advected quantities, the use of momentum maps became a useful conceptual technique for geometers 
  to understand the numerical scheme of \cite{Beg2005}.
  The identification of numerous mathematical terms in \cite{Beg2005} as momentum maps was performed in \cite{BruverisHolmRatiu2011}.
  
  \subsection{Jet particles}
  \label{sec:jetparticles}
  The particle method implemented in \cite{JoshiMiller2000} allowed only for deformations
  that acted as ``local translations'' (see Figure \ref{fig:shots}(a)).
  Motivated by a desire to create more general deformations
  \cite{Sommer2013} introduced a hierarchy of particles which advect jet-data.
  We call the particles \emph{jet-particles} in this paper.
  The first order jet-particles modify the Jacobian matrix at the particle locations
  and allow for ``locally linear'' transformations such as local scalings and local rotations
  (see Figure \ref{fig:shots}(b-e)).
  Second order jet-particles allow for deformations which are ``locally quadratic'' (i.e. transformations 
  with nontrivial Hessians (see Figure \ref{fig:shots}(f-h)).
  The geometric and hierarchal structure of \cite{Sommer2013} was investigated in \cite{Jacobs_MFCA_2013}
  where the Lie groupoid structure of jet-particles was linked to the Lie group structure
  of the diffeomorphism group, thus making the case for jet-particles as multi-scale representations
  of diffeomorphisms.
  Independently, an incompressible version of this idea was invented for the purpose of
  incompressible fluid modelling in \cite{JacobsRatiuDesbrun2013}.
  Solutions to this fluid model were numerically computed in \cite{CotterHolmJacobsMeier2014}
  based upon the regularized Euler fluid equations developed in \cite{MumfordMichor2013}
  and expressions for matrix-valued reproducing kernels derived in \cite{MicheliGlaunes2014}.
  The final section of \cite{CotterHolmJacobsMeier2014} provides formulas 
  which illustrate how jet-particles
  in the $k$th level of the hierarchy yield deformations which are
  approximated by particles in the $(k-1)$th level of the hierarchy.
  The approximation being accurate to an order $O(h^{k})$ where $h > 0$ is
  some measure of particle spacing.
  This approximation is more or less equivalent
  to the approximation of a partial differential operator by a finite differences,
  and it will serve as one of the main tools used in this paper in producing
  higher-order accurate numerical schemes.
  
\begin{figure}[t]
  \begin{center}
      \subfigure[translation]{
        \includegraphics[width=.20\columnwidth,trim=80 50 80 50,clip]{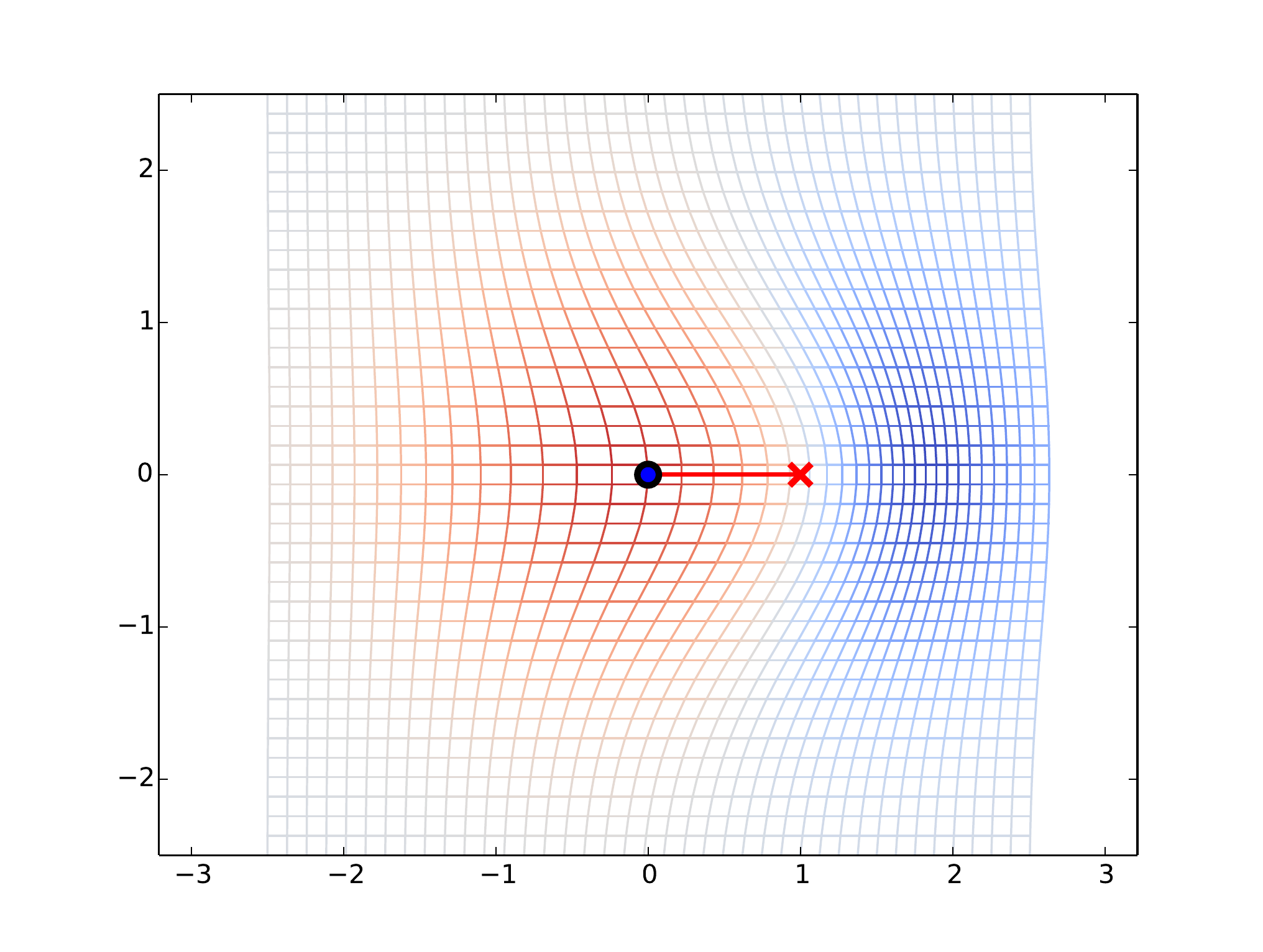}
      }
      \subfigure[expansion]{
        \includegraphics[width=.20\columnwidth,trim=80 50 80 50,clip]{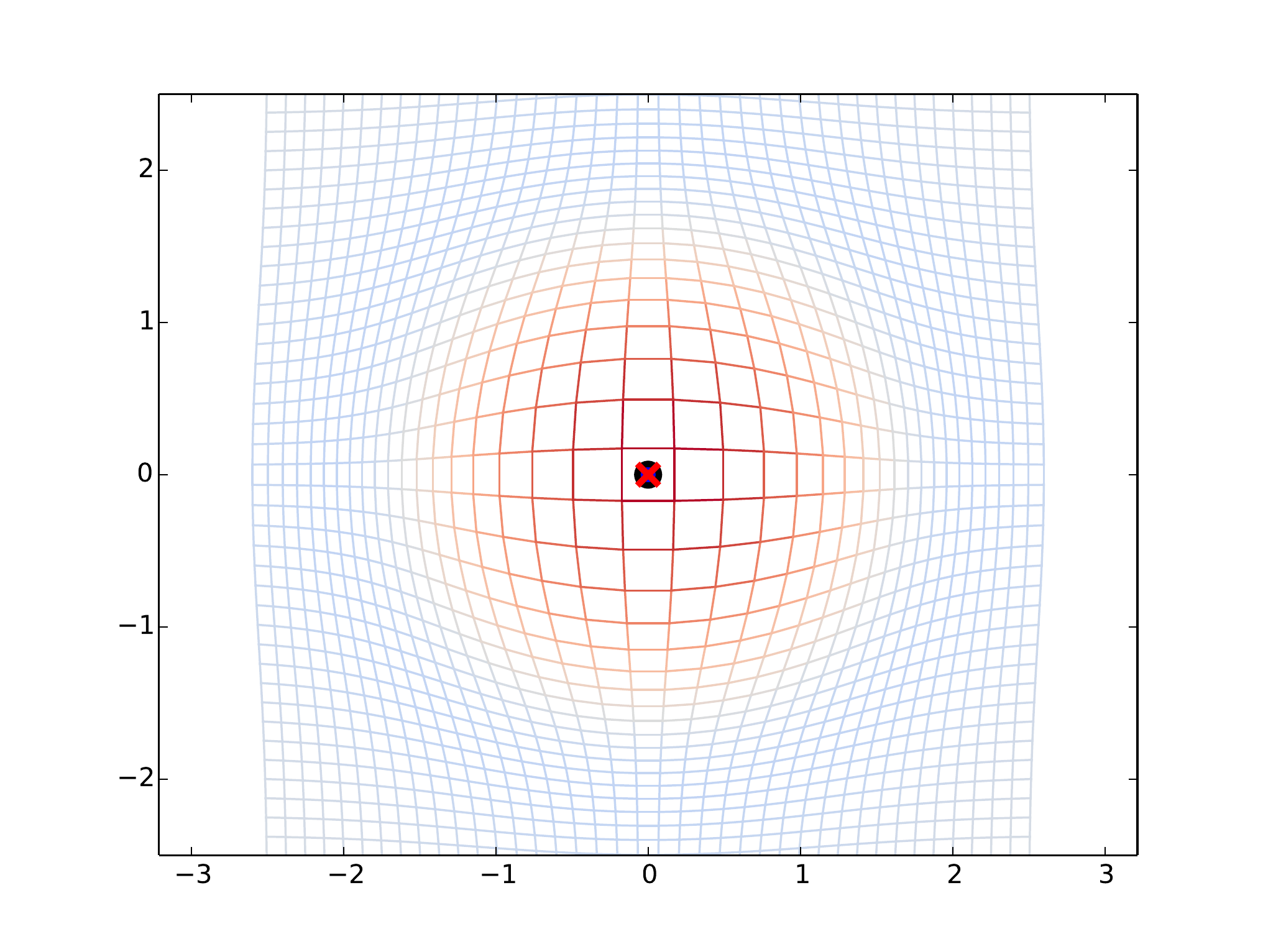}
      }
      \subfigure[rotation]{
        \includegraphics[width=.20\columnwidth,trim=80 50 80 50,clip]{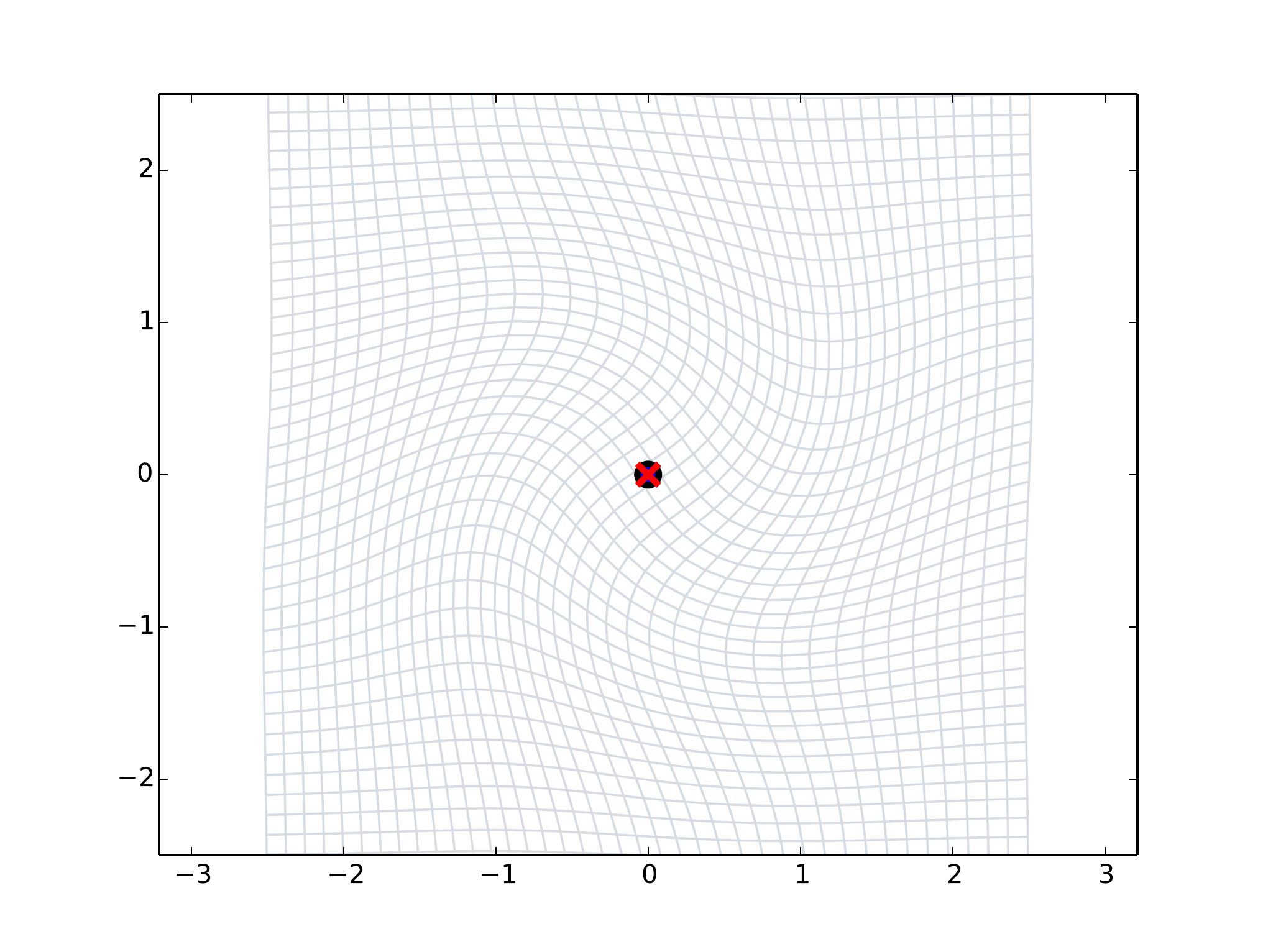}
      }
      \subfigure[stretch]{
        \includegraphics[width=.20\columnwidth,trim=80 50 80 50,clip]{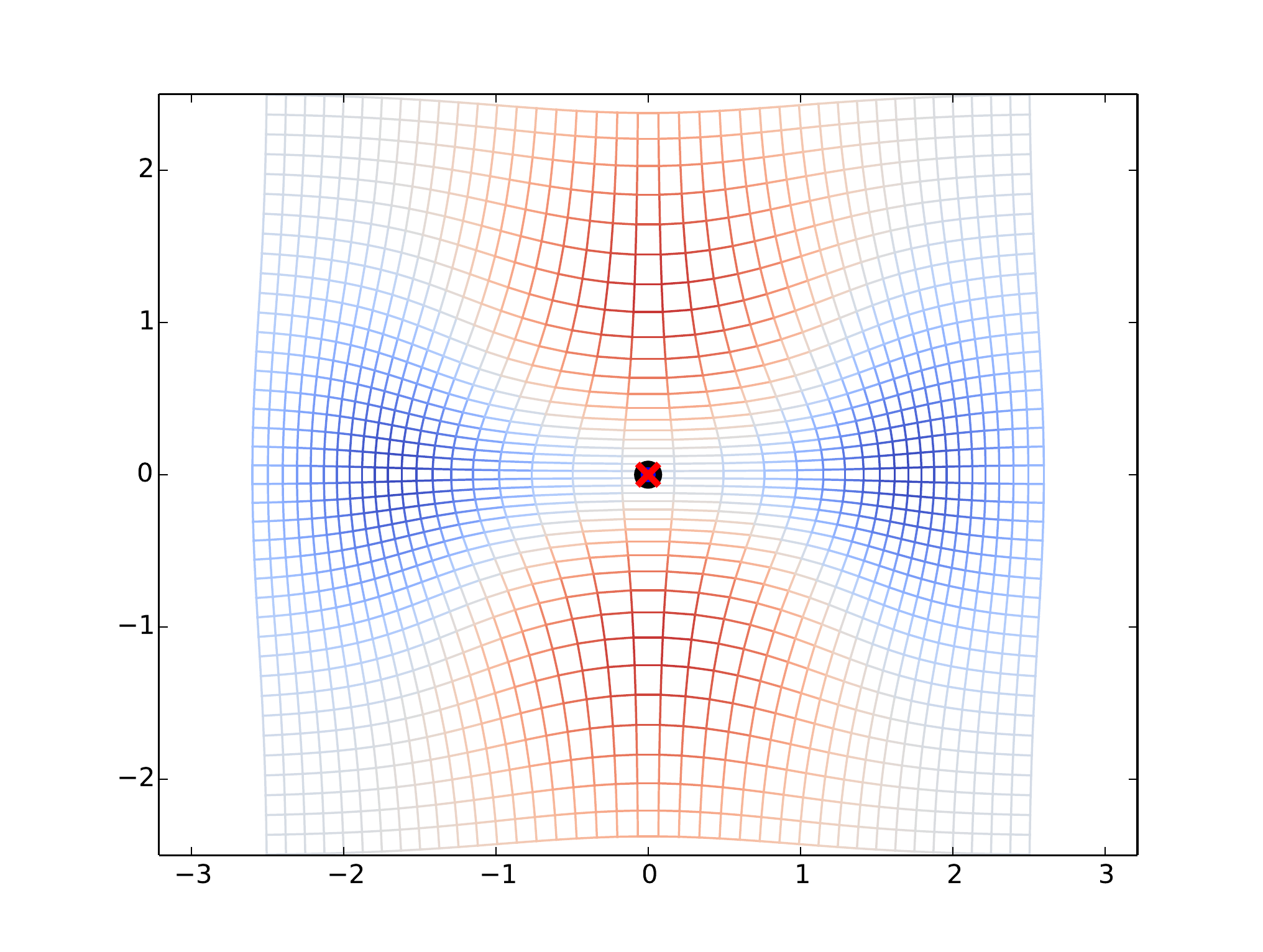}
      }
      \newline
      \subfigure[shear]{
        \includegraphics[width=.20\columnwidth,trim=80 50 80 50,clip]{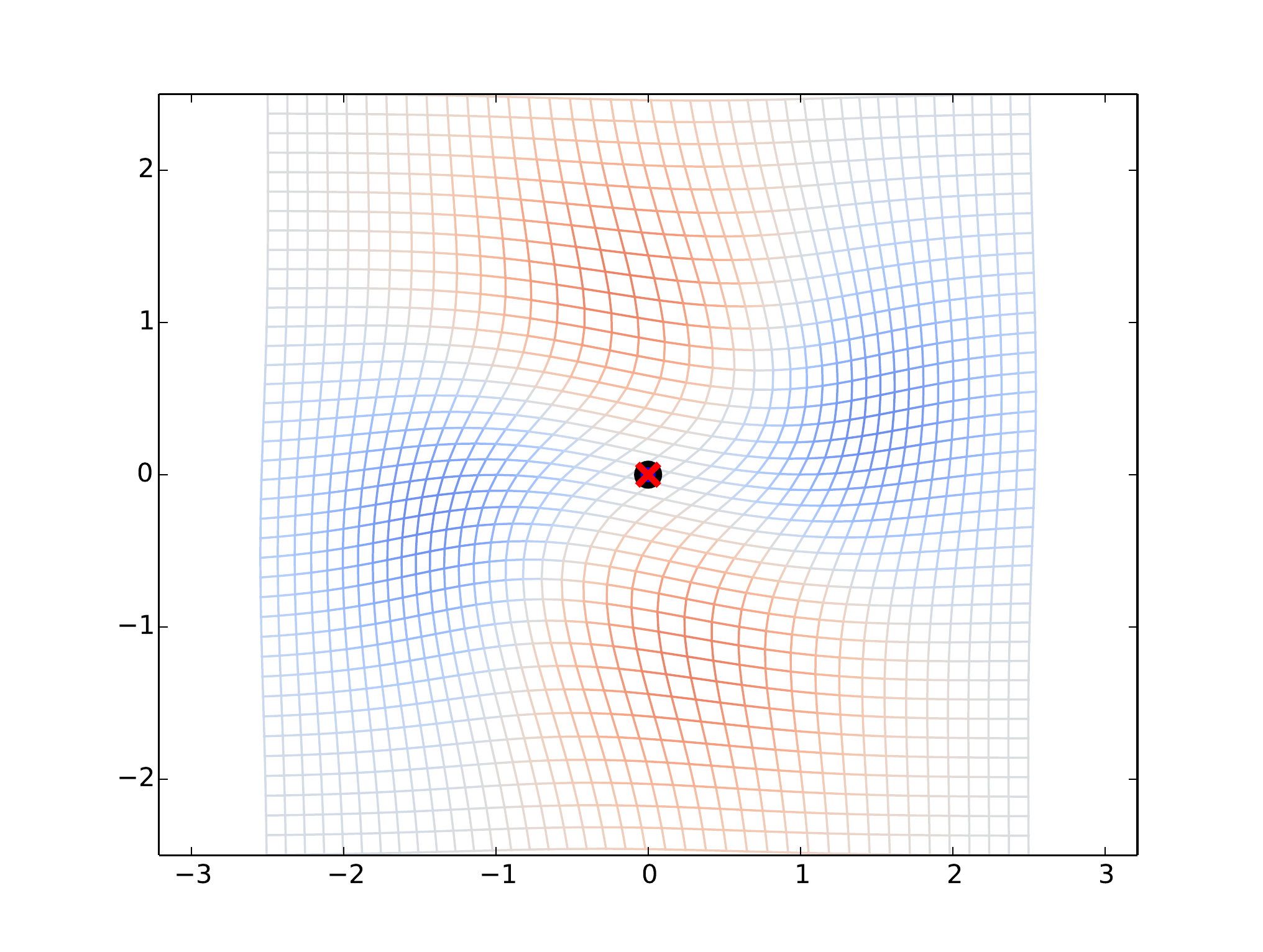}
      }
      \subfigure[2nd order]{
        \includegraphics[width=.20\columnwidth,trim=80 50 80 50,clip]{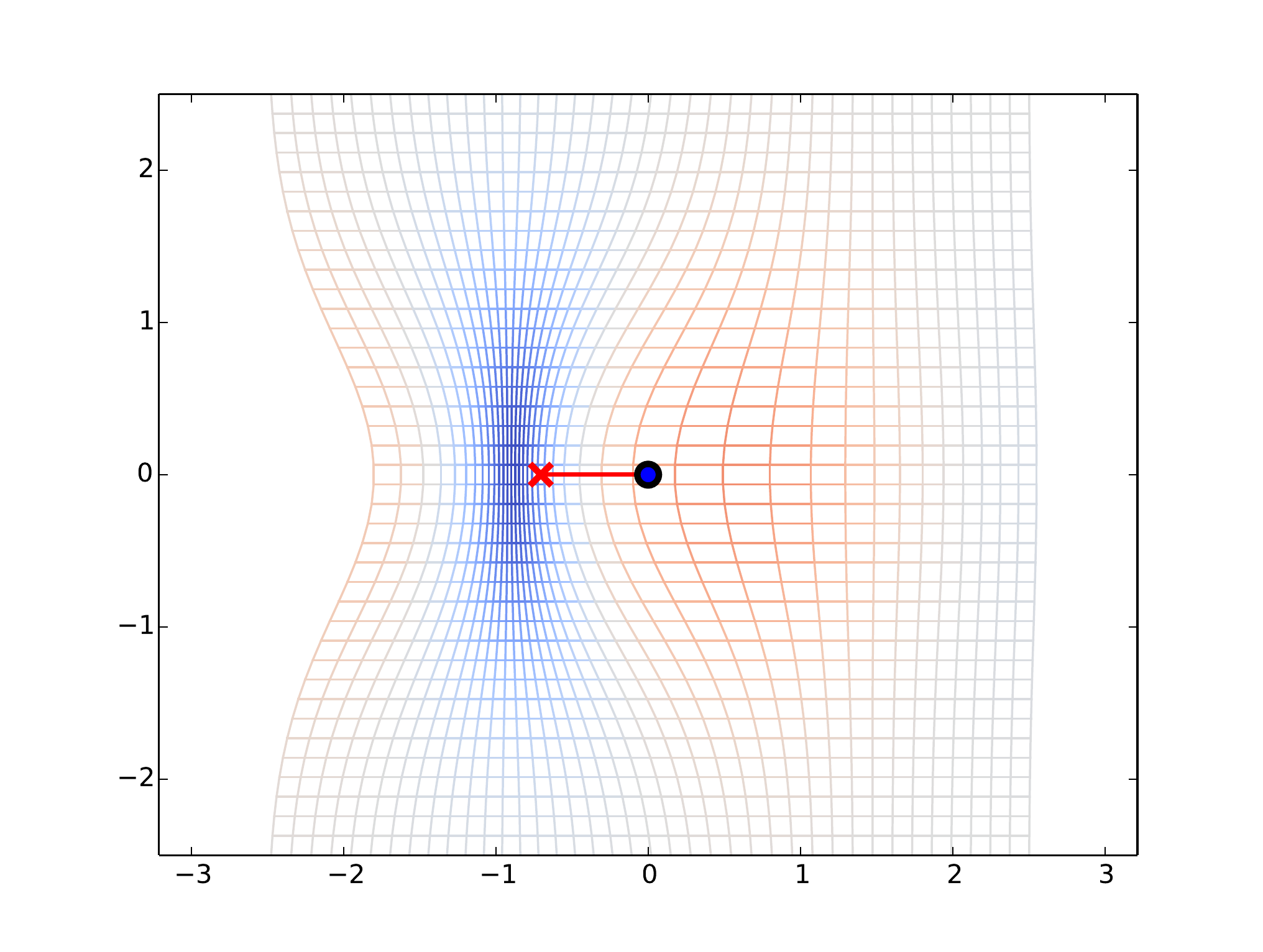}
      }
      \subfigure[2nd order]{
        \includegraphics[width=.20\columnwidth,trim=80 50 80 50,clip]{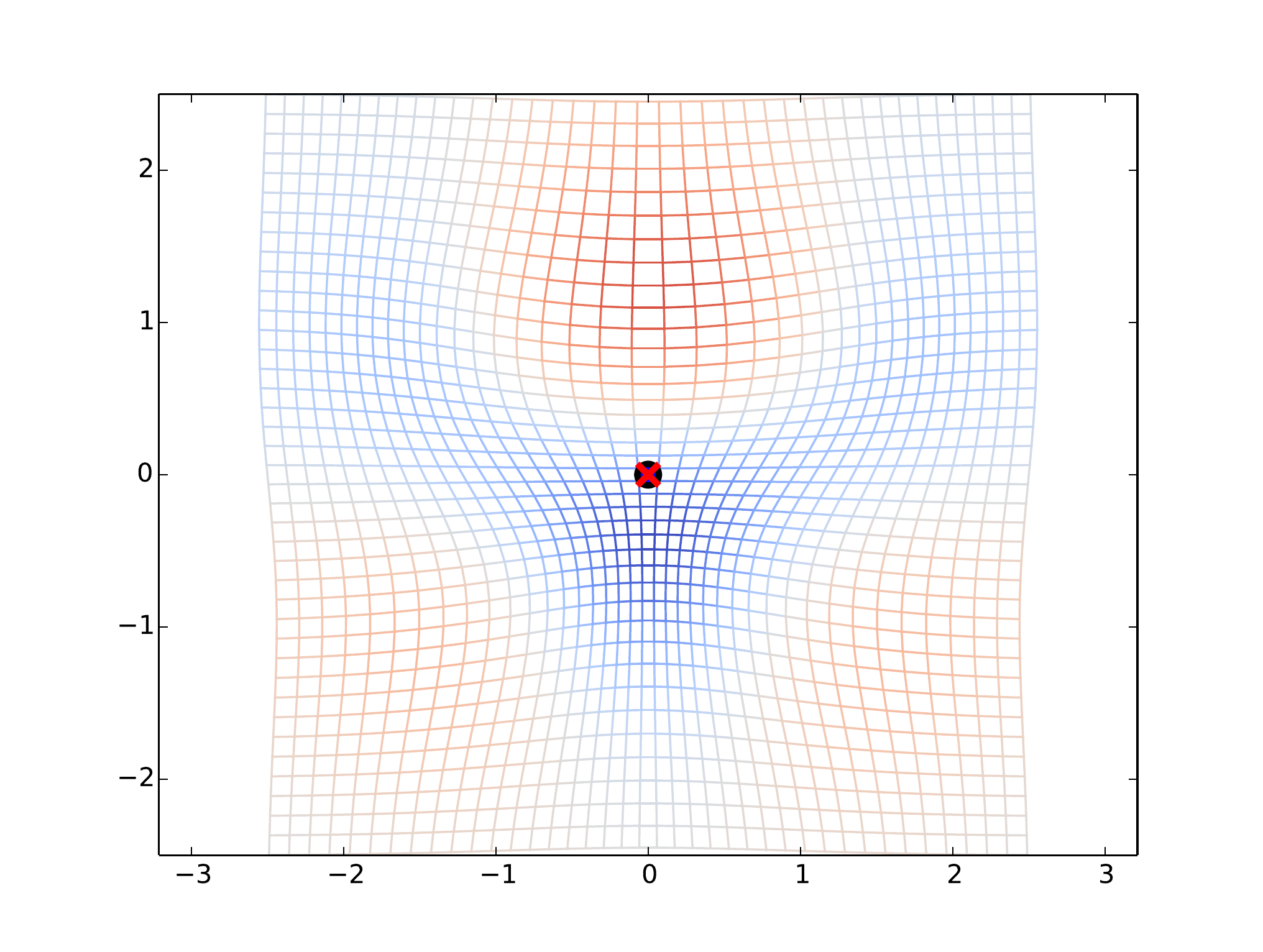}
      }
      \subfigure[2nd order]{
        \includegraphics[width=.20\columnwidth,trim=80 50 80 50,clip]{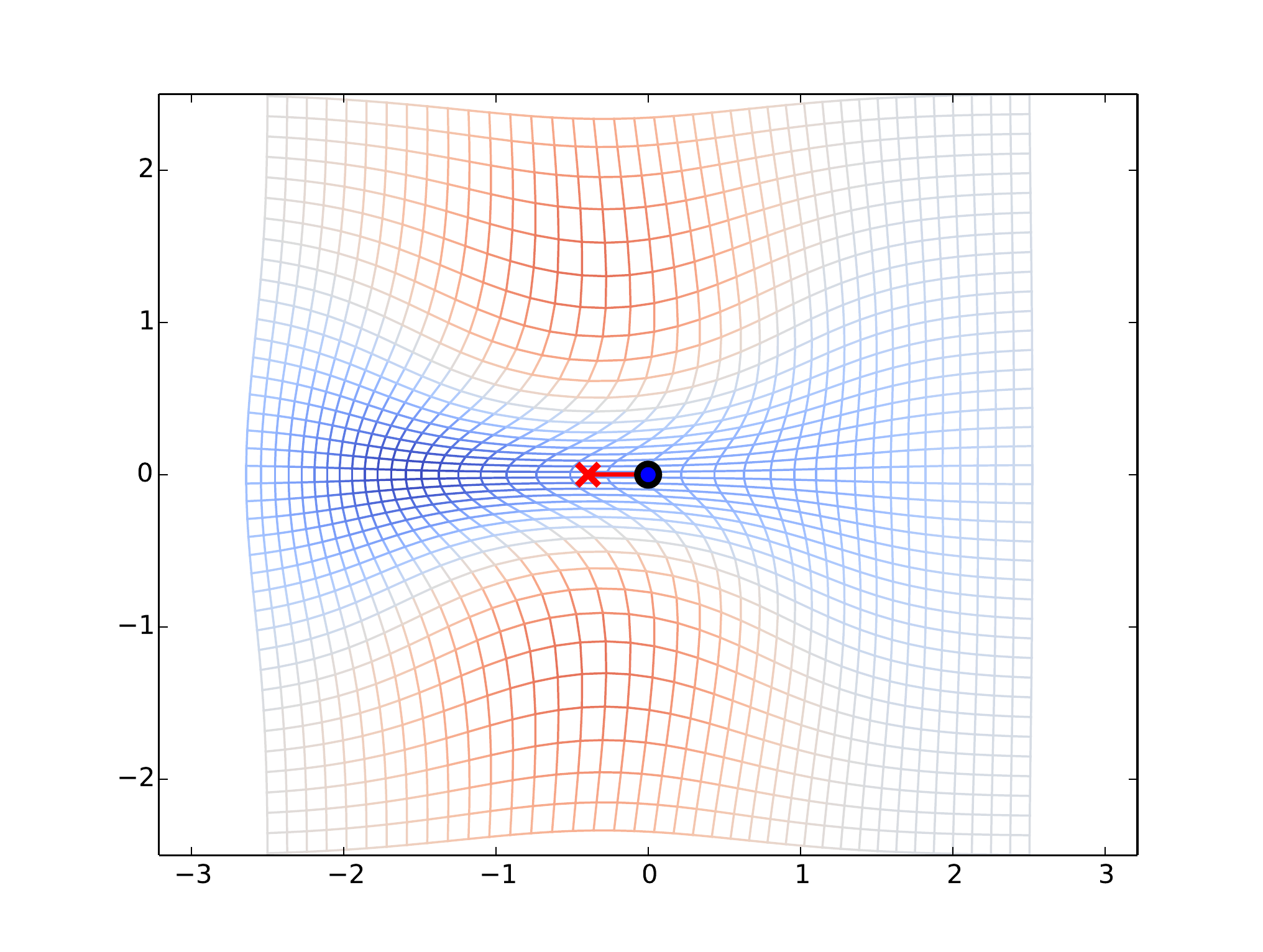}
      }
  \end{center}
  \caption{Deformations of initially square grids. (a) zeroth
    order, (b-e) first order, (f-h) second order. A single jet-particle is located at the
    blue dots before moving with the flows to the red crosses. Grids are colored by
    log-Jacobian determinant.
    }

  \label{fig:shots}
\end{figure}

\section{LDDMM}
Let $M$ be a manifold and let $V \subset \mathfrak{X}( M)$ be a subspace of the vector-fields
on $M$ equipped with an inner-product $\langle \cdot , \cdot \rangle_V : V \times V \to \mathbb{R}$.
Let $G_V \subset \Diff( \mathbb{R}^n)$ be the corresponding topological Lie
group to which $V$ integrates \cite[Chapter 8]{Younes2010}.
To do image registration, we try to assemble a ``small'' diffeomorphism by minimizing a cost function on
the space of curve in $G_V$.
The standard cost function takes a time-dependent diffeomorphism, $\varphi_t$, and
outputs a real number.
Mathematically, the cost function is often taken to be a map $E_{G_V}: C^1( [0,1] : G_V) \to \mathbb{R}$ given
by
\begin{align*}
  E_{G_V}[ \varphi( \cdot ) ] := \frac{1}{2}\int_{0}^{1} \ell(v(t)) dt + F( \varphi_1) ,
\end{align*}
where $v(t) \in V$ is the Eulerian velocity field $v(t,x) = \partial_t \varphi_t( \varphi_t^{-1}(x))$
and $\ell$ is a ``control-cost''.
Explicitly, $\varphi_t \in G_V$ is obtained from $v(t) \in V$ via the initial value problem
\begin{align}
  \begin{cases}
    \frac{d}{dt} \varphi_t = v(t) \circ \varphi_t \\
    \varphi_0 = id
  \end{cases}. \label{eq:flow}
\end{align}
One then obtains extremizers of $E_{G_V}$ by solving the Euler-Lagrange equations on $G_V$.
However, $G_V$ is a non-commutative group, and can be very difficult to work with.
It is typical to express $E_{G_V}$ as a cost function on the vector-space $V$
and incorporate \eqref{eq:flow} as a constraint.
This means optimizing a cost function $E: C^1( [0,1] , V) \to \mathbb{R}$
with respect to constrained variations.
Any extremizer, $v(\cdot)$, of $E$ must necessarily satisfy a symmetry reduced form of the Euler-Lagrange equations,
known as the Euler-Poincar\'e equation.
In essence, the Euler-Poincar\'e equations are nothing but the Euler-Lagrange equations pulled to the space $V$.
For a generic $\ell$, the Euler-Poincar\'e equations take the form
\begin{align}
	\frac{d}{dt} \left( \frac{ \delta \ell}{\delta v} \right) + \ad^*_v \left( \frac{\delta \ell}{\delta v} \right) = 0. \label{eq:EP}
\end{align}
We suggest \cite{MandS} for further information on the Euler-Poincar\'e equations.
Equation \eqref{eq:EP} is an evolution equation, which allows us to search over the space of initial conditions 
(i.e. $V$) in place of optimizing over space of curves (i.e. $C^1([0,1];V)$).
Explicitly, this is done by considering the map $\evol_{EP}: V \to C^1([0,1];V)$
which sends each $v_0 \in V$ to the curve $v(\cdot) \in C^1([0,1];V)$ obtained by integrating \eqref{eq:EP} with initial condition $v_0$.
We can then pre-compose $E$ with $\evol_{EP}$ to produce the function
\begin{align}
	e := E \circ \evol_{EP} : V \to \mathbb{R}. \label{eq:vector_cost_function}
\end{align}
The initial condition $v^* \in V$ minimizes $e$ if and only the solution $v( \cdot) = \evol_{EP}(v^*)$ of \eqref{eq:EP}
minimizes $E$.
Generally, solutions to \eqref{eq:EP} are extremizers of $E$, and one must appeal to higher-order variations in order to obtain sufficient conditions
for an extremizer to be a minimizer.
However, we will not pursue these matters in this article.

\subsection{Overview of the problem and our solution}

Particle methods are typically used to approximate a diffeomorphism in the following way.
We usually compute all quantities with respect to an initial condition where all the particles
lie on a grid/mesh and prove convergence as the mesh width, $h$, tends to $0$.
However, it would be nice to have an order of accuracy as well.

\begin{center}
\emph{PROBLEM:  Can we solve for a minimizer of $E$ with a convergence rate of $O(h^p)$ for some $p \in \mathbb{N}$?}
\end{center}

Our strategy for tackling this problem is to approximate $E$ with a sequence of $O(h^p)$-accurate curve energies $E_h$ for which we can compute the minimizers exactly up to time discretization (i.e. the computed solutions have no spatial discretization error). More specifically, the mesh size will determine a continuous sequence of subgroups $G_h \subset G$. We will approximate the matching functional $F$, with a $G_h$-invariant functional $F_h:G_V \to \mathbb{R}$ such that for a fixed $\varphi \in G_V$
\begin{align*}
  F(\varphi) - F_h(\varphi) = O(h^p)
\end{align*}
for some $p \in \mathbb{N}$.
We find the curve energies to be $O(h^p)$ accurate as well,
and this accuracy will transfer to the solutions
for a sufficiently wide range of scenarios.

\section{Reduction theory}
In this section, we review subgroup reduction of a class of
optimization problems using Clebsch variables.
In the Hamiltonian context, Clebsch variables are also called symplectic
variables, and constitute a Poisson map $\psi: T^*\mathbb{R}^n \to P$.
This is useful when $2n < \dim(P)$, since solutions to certain Hamiltonian equations on $P$ can be derived by solving Hamiltonian equations on $T^*\mathbb{R}^n$ first \cite{MarsdenWeinstein1983,Weinstein1983}.
The Lagrangian version of this idea was further developed in the context of equations with hydrodynamic background in \cite{HolmMarsden2005}.
It is this later perspective which we shall take in this paper, since problem
setup is stated in Lagrangian form. A more thorough overview of the role of reduction by symmetry in LDDMM can be found in \cite{sommer_reduction_2015}.

Let $G$ be a Lie group and $G_s \subset G$
be a Lie subgroup with Lie algebras $\mathfrak{g}$ and $\mathfrak{g}_s$ respectively.
We will denote the homogenous space
 of right cosets by $Q = G / G_s$, and we will denote
 the corresponding principal bundle projection by
$\pi : G \to Q$.
Note that $G$ naturally acts on $Q$ through the formula $g \cdot \pi(\tilde{g}) = \pi( g \cdot \tilde{g})$.
Given this action, the corresponding (left) momentum map,
$J : T^*Q \to \mathfrak{g}^*$, is defined by the condition
\begin{align*}
  \langle J( q,p) , \xi \rangle = \langle p , \xi \cdot q \rangle\quad, \qquad \forall \xi \in \mathfrak{g}.
\end{align*}

Let $L : TG \to \mathbb{R}$ be the Lagrangian and let $F:G \to
\mathbb{R}$.
We wish to minimize the curve energy or ``action''
\begin{align}
  E[ g(\cdot) ] = \int_0^1{ L(g(t),\dot{g}(t)) dt} + F(g(1)) \label{eq:energy_1}
\end{align}
over the space of curves $g( t ) \in G$ on the interval $[0,1]$ with
$g(0) = id$.  That is to say $E : C^1_{id}( [0,1] ; G) \to \mathbb{R}$
where $C^1_{id}([0,1];G)$ denotes the space of $C^1$ curves in $G$
originating from the identity.
Extremization of $E$ means taking a variation in $C^1_{id}( [0,1] ; G)$, which
is a variation of a curve with a fixed end-point at
$t=0$ but \emph{not} at $t=1$.  It is simple to show that any
solution must satisfy the boundary value problem
\begin{align}
  \begin{cases}
    \frac{d}{dt} \left( \pder{L}{\dot{g}} \right) - \pder{L}{g} = 0 \\
    g(0) = id \quad , \quad \left. \pder{L}{\dot{g}} \right|_{t=1} + dF(
    g(1) ) = 0.
  \end{cases} \label{eq:EL}
\end{align}
If the dimension of $G$ is large, integrating this equation
can be troublesome.
However, in the presence of a $G_s$-symmetry a reduction can be applied
to reduce the problem to a boundary value problem on $Q$.

Throughout this section we will assume that $F$ is $G_s$ invariant.
As a result there exists a function $f: Q \to \mathbb{R}$ defined by the condition
\begin{align*}
	f( q) = F(g) \quad \forall q \in Q , g \in G \text{ such that } q = \pi(g).
\end{align*}
More succinctly, $f = F \circ \pi$.
We will also assume that $L(g,\dot{g})$ is $G$-invariant, and comes from a reduced Lagrangian
function $\ell: \mathfrak{g} \to \mathbb{R}$.
Finally, we will assume that the Legendre transformation,
$
  \frac{\delta \ell}{\delta \xi } : \mathfrak{g} \to \mathfrak{g}^*,
$
is invertible.
The reduced Hamiltonian $h : \mathfrak{g}^* \to
\mathbb{R}$ is then given by
\begin{align*}
  h( \mu ) = \left \langle \mu , \frac{\delta \ell}{\delta \xi}^{-1}(\mu)
  \right \rangle - \ell \left( \frac{ \delta \ell}{\delta \xi}^{-1}( \mu) \right).
\end{align*}

\begin{thm}[c.f. \cite{BruverisHolmRatiu2011}] \label{thm:reduction}
  Let $H := h \circ J : T^*Q \to \mathbb{R}$.
  If the curve $(q,p)(t) \in T^*Q$ satisfies
\begin{align}
  \begin{cases} \dot{q} = \pder{H}{p} \quad , \quad \dot{p} = - \pder{H}{q} \\
  q(0) = \pi( id) \quad , \quad p(1) + df(q(1)) = 0
  \end{cases} \label{eq:Hamilton}
\end{align}
  then the curve $g(t)$ obtained by integrating the initial value problem
    \begin{align*}
      \dot{g}(t) = \xi(t) \cdot g(t) \quad , \quad \xi = (\delta \ell/\delta \xi )^{-1}( J(q,p) ) \quad,\quad g(0) = id
    \end{align*}
    satisfies \eqref{eq:EL}. Moreover, all minimizers of \eqref{eq:energy_1} must
    be of this form.
\end{thm}
\begin{proof}
  We can replace $E$ with the
(equivalent) curve energy $E_2: C^1([0,1] ; \mathfrak{g}) \to \mathbb{R}$ given by
\begin{align}
  E_2[ \xi ] = \int_0^1{ \ell(\xi(t)) dt} + F(g(1)) \label{eq:energy_2}
\end{align}
where $g(1) \in G$ is implicitly obtained through the reconstruction equation
$\frac{dg}{dt} = \xi \cdot g$ which we view as a constraint.
Minimizers of $E_2$ are related to minimizers of $E$ through the reconstruction equation
as well.

We are now going to use the $G_s$ symmetry of \eqref{eq:energy_2} to reduce the
dimensionality of the problem.
The $G_s$ invariance of $F$ implies the existence of a function $f: Q
\to \mathbb{R}$ such that $F = f \circ \pi$.
Therefore, we may equivalently express $E_2$ as the energy functional
\begin{align}
  E_2[ \xi ] = \int_0^1{ \ell( \xi(t) ) dt} + f(q(1)) \label{eq:energy_3}
\end{align}
where $q(1)$ is obtained through the reconstruction equation
$\dot{q}(t) = \xi(t) \cdot q(t)$ with the initial condition $q(0) =
 \pi_s(id)$.
Again, the dynamic constraint $\dot{q} = \xi \cdot q$ makes this a constrained optimization problem.
We may take the dual of this constrained optimization problem by using Lagrange multipliers to get an equivalent unconstrained optimization problem \cite{BoydVandenberghe2004}.
In our case, the dual problem is that of extremizing the (unconstrained) curve energy $E_3 : C^1( [0,1] ; \mathfrak{g} \times T^*Q) \to \mathbb{R}$ given by
\begin{align*}
  E_3[ \xi , q , p ] = \int_0^1 \ell(\xi(t)) + \langle p(t) , \dot{q}(t) -
  \xi(t) \cdot q(t) \rangle dt + f(q(1)).
\end{align*}
Using the definition of $J$ we can re-write this as
\begin{align*}
  E_3[ \xi , q , p ] = \int_0^1 \ell(\xi) + \langle p , \dot{q}  \rangle - \langle J(q,p),
  \xi \rangle dt + f(q(1)).
\end{align*}
We find that stationarity with respect to arbitrary variations of $\xi$ implies
\begin{align}
  \frac{ \delta \ell}{\delta \xi} = J(q(t) , p(t)). \label{eq:mu}
\end{align}
  We may view \eqref{eq:mu} as a constraint
which defines $\xi$ in terms of the $q$'s and $p$'s.
Explicitly, \eqref{eq:mu} tell us
\begin{align*}
	\xi = \frac{\delta \ell}{\delta \xi}^{-1}( J(q , p)).
\end{align*}
We can substitute this into the previous curve energy to eliminate
the variable $\xi$ and express $E_3$ solely in terms of $p$ and $q$.
We thus obtain the curve energy
\begin{align*}
  E_4[ q , p ] &= \int_0^1 \ell \left( \frac{\delta \ell}{\delta \xi}^{-1}( J(q , p)) \right) + \langle p , \dot{q} \rangle -
  \left\langle J(q,p) ,  \frac{\delta \ell}{\delta \xi}^{-1}( J(q , p)) \right\rangle dt \\
  	&\quad+ f(q(1)).
\end{align*}
By observing
\begin{align*}
	H(q,p) = h(J(q,p)) = \left\langle J(q,p) , \frac{\delta \ell}{\delta \xi}^{-1}( J(q , p)) \right\rangle - \ell \left( \frac{\delta \ell}{\delta \xi}^{-1}( J(q , p)) \right), 
\end{align*}
we can write $E_4$ as
\begin{align}
	E_4[q,p] &= \int_0^1 \langle p , \dot{q} \rangle - H(q,p) dt + f(q(1)). \label{eq:reduced_curve_energy}
\end{align}
By taking arbitrary variations of $q$ and $p$, we find that extremization of $E_4$ yields the desired result.
\end{proof}

As a corollary, we find that extremizers of $E$ in \eqref{eq:energy_1}
can be derived by finding the extremizers of $E_4$ in \eqref{eq:reduced_curve_energy}.
As $E$ is a curve energy over $TG$ and $E_4$ is a curve energy over $T^*Q = T^*(G/G_s)$,
we can see the computational significance of this result most clearly when the dimension of $Q$ is small compared to $G$ (e.g. finite compared to infinite).
More specifically,  Theorem \ref{thm:reduction} allows us to minimize curve energies using the following
  gradient descent algorithm.\\
\begin{center}
\fbox{
    \parbox{0.9\textwidth}{
    	{\bf Algorithm for general Lie groups}  
        \begin{enumerate}
          \item Solve for $(q(t),p(t)) \in T^{\ast}Q$ in \eqref{eq:Hamilton}.
          \item Set $\xi(t) = (\delta \ell / \delta \xi)^{-1} \cdot J( q(t),p(t))$
          \item If necessary, obtain $g(t) \in G$ as a solution to the initial value problem, $\dot{g} = \xi \cdot g$ , $g(0) = id$.
          \item Evaluate cost function $E_4$, and backward compute the adjoint equations \cite{Sontag1998} to compute the gradient of the cost
          	function with respect to a new initial condition.
	 \item If the gradient is below some tolerance, $\epsilon$, then stop.
	 	Otherwise use the gradient to create a new initial condition and return to step 1.
        \end{enumerate}
    }
}
\end{center}

We say ``if neccessary'' in step 3 because computation of $g(t)$ is not needed in the context of image registration.
We will find that only $q(t)$ is needed.
In any case, if $g(t)$ were to be computed, the resulting curve would minimize the original curve energy $E$ given
in equation \eqref{eq:energy_1}, and all the minimizers of the original problem are obtained in this way.
Again, the advantage of this method is that the bulk of the computation
is performed on the lower dimensional space $T^*Q$ rather than $TG$.

In the next sections we will consider the case where $G$ is a diffeomorphism group, and $G_s$ is a subgroup such that $Q = G/G_s$ is the (finite-dimensional) space of jet-particles.

\section{Jets as Homogenous spaces}
In order to invoke the findings of the previous section, we must find a way to characterize the space of jet-particles
as a homogenous space (i.e. a group modulo a subgroup).
This is the content of Proposition \ref{prop:jets_as_quotient}, the main result in this section.

Let $\Lambda \subset M$ be a finite set of distinct points in $M$. If $f$ is any $k$-differentiable map from a neighborhood of $\Lambda$, the $k$-jet of $f$ is denoted $\Jet^{(k)}_{\Lambda}( f )$. In coordinates,
$\mathcal{J}^{(k)}_\Lambda(f)$ is represented by the coefficients of the $k$th order Taylor expansions of $f$ about each of the points in $\Lambda$.  We call $\Jet^{(k)}_{\Lambda}$ the ``$k$th order jet functor about $\Lambda$''.  This is indeed a functor, and can be applied to any $k$-differentiable map from subsets of $M$ which contain $\Lambda$, including real valued functions, diffeomorphisms, and curves supported on $\Lambda$ \cite[Chapter IV]{KMS99}.

Let $G = \Diff(M)$ and let $e \in G$ denote the identity transformation on $M$.  We can consider the subgroup
\[
	G^{(0)}_\Lambda := \{ \psi \in G \mid \psi(x) = x \quad  \forall x \in \Lambda \}
\]
and the normal subgroups
\[
	G^{(k)}_\Lambda :=  \{ \psi \in G^{(0)}_\Lambda \mid \Jet_\Lambda^{(k)} \psi = \Jet_\Lambda^{(k)} e \}
\]

Moreover, the Lie algebra of $G^{(k)}_\Lambda$ is
\[
	\mathfrak{g}^{(k)}_\Lambda = \{ \eta^{(k)}_\Lambda \in \mathfrak{X}(M) \mid \Jet_\Lambda^{(k)} \eta^{(k)}_\Lambda = \Jet^{(k)}_\Lambda(0) \}.
\]
In other words, $\mathfrak{g}^{(k)}_\Lambda$ is the sub-algebra of $\mathfrak{X}(M)$ consisting of vector fields with vanishing partial derivatives up to order $k$ at the points of $\Lambda$.

\begin{prop} \label{prop:jets_as_quotient}
	The functor, $\Jet^{(k)}_\Lambda$, is the principal bundle projection from $G$ to $Q^{(k)} = G / G^{(k)}_\Lambda$.
\end{prop}
\begin{proof}
	This is merely the definition of $\Jet^{(k)}_\Lambda$,
	and a more thorough description of this statement can be found in \cite{KolarMichorSlovak1993}.
	Nonetheless, we will attempt a skeletal proof here.
	
	If $\varphi_2 = \varphi_1 \circ \psi$ for some $\psi \in G^{(k)}_\Lambda$ then $\Jet^k_\Lambda( \varphi_2) = \Jet^k_\Lambda( \varphi_1 \circ \psi)$.  However, $\psi$ has absolutely no impact on the $k$th order Taylor expansion because the Taylor expansion of $\psi$ is trivial to $k$th order.  Thus $\Jet^{(k)}_\Lambda( \varphi_1) = \Jet^{(k)}_\Lambda( \varphi_2)$ and so $\Jet_\Lambda^{(k)}$ is a well defined map on the coset space $Q^{(k)}$.  Conversely, for each element $q \in Q^{(k)}$ one can show that the inverse image $(\Jet^{(k)}_\Lambda)^{-1} (q)$ is composed of a single $G^{(k)}_\Lambda$ orbit and no more.
\end{proof}
For example if $\Lambda$ consists of only two distinct points then 
\begin{align*}
	& Q^{(0)} = \{ (y_1,y_2 ) \in M^{ 2 } \mid  y_1 \neq y_2 \},\\
	& Q^{(1)} = \{ ( f_1, f_2 ) \in \Fr(M)^2 \mid \pi_{\Fr}(f_1) \neq \pi_{\Fr}(f_2) \}.
\end{align*}
where $\pi_{\Fr}: \Fr(M) \to M$ is the frame bundle of $M$.

\begin{prop}
	$\Jet^{(k)}_\Lambda( G^{(0)}_\Lambda)$ is a (finite dimensional) Lie group, and the functor $\Jet^{(k)}_\Lambda$ restricted to $G^{(0)}_\Lambda$ is a group homomorphism.
	Moreover $\Jet^{(k)}_\Lambda(G^{(0)}_\Lambda)$ is a normal subgroup of $\Jet^{(k)}_\Lambda(G^{(l)}_\Lambda)$ for all $l \in \mathbb{N}$.
\end{prop}

\begin{cor}
	The space $Q^{(k)}$ is a (finite-dimensional) principal bundle with structure group $\Jet^{(k)}_\Lambda( G^{(0)}_\Lambda )$.
\end{cor}
For $k = 0$ this structure group is trivial.
At $k=1$ this structure group is identifiable with $\GL(d)$ where $d = \dim(M)$.

\section{An $O(h^p)$ accurate algorithm}
In this section, we describe the basic strategy for using jet-particles
to get high order accuracy in solutions to LDDMM problems posed on $M = \mathbb{R}^d$.
The algorithm uses an $O(h^p)$ approximation to the matching term 
which is $G_\Lambda^{(k)}$ invariant.
We then invoke Theorem \ref{thm:reduction} to reduce the problem to a
finite dimensional boundary problem on the space of $k$th order jet-particles $Q^{(k)}$.
We solve this problem to obtain
an approximation of the solution to the original problem.

As discussed in Section~\ref{sec:jetparticles}, a finite number of jet-particles in $Q^{(k-1)}$ can approximate jet-particles in $Q^{(k)}$. In particular, zeroth order jet-particles can approximate any of the higher-order jet-particles in the hierarchy. While using low order particles may perform well in practical applications, they do not represent exact solutions to higher-order discretizations of the matching term. Therefore, they cannot represent the exact solutions to the approximated problem that we seek here. Furthermore, using lower-order particles to represent higher-order jets can be numerically unstable as the approximation is in essence a finite difference approximation: the particles need to be very close and the momentum can be of very large magnitude.

We will assume that the problem is defined on a reproducing kernel Hilbert space (RKHS), which we denote by $V \subset \mathfrak{X}( \mathbb{R}^n)$
where $\mathfrak{X}(\mathbb{R})$ will denote the space of $C^k$ vector fields.
We will denote the kernel of $V$ by $K: \mathbb{R}^d \times \mathbb{R}^d \to \mathbb{R}$ \cite[Chapter 9]{Younes2010},
and we will assume $V$ satisfies the admissibility condition
\begin{align}
  \| v \|_V \geq \| v \|_{\bar{k},\infty} \label{eq:admissible}
\end{align}
for a constant $C>0$ a positive integer $\bar{k} \in \mathbb{N}$ and all $v \in V$.
We will denote the topological group which
integrates $V$ by $G_V$.

To make precise what we mean by ``an $O(h^p)$ approximation'' to a matching term, we will recall the ``big $O$'' notation.
\begin{defn}
  Let $F : G_V \to \mathbb{R}$, and let $F_h: G_V \to \mathbb{R}$ depend
on a parameter $h > 0$.  We say that $F_h$ is an
\emph{$O(h^p)$-approximation} to $F$ if
\begin{align*}
  \lim_{h \to 0} \left( \frac{ F(x) - F_h(x) }{h^p} \right) < \infty
\end{align*}
for all $x \in \mathbb{R}^d$.
Moreover, $O(h^p)$ will serve as a place-holder for an arbitrary function within the equivalence class of all
functions of $h$ which vanish at a rate of $h^p$ or faster as $h \to 0$.
Under this notation, $F_h$ is an $O(h^p)$-approximation of $F$ if $F = F_h +  O(h^p)$.
\end{defn}

To illustrate how we may produce $O(h^p)$-approximations to matching functions we will consider the following example.
\begin{example}
\label{sec:example1}

Let $I_0,I_1 \in C^k( \mathbb{R}^d ; [0,1])$ be two greyscale images with compact support.
We can consider the matching functional $F: \Diff(\mathbb{R}^d) \to \mathbb{R}$ given by
\begin{align*}
  F( \varphi ) = \frac{1}{\sigma} \| I_0 - (I_1 \circ \varphi) \|_{L_2}^2
  = \frac{1}{\sigma} \int_{\mathbb{R}^d} | I_0(x) - I_1( \varphi(x)) |^2 dx.
\end{align*}
As $I_0$ and $I_1$ each have compact support, the integral term can be restricted to a compact domain.
We will continue to write our integrals as integrations over $\R^d$, but we will exploit this compactification when we need to.

Consider the regular lattice $\Lambda_h = \mathbb{Z}^d h$ whereupon, 
for sufficiently small $h > 0$, the $L_2$-integral can be approximated to order $O(h^d)$ with a Riemann sum
\begin{align*}
  F_h^{(0)} (\varphi) =& \sum_{x \in \Lambda_h} h^d (I_0(x) - I_1( \varphi(x)) )^2
\end{align*}
While the order of the set $\Lambda_h$ is infinite, the sum over $\Lambda_h$ used to compute $F_h$ has only finitely many non-zero terms to consider because $I_0$ and $I_1$ have compact support.
Moreover, $F_h$ is $G_{\Lambda_h}^{(0)}$ invariant
because it only depends on $\varphi(x)$ for $x \in \Lambda_h$.

An $O(h^{d+2})$ approximation is given by
\begin{align*}
  F^{(2)}_h (q) =& \sum_{x \in \Lambda_h} h^d (I_0(x) - I_1(\varphi(x) ) )^2  \\
  & +\sum_{\alpha}\frac{h^{d+2}}{12} \left[ ( \partial_\alpha I_0(x) - \partial_\beta
    I_1(\varphi(x) ) \partial_\alpha \varphi^\beta(x) )^2 \right] \\
   & \qquad\quad + \frac{h^{d+2}}{12} \big[ \left( I_0(x) - I_1(\varphi(x) ) \right)
   ( \partial_{\alpha}^2 I_0(x) \\
 &\qquad\quad\  - \partial_{\beta \gamma} I_1(\varphi(x)) \partial_\alpha\varphi^\beta(x) \partial_\alpha \varphi^\gamma(x)
 - \partial_{\gamma} I_1(\varphi(x) ) \partial_{\alpha \alpha} \varphi^\gamma(x)) \big]\ ,
\end{align*}
and we can observe that $F^{(2)}_h$ is $G^{(2)}_{\Lambda_h}$ invariant because $F^{(2)}(\varphi)$ only depends on the 2nd order Taylor expansion of $\varphi$ centered at each $x \in \Lambda_h$.
\end{example}

Given a $G^{(k)}_{\Lambda_h}$-invariant  $O(h^p)$-approximation $F_h: \Diff(\mathbb{R}^d) \to \mathbb{R}$ to the matching term $F$, we may consider the alternative curve energy
\begin{align*}
  E_h[ \varphi ] = \frac{1}{2} \int_{0}^{1} \|v(t)\|_V^2 + F_h( \varphi_1)\ ,
\end{align*}
where $v(t) \in V$ is the Eulerian velocity field $v(t,x) = \partial_t \varphi_t ( \varphi_t^{-1}(x))$.
For a fixed curve $\varphi_t$, we observe that $E_h$ is an $O(h^p)$-approximation to $E$.
One might surmise that the extremizers of $E_h$ provide good approximations of the extremizers of $E$.
This is important because $E_h$ is $G^{(k)}_{\Lambda_h}$-invariant, and we can invoke Theorem \ref{thm:reduction} to solve for extremizers of $E_h$, but we can not do this for $E$.
Fortunately, for many choices of $F_h$, there will be minimizers of $E_h$ that converge to those of $E$ as $h \to 0$ with a known convergence rate.

\begin{thm} \label{thm:order}
  Let $F :G_V \to \mathbb{R}$ be $C^2$ with respect to 
  the topology induced by $V$.\footnote{We will assume that $G_V$ is a smooth manifold and a topological Lie group.  For example the space of $H^s$ diffeomorphisms with $s$ sufficiently large.}
  Let $F_h :G_V \to \mathbb{R}$ be $C^2$ and an $O(h^p)$-approximation for $F$ with $G^{(k)}_{\Lambda_h}$ invariance,
  constructed as above.
  Consider the curve energies
  $E,E_h : C([0,1],V) \to \mathbb{R}$
  \begin{align*}
    E[v(\cdot)] = \frac{1}{2} \int_{0}^{1} \| v(t) \|_V dt + F( \varphi_1) \\
   E_h[ v(\cdot)] = \frac{1}{2} \int_{0}^{1} \| v(t) \|_V dt + F_h( \varphi_1)\ .
  \end{align*}
  where $\varphi_1 \in G_V$ is the Lie integration of $v(t)$.
  Let $v^*$ minimize 
  \begin{align*}
  	e = E \circ \evol_{EP} : V \to \mathbb{R}
  \end{align*}
  If the Hessian at $v^* \in V$ is bounded, positive definite, and non-degenerate at $v^* \in V$, and $\evol_{EP}$ exhibits $C^2$ dependency upon the initial velocity field,
  then, for sufficiently small $h$, there exist a minimizer $v_{h}^*$ of
  \begin{align*}
  	e_h = E_h \circ \evol_{EP}:V \to \mathbb{R}
  \end{align*}
  which is an $O(h^p)$-approximations of $v^*$ in the $V$-norm.
\end{thm}

We will employ the following well-known result to approximate vector-fields in $V$ with finite linear combinations of the RKHS kernel $K$.
\begin{lem} \label{lem:dense}
  Assume $V$ satisfies the admissibility assumption \eqref{eq:admissible}.
  Consider the subspace of vector-fields
  \begin{align*}
    V_h^{(k)} = \{ v \in \mathfrak{X}(\mathbb{R}^d) \mid v = \sum_{y \in \Lambda_h, |\alpha| \leq k } \alpha_y \partial_\alpha K(x - y) \},
  \end{align*}
  for $k < (\bar{k} - 1) / 2$.
  The set
  $
    W = \cup_{h > 0} V_h^{(0)}
  $
  is dense in $V$ with respect to $\langle \cdot , \cdot \rangle_V$.
\end{lem}

\begin{proof}
  Let $\{ h_j > 0\}$ be a sequence such that $\lim_{j \to \infty}(h_j) = 0$.
  Let $v \in V$ be orthogonal to $W$.
  Thus $\langle v , w \rangle_V = 0$ for all $w \in W$.
  That is to say $v(x) = 0$ for all $x \in \Lambda_{h_j}$ and all $j \in \mathbb{N}$.
  However, any point $y \in \mathbb{R}^n$ is the limit of a sequence $\{ x_j \in \Lambda_{h_j} \}$.
  Since all members of $V$ are continuous, it must be the case that $v = 0$.
\end{proof}

A direct corollary is that $W^{(k)} = \cup_{h>0} V_h^{(k)}$ is dense in $V$ since $V_h^{(0)} \subset V_h^{(k)}$ for any $k \in \mathbb{N}$.
Lemma \ref{lem:dense} will allow us to approximate our cost functional on $V$.

\begin{proof}[Proof of Theorem \ref{thm:order}]
By the Morse Lemma (suitably generalized to Hilbert Manifolds \cite{Tromba1983,GolubitskyMarsden1983}),
there exists a smooth coordinate chart around $v^*$,
$\Phi:U \to V$, such that  $\Phi(v^*) = v^*$ and
$\tilde{e}(v^* + w) = \tilde{e}(v^*) + D^2_{v^*}\tilde{e}(  w , w)$, where $\tilde{e} := e \circ \Phi$.
Define also $\tilde{e}_h := e_h \circ \Phi$

Note that $e(v^*) - e_h(v^*) = F(\varphi^*) - F_h(\varphi^*)$ for $\varphi^* := \evol_{EP}(v^*)$.
By the assumption on
$\evol_{EP}$, and since $F,F_h \in C^2(G)$, we observe that
$e- e_h$ is $C^2$ at $v^* \in V$.
Moreover, we know that $e - e_h = (F - F_h) \circ \evol_{EP} = O(h^p)$.
We can discard the ``big O'' notation and write
\begin{align*}
	e(v) - e_h(v) = A(v) h^{p} + B(v,h)
\end{align*}
where $A \in C^2(\mathfrak{X}(M))$ is independent of $h$ and $\partial^k_h |_{h=0}B = 0$ for $k \leq p$.

Since $D^2e(v^*)$ is nondegenerate, there exists $\kappa>0$ such that
$\|D^2\tilde{e}(w)\|\ge \kappa^2\|w\|$. Therefore
\begin{equation*}
  \tilde{e}(v^* + w) 
  \ge \tilde{e}(v^*) + \kappa\|w\|^2 
  \ .
\end{equation*}
Thus, for sufficiently small $r>0$, $\tilde{e}(v^* + w)\ge\tilde{e}(v^*)+\delta_r$
with $\delta_r=\kappa^2r^2$
when $\|w\|=r$. Given such $r>0$, choose $h$ sufficiently small so that $|e(v)-
e_h(v)|<\delta_r/3$ in $U$ and so that there exists $v^*_h\in \Phi^{-1}(V_h^{(k)} \cap U)$ with
$|\tilde{e}(v^*)-\tilde{e}(v^*_h)|<\delta_r/3$ (we know such a $v_h^*$ exists by Lemma \ref{lem:dense}). Then $\tilde{e}_h(v^*_h)<\tilde{e}(v^*)+2\delta_r/3$ and
$\tilde{e}_h(v^*+w)>\tilde{e}(v^*+w)-\delta_r/3\ge\tilde{e}(v^*)+2\delta_r/3$ when $\|w\|=r$.
Thus there exists a point inside the intersection of the $r$-ball of $v^*$ and $\Phi^{-1}(V_h^{(k)} \cap U)$ where
$\tilde{e}_h$ is strictly smaller than on the boundary of this intersection. Since
$V_h^{(k)}$ is finite dimensional, this implies the existence of a local minimizer
$\tilde{v}_h^*\in \Phi^{-1}(V_h^{(k)} \cap U)$ of $\tilde{e}_h|_{ \Phi^{-1}(U\cap {V_h^{(k)}}) }$. By
Theorem~\ref{thm:reduction}, $\tilde{v}_h^*$ is also a local minimizer of $\tilde{e}_h$ on $U$.
As $h\rightarrow 0$, we can let $r\to 0$,  and
the local minima $\tilde{v}_h^*$ will approach $v^*$. In addition, 
$v_h^* = \Phi^{-1}(\tilde{v}^*_h)$ is a local minimum for $e_h$ which must also approach $v^*$.
This proves convergence as $h \to 0$.


We now address the order of accuracy.
As $v_h^*$ is a critical point of $e_h$, we have that $De_h(v_h^*) = 0$, where $De_h$ is the Frech\'et derivative of $e_h$.
If we define $\tilde{w} = \tilde{v}^* - \tilde{v}^*_h$
then we observe
\begin{align*}
  0 = D \tilde{e}_h(\tilde{v}_h^* ) &= D \tilde{e}(\tilde{v}_h^*) - D \tilde{A}(\tilde{v}_h^*) h^p - D \tilde{B}(\tilde{v}_h^*,h) \\
  &= D \tilde{e}(v^*) + D^2 \tilde{e}(v^*)( \tilde{w},\cdot) - D \tilde{A}(\tilde{v}_h^*) h^p - D \tilde{B}(\tilde{v}_h^*,h)\ .
\end{align*}
Moreover $D \tilde{e}(v^*) = 0$ because $v^*$ is a critical point of $\tilde{e}$.
Thus we observe
\begin{align*}
	D^2 \tilde{e}(v^*)( \tilde{w},\cdot) =  D \tilde{A}(\tilde{v}_h^*) h^p + D \tilde{B}(\tilde{v}_h^*,h)
\end{align*}
We can observe that the Hessian $D^2 \tilde{e}(v^*)$ is related to the Hessian $D^2 e(v^*)$ via 
pre-composition by the bounded linear operator $D\Phi(v^*)$.
Thus $D^2 \tilde{e}(v^*)$ is a bounded operator from $U$ into $V^*$.
By assumption, this Hessian is non-degenerate, and thus invertible. 
Thus we observe $\tilde{w} = [D^2 \tilde{e}(v^*)]^{-1} \cdot \left( D \tilde{A}(\tilde{v}_h^*) h^p + D \tilde{B}(\tilde{v}_h^*,h) \right)$.
In other words, $v^* = \tilde{v}^*_h + O(h^p)$.
So there exists an $O(1)$ function $C(v)$ and a function $F(v,h)$ such that $v^* = v_h^* + C(v) h^p + F(v,h)$ where $d^k F / d h^k = 0$ for $k \leq p$.
Thus we find
\begin{align*}
	v^* &= \Phi^{-1}(v^*) = \Phi^{-1}(\tilde{v}^*_h + C(v) h^p + F(v,h) ) \\
	&= \Phi^{-1}( \tilde{v}_h^*) + D\Phi^{-1}(\tilde{v}_h^*) \cdot \left( C(v) h^p + F(v,h) \right) + O(h^{2p})\\
	&= v_h^* + O(h^p).
\end{align*}
\end{proof}

The assumption that the Hessian of the curve energy be non-degenerate
is generally difficult to check in practice.
We can still invoke this theorem in specific examples
because the minimizer of
\begin{align*}
  E(v ) = \frac{1}{2} \int_0^1  \| v(t) \|_V^2 dt \ ,\ v(t) = \evol^t_{EP}(v)
\end{align*}
is $v^* = 0$, and the Hessian is identical to the inner product on $V$.
We can view all relevant examples as perturbations of this curve energy,
 and use the continuity of the Hessian operator to invoke Theorem \ref{thm:order}.

Setting $G = G_V$, $Q = Q^{(k)} = G_V / G^{(k)}_{\Lambda_h}$ in Algorithm 1, we obtain
the special case of Algorithm 1 given by
\begin{center}
\fbox{
    \parbox{0.9\textwidth}{
    	{\bf Algorithm 2:}  
        \begin{enumerate}
          \item Solve for $(q(t),p(t)) \in T^{\ast}Q^{(k)}$ in \eqref{eq:Hamilton}.
          \item If necessary, set $u(t) = K * J( q(t),p(t))$
          and obtain $\varphi_t \in G_V$ through the reconstruction formula $\dot{\varphi_t}(x) = u (\varphi(x))$ for all $x \in M$.
          \item Evaluate the cost function, and backward compute the adjoint equations to compute the gradient of the cost
          	function with respect to a new initial condition.
	 \item If the gradient is below some tolerance, $\epsilon$, then stop.
	 	Otherwise use the gradient to create a new initial condition and return to step 1.
        \end{enumerate}
    }
}
\end{center}
Step (1) concerns solving a system of Hamiltonian equations on the space of $k$th order jet-particles.
Here the configuration is given by numbers $[q_i]\indices{^\alpha_\beta}$ where $\alpha \in \{ 1,\dots,d\}$, $\beta$ is a multi-index on $\R^d$ of degree less than or equal to $k$, and $i \in \{ 1,\dots,N\}$
where $N$ is the number of jet-particles.  The Hamiltonian can usually be computed in closed form if the Green's kernel, $K$, of the RKHS $V$ is known in closed form.
For example if $k=0$, we obtain traditional particles (i.e. the multi-index $\beta$ is of degree $0$) and the momentum map is given by
\begin{align*}
	J(q,p) = \sum_{i=1}^N p_i \otimes \delta_{q_i}
\end{align*}
where $(q_i,p_i) \in T^*\R^d \cong \R^d \times \R^d$, and
where $p_i \otimes \delta_{q_i}$ is a \emph{one-form density} representation of the element of $V^*$ given by $ \langle p_i \otimes \delta_{q_i} , w \rangle := p_i \cdot w(q_i)$
for all $w \in V$.
The Hamiltonian is then given by
\begin{align*}
	H(q,p) = \frac{1}{2}  \sum_{i,j =1}^N (p_i \cdot p_j) K(q_i - q_j)\ ,
\end{align*}
and Hamilton's equations are
\begin{align*}
	\dot{q}_i = \sum_{j=1}^N p_j K(q_i - q_j) \quad,\quad \dot{p}_i = - \sum_{j=1}^N (p_i \cdot p_j) DK(q_i - q_j) \ .
\end{align*}
The velocity field in step (2) is
\begin{align*}
	u(x) = \sum_{j=1}^N p_j K(x - q_j) \ .
\end{align*}
Steps (3) and (4) are obtained through formulas of comparable complexity.

For arbitrary $k$, the relationship between the momenta $(q,p) \in T^*Q^{(k)}$ and
the velocity field $u$ invokes the multivariate Fa\`a di Bruno formula, and so the algebra rises in complexity very quickly \cite{ConstantineSavits1996}.
In the case of $k=2$, the Hamiltonian is substantially more complex but still tractable (see \eqref{eq:Hamiltonian} in Appendix \ref{sec:eom}).
However, beyond $k=2$, a symbolic algebra package is advised. 
The examples we will be considering in this paper concern the case where $d = 2$, and  $k=0$ or $2$.
By Theorem \ref{thm:order}, we should be able to approximate minimizers of $E$
with $O(h^{2})$ and $O(h^4)$ accuracy respectively in the $V$-norm.

\section{Numerical Results}
In this section we, will illustrate the deformations encoded by jet-particles of various orders and
numerically verify the $O(h^{d+k})$ convergence
rate of the matching functional approximation for $k=0,2$. In addition, we will show that the second order approximation 
$F_h^{(k)}$, $k=2$ allows matching of
second order image features.
We use simple examples to describe the different capabilities of higher order jet-particles
over lower order jet-particles.
We do this by illustrating structures that cannot be matched with low numbers of regular zeroth order landmarks,
but can still be matched successfully with first and second order jet-particles. These effects
imply more precise matching of small scale features on larger images where more spatial derivatives can be leveraged.
In all examples, the jet-particles will be positioned on regular grids in the image
domains.

We do not pursue approximations with $k>2$ due to the difficulties of taking very high order derivatives of images and
kernel functions. In addition, code complexity rises rapidly as the order increases beyond 2.

\subsection{Implementation}
The results are obtained using the \emph{jetflows} code available at
\url{http://www.github.com/stefansommer/jetflows}. The package include scripts
for producing the figures displayed in this section. The implementation follows
Algorithm~2. The flow equations that are
given in explicit form in Appendix~\ref{sec:eom} 
and the adjoint equations that are given in explicit form in
Appendix~\ref{sec:adjoint} are
numerically integrated using SciPy's odeint solver
(\url{http://scipy.org}). Both equation systems require a series of tensor
multiplications.
The optimization is
performed with a quasi-Newton BFGS optimizer. The algorithm uses isotropic Gaussian
kernels. The images to be matched are pre-smoothed with a
Gaussian filter, and image derivatives are computed as analytic gradients of
B-spline interpolations of the smoothed images.
\begin{figure}[t]
  \begin{center}
    \subfigure[$f(x,y)=x+y$]{
      \begin{minipage}{0.30\columnwidth}
        \includegraphics[width=1.0\columnwidth]{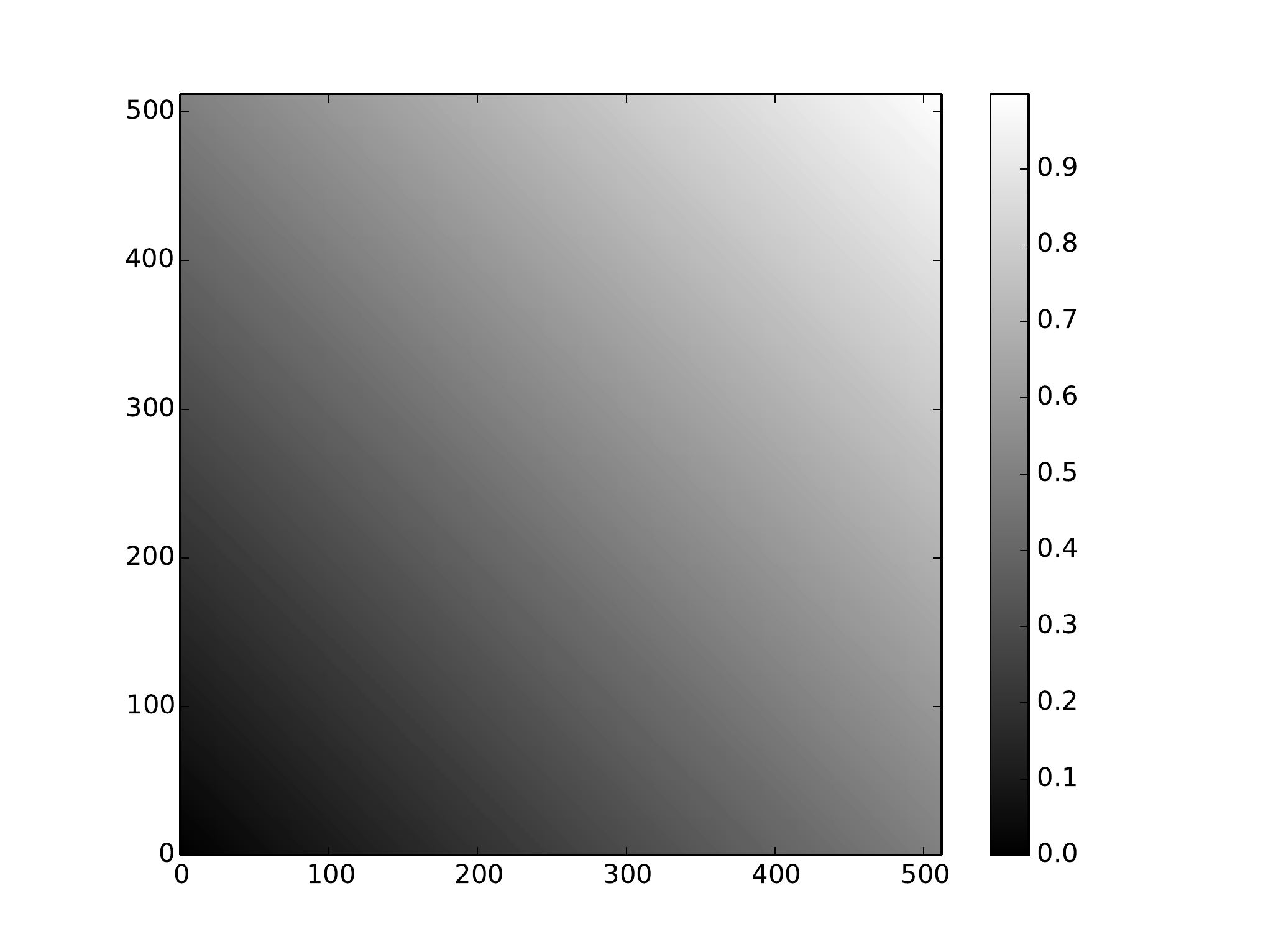}
        \includegraphics[width=.9\columnwidth]{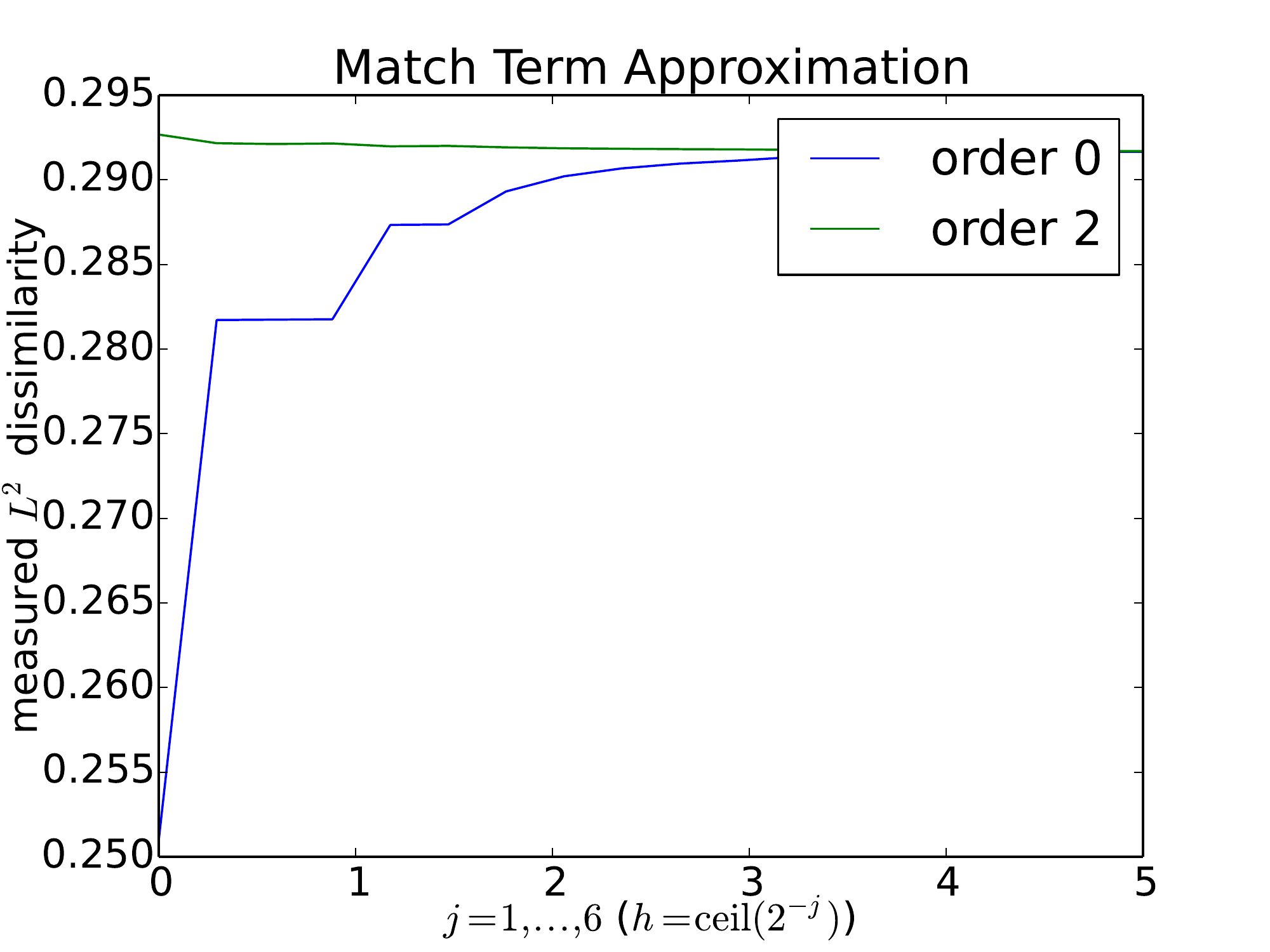}
        \includegraphics[width=.9\columnwidth]{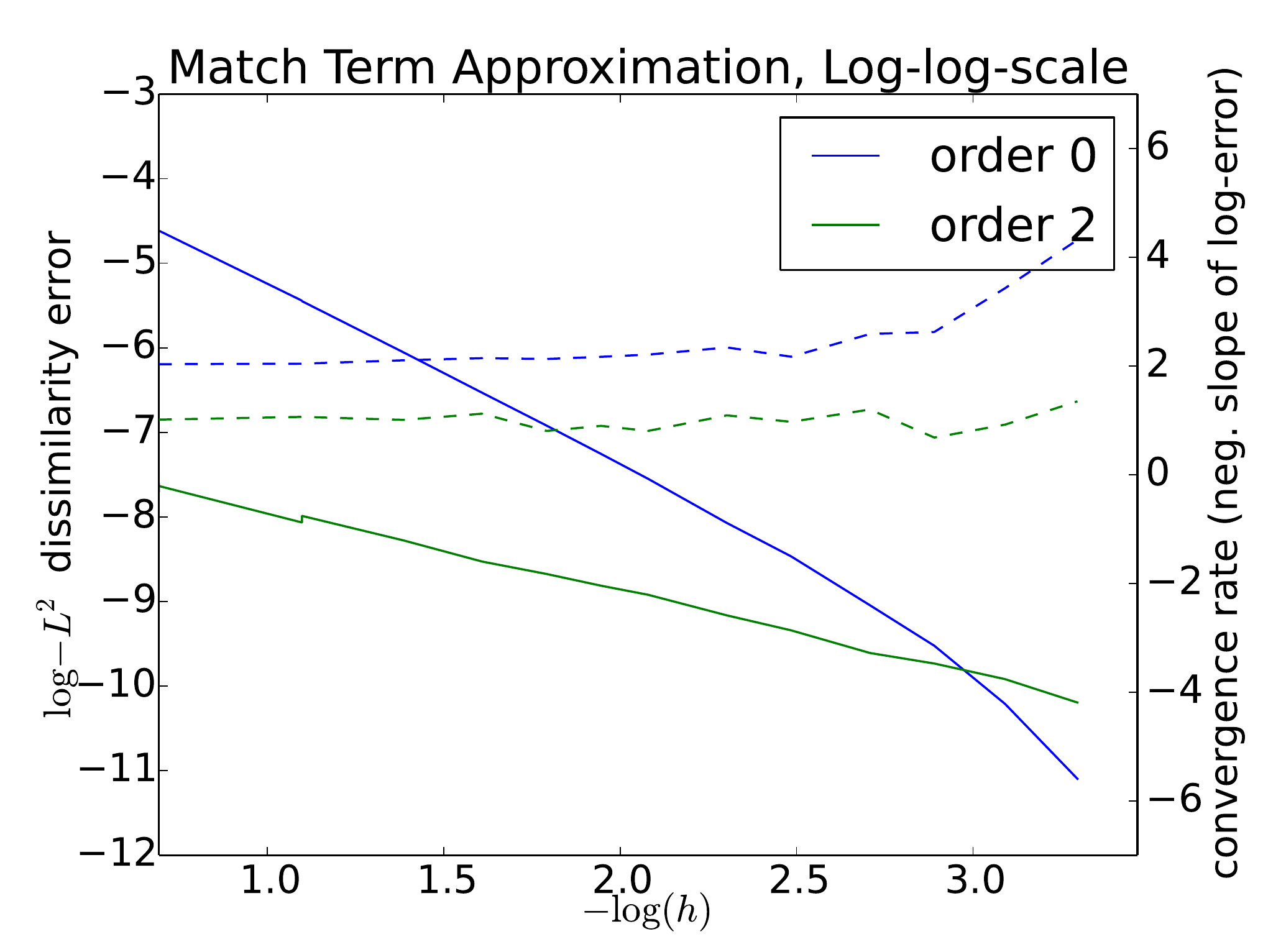}
      \end{minipage}
    }
    \subfigure[$f(x,y)=(x+y)^2$]{
      \begin{minipage}{0.30\columnwidth}
        \includegraphics[width=1.0\columnwidth]{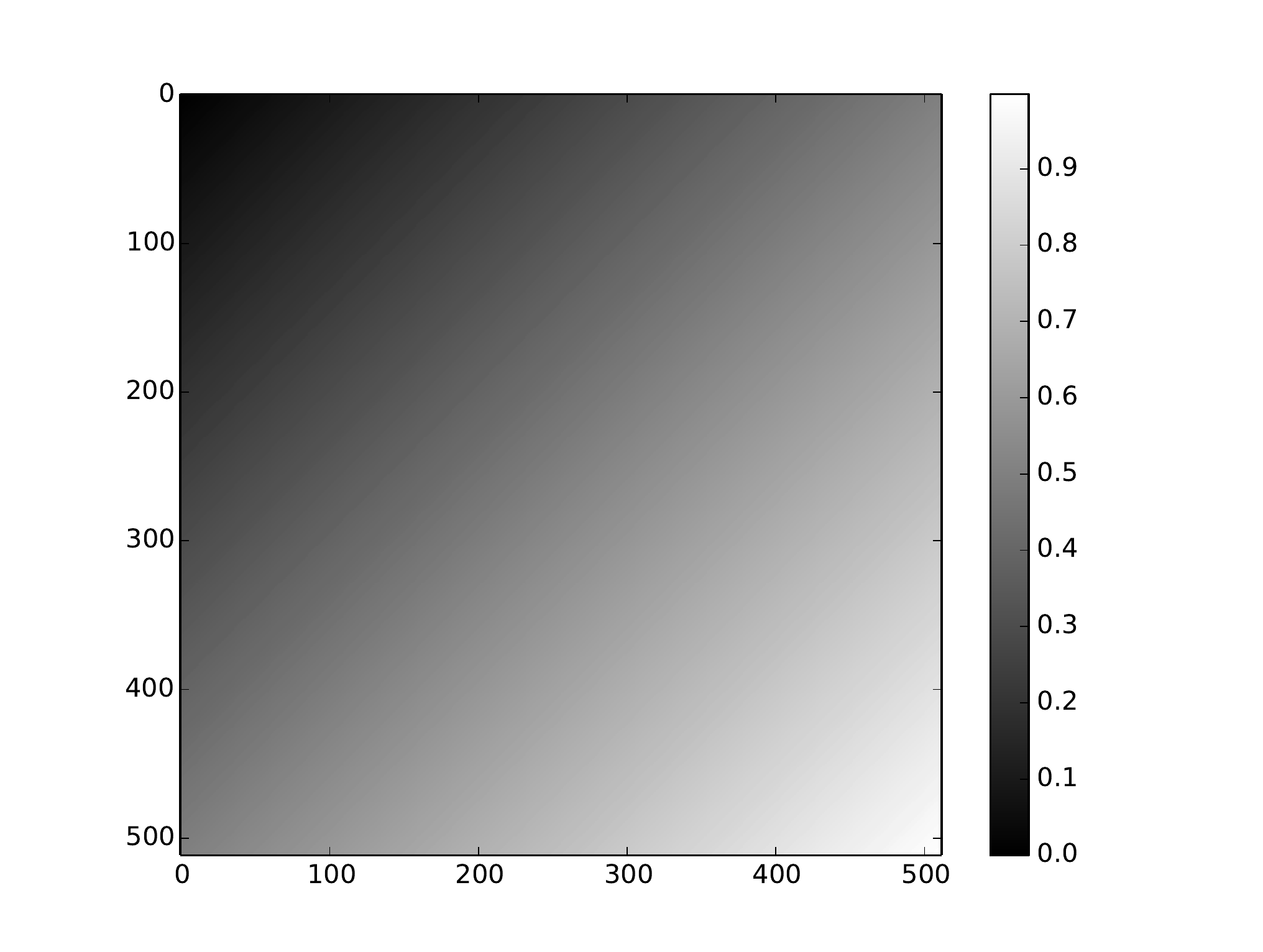}
        \includegraphics[width=.9\columnwidth]{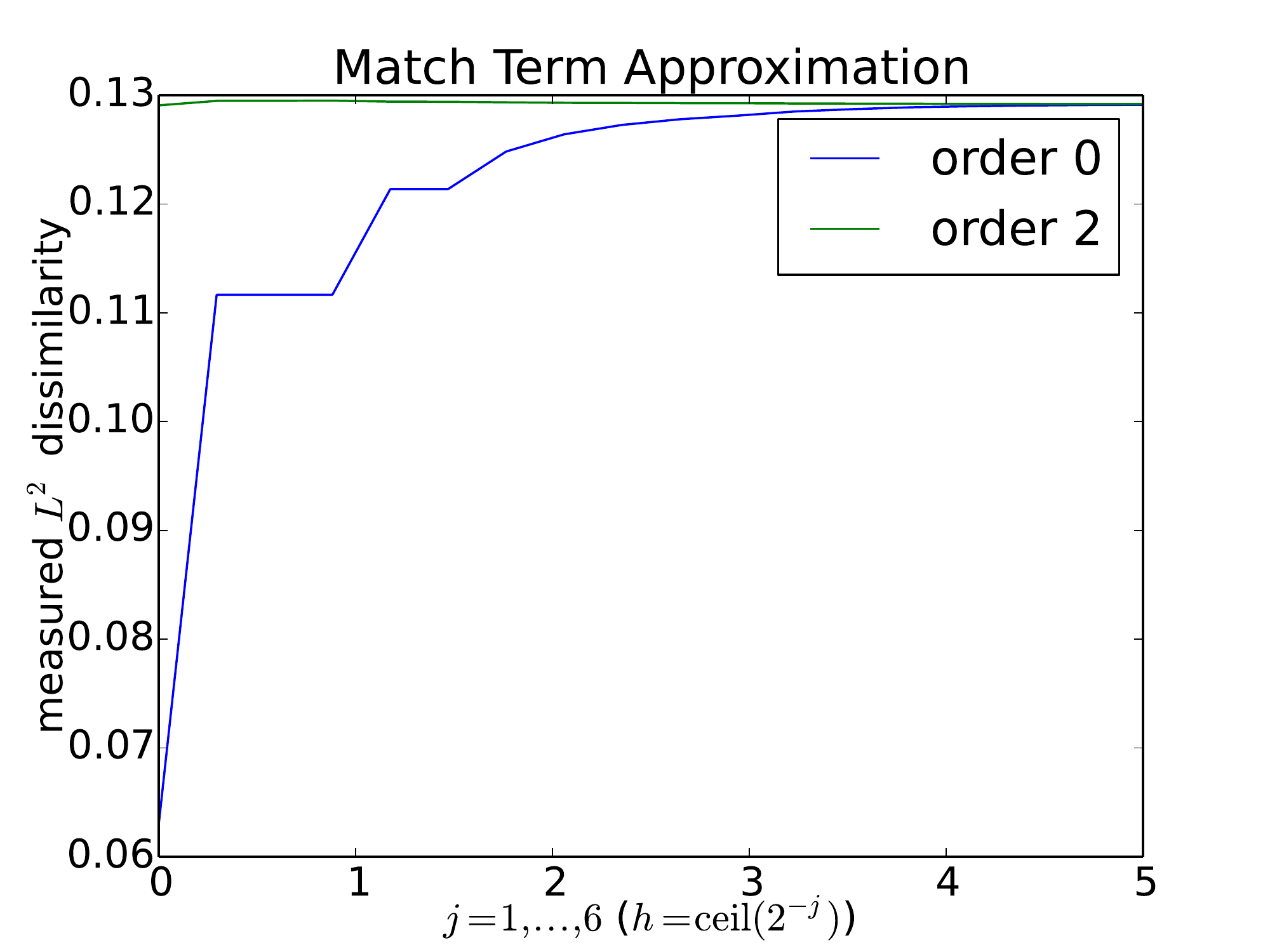}
        \includegraphics[width=.9\columnwidth]{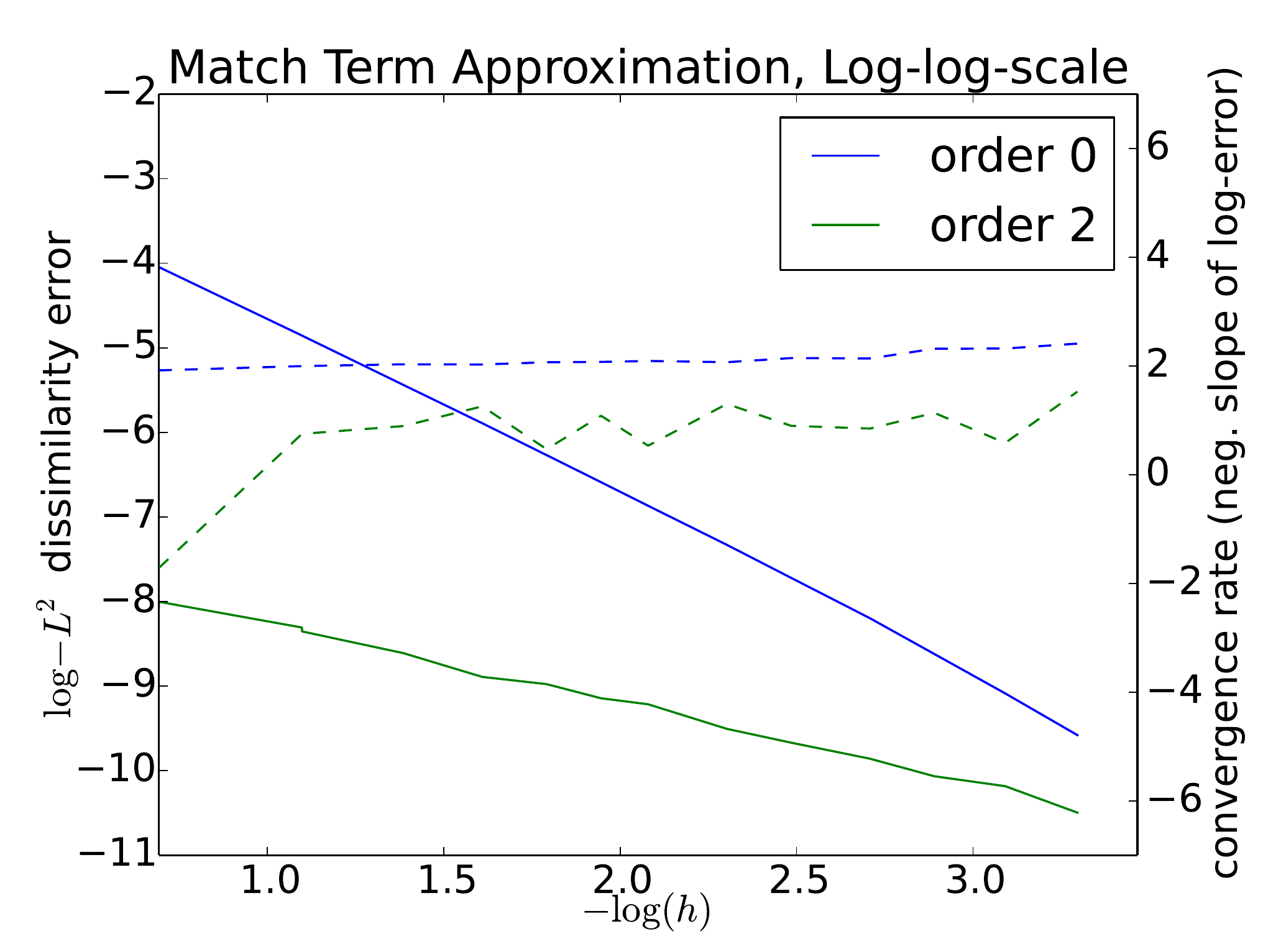}
      \end{minipage}
    }
    \subfigure[$f(x,y)=\sin(6\pi x)+x^2$]{
      \begin{minipage}{0.30\columnwidth}
        \includegraphics[width=1.0\columnwidth]{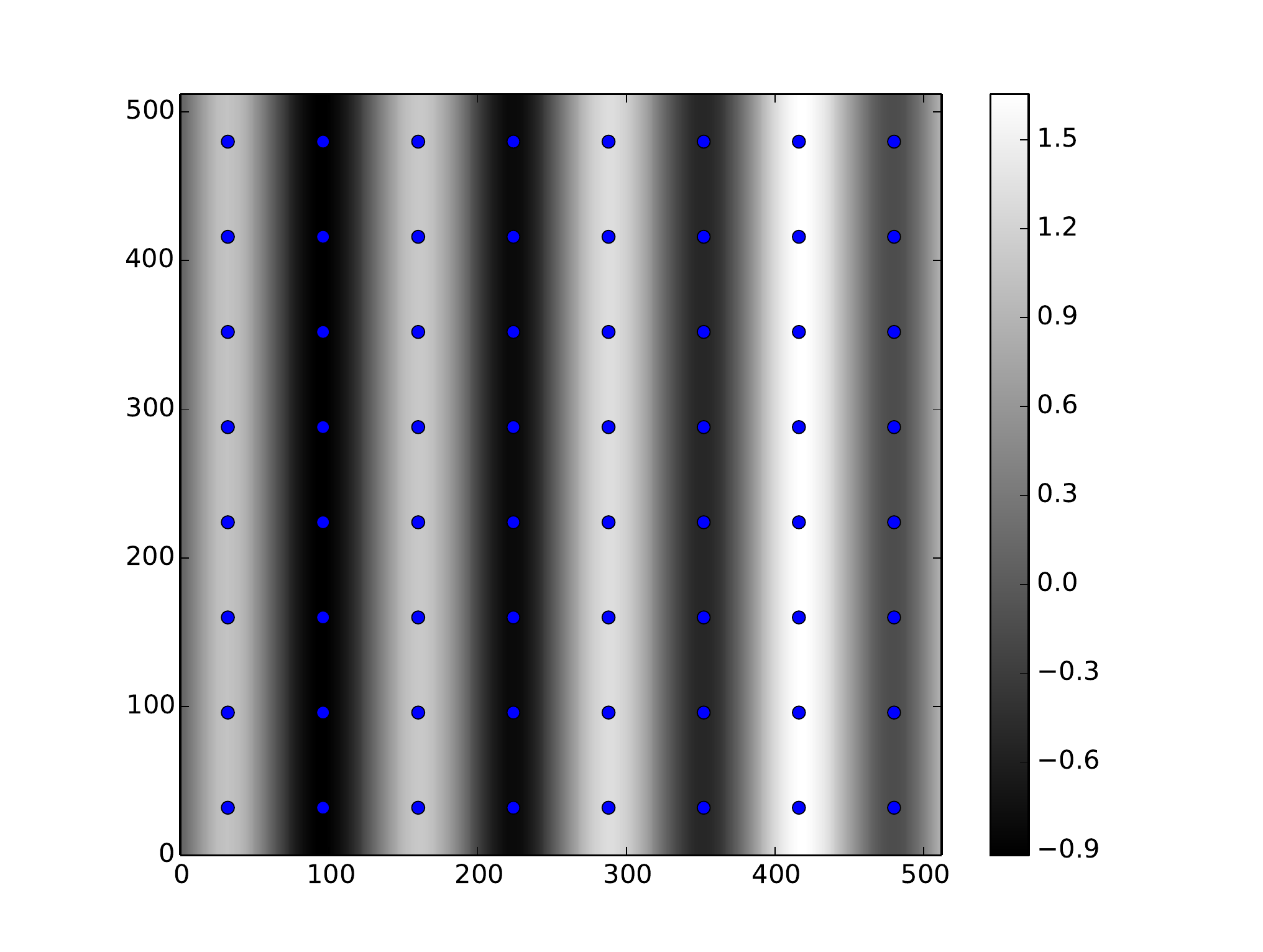}
        \includegraphics[width=.9\columnwidth]{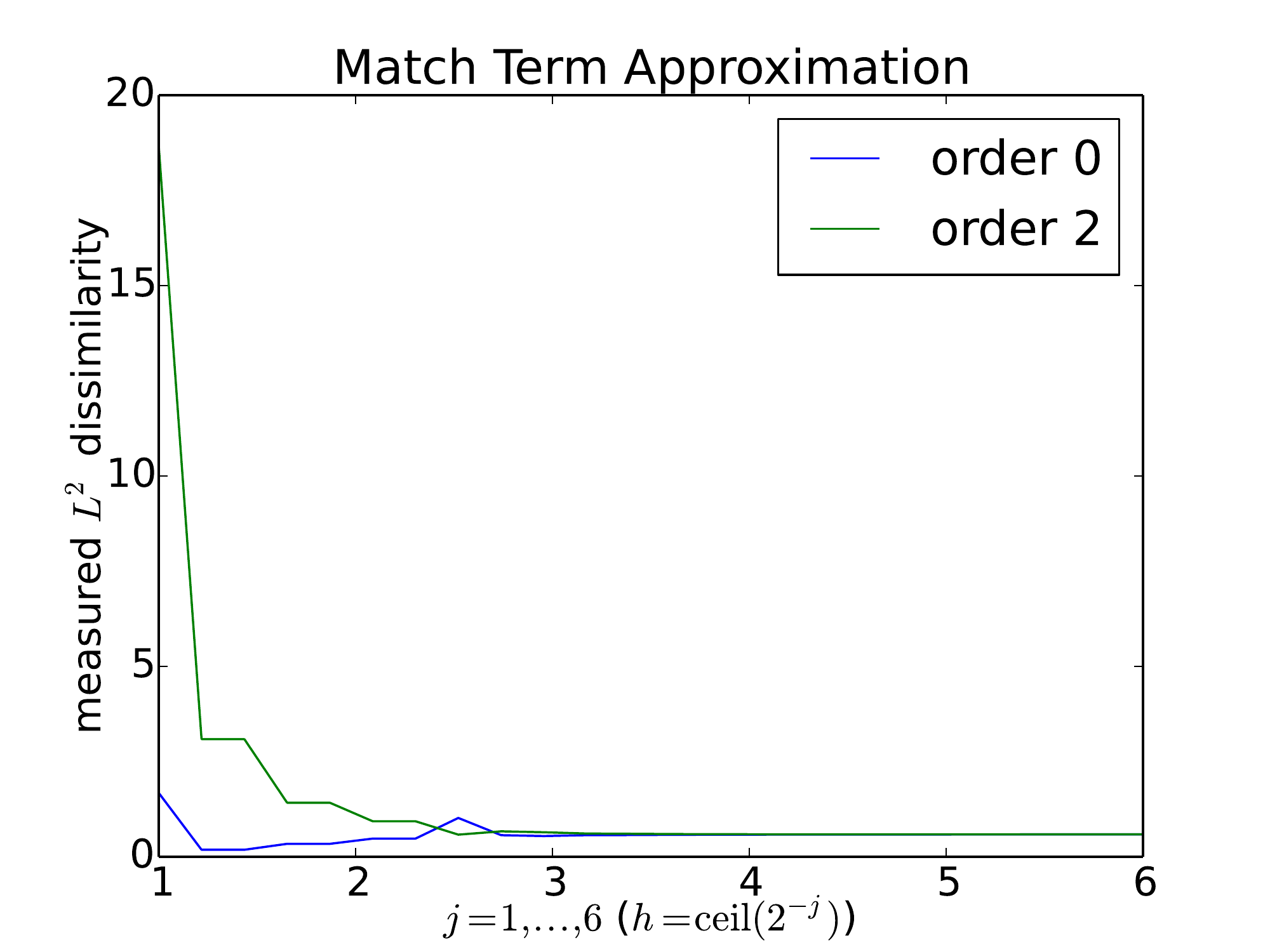}
        \includegraphics[width=.9\columnwidth]{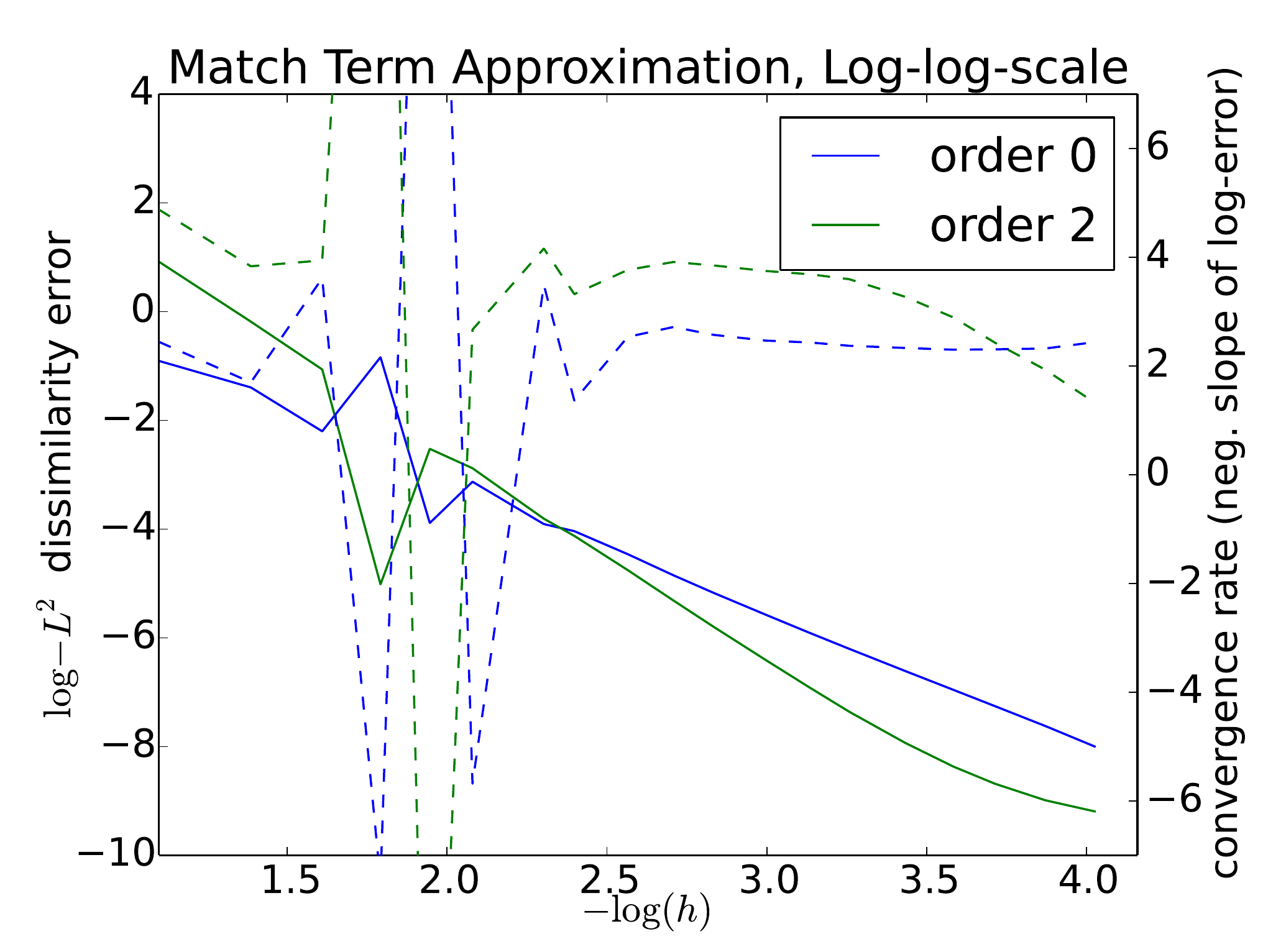}
      \end{minipage}
    }
  \end{center}
  \caption{Convergence of matching functional $F_h^{(k)}$, $k=0,2$. Top row:
  (a) linear, (b) quadratic, and (c) non-polynomial images. Lower rows, horz. axis:
  decreasing $h$ (increasing nr. of sample points); vert. axis: $F_h^{(k)}$ (solid, left axis) and
convergence rate (dashed, right axis). With linear and quadratic images, the error is vanishing with
$k=2$ and using only one sample point. Average convergence rates, $k=0$: quadratic; $k=2$:
quartic as expected. (c, top row) sample points for $h=2^{-3}$ ($2^3$ sample points per axis).}
  \label{fig:simtest}
\end{figure}

\subsection{Jet Deformations}
Figure~\ref{fig:shots} (page~\pageref{fig:shots}) shows the deformations encoded by zeroth, first and
second order jet-particles on initially square grids. Note the locally affine
deformations arising from the zeroth and first order jet-particles. Up to rotation of the
axes, the three first order examples in the figure constitute a basis of the 4 dimensional space of first order jet-particles with fixed lower-order components.
Likewise, up to rotation, the three second order examples constitute a basis for the 6 dimensional
space of second order jet-particles.

\subsection{Matching Functional Approximation}
We here illustrate and test the convergence rate of the matching functional
approximations.
In Figure~\ref{fig:simtest} (page \pageref{fig:simtest}), the approximations $F_h^{(p)}$
are compared for $p=0,2$ and varying grid sizes on three synthetic
images supported on the unit square. The first two images (a,b) are generated
by first and second order polynomials, respectively, while the last image (c)
is generated by a trigonometric function. 
A truncated Taylor expansion can therefore only approximate the image (c). The second order approximation $F_h^{(2)}$ models $F$ locally with a
second order polynomial, and it is thus expected that the error should vanish on
the images (a,b). As the mesh width $h$ decreases, we expect to observe 
$O(h^2)$ convergence rate for the zeroth order approximation
$F_h^{(0)}$ on all three images. Likewise, we expect a convergence rate of
$O(h^4)$ for $F_h^{(2)}$ on image (c).

In accordance with these expectations, we see the vanishing error for
$F_h^{(2)}$ on (a,b) and decreasing error on (c) (lower row, solid green lines).
The non-monotonic convergence seen on (c) is a result of the polynomial
approximation being integrated over a compact domain. The zeroth order
approximation $F_h^{(0)}$ likewise decreases with $h^2$ convergence rate (lower
row, dashed blue lines). The convergence rate of $F_h^{(2)}$ on image (c) stabilizes at approximately 
$h^4$ until it decreases due to numerical errors introduced when the error approaches
the machine precision.

\begin{figure}[t!]
  \begin{center}
    \subfigure[]{
        \includegraphics[width=.22\columnwidth,trim=175 150 240 150,clip]{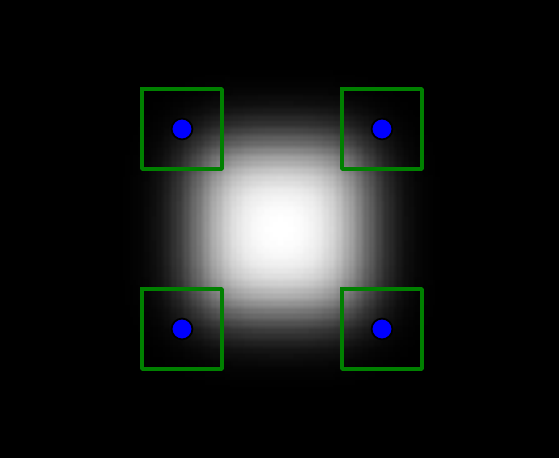}
    }
    \subfigure[]{
        \includegraphics[width=.22\columnwidth,trim=175 150 240 150,clip]{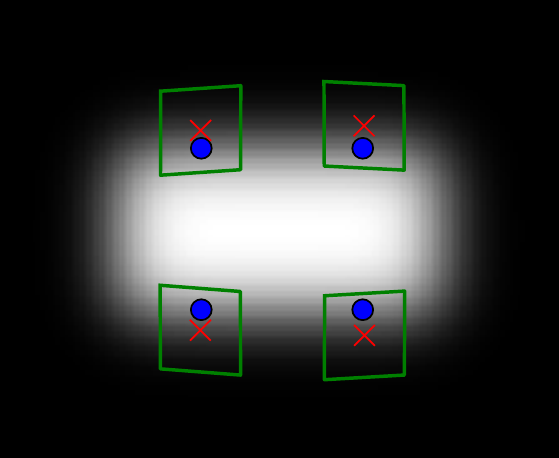}
    }
    \subfigure[]{
        \includegraphics[width=.22\columnwidth,trim=175 150 240 150,clip]{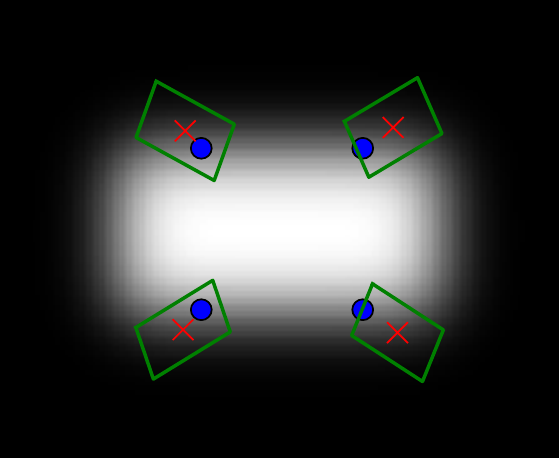}
    }
    \subfigure[]{
        \includegraphics[width=.22\columnwidth,trim=175 150 240 150,clip]{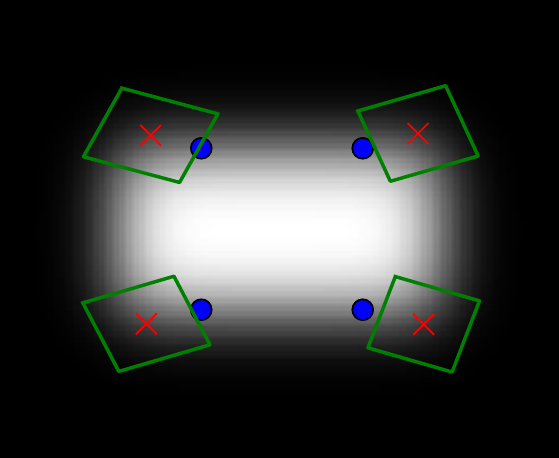}
    }
    \newline
    \subfigure[]{
        \includegraphics[width=.22\columnwidth,trim=90 70 155 80,clip]{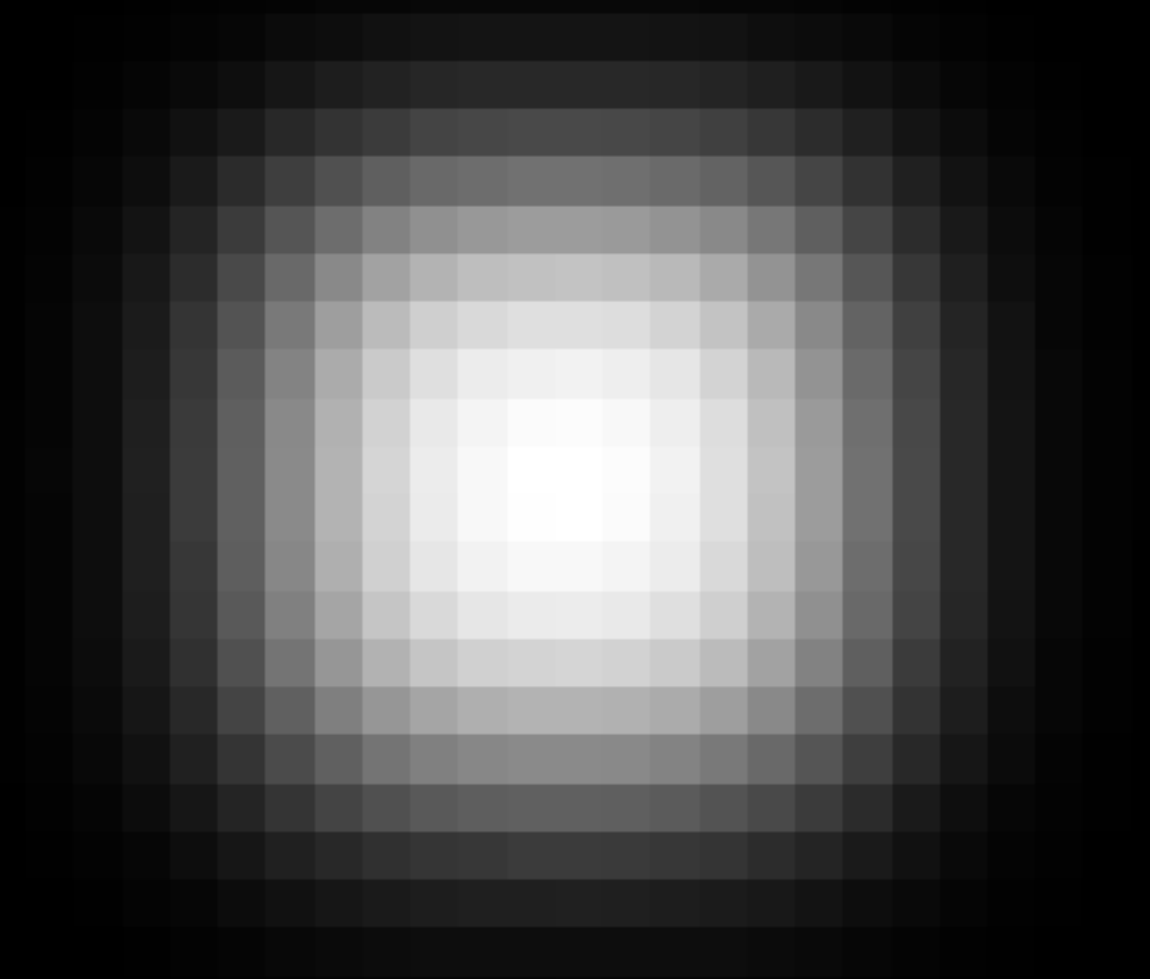}
    }
    \subfigure[]{
        \includegraphics[width=.22\columnwidth,trim=90 70 155 80,clip]{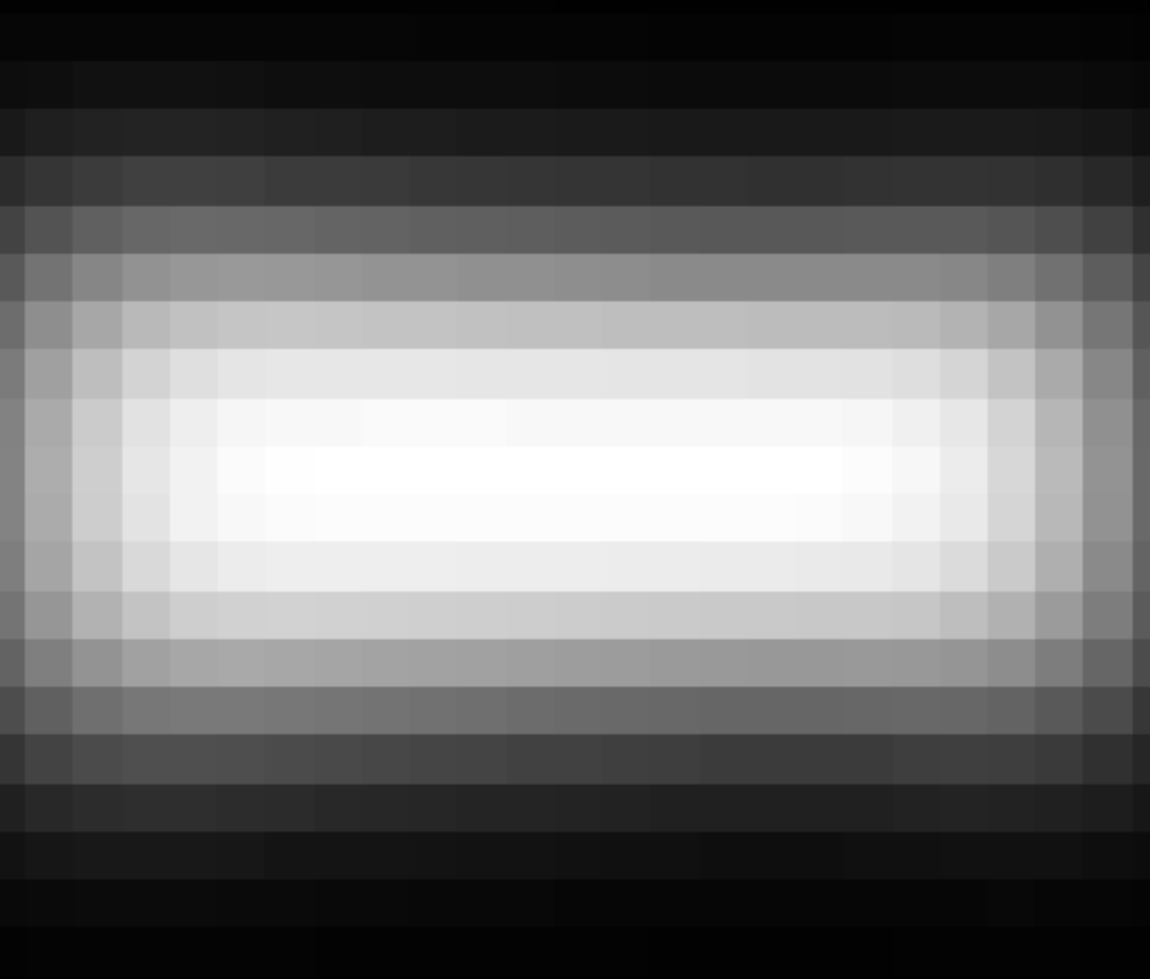}
    }
    \subfigure[]{
        \includegraphics[width=.22\columnwidth,trim=90 70 155 80,clip]{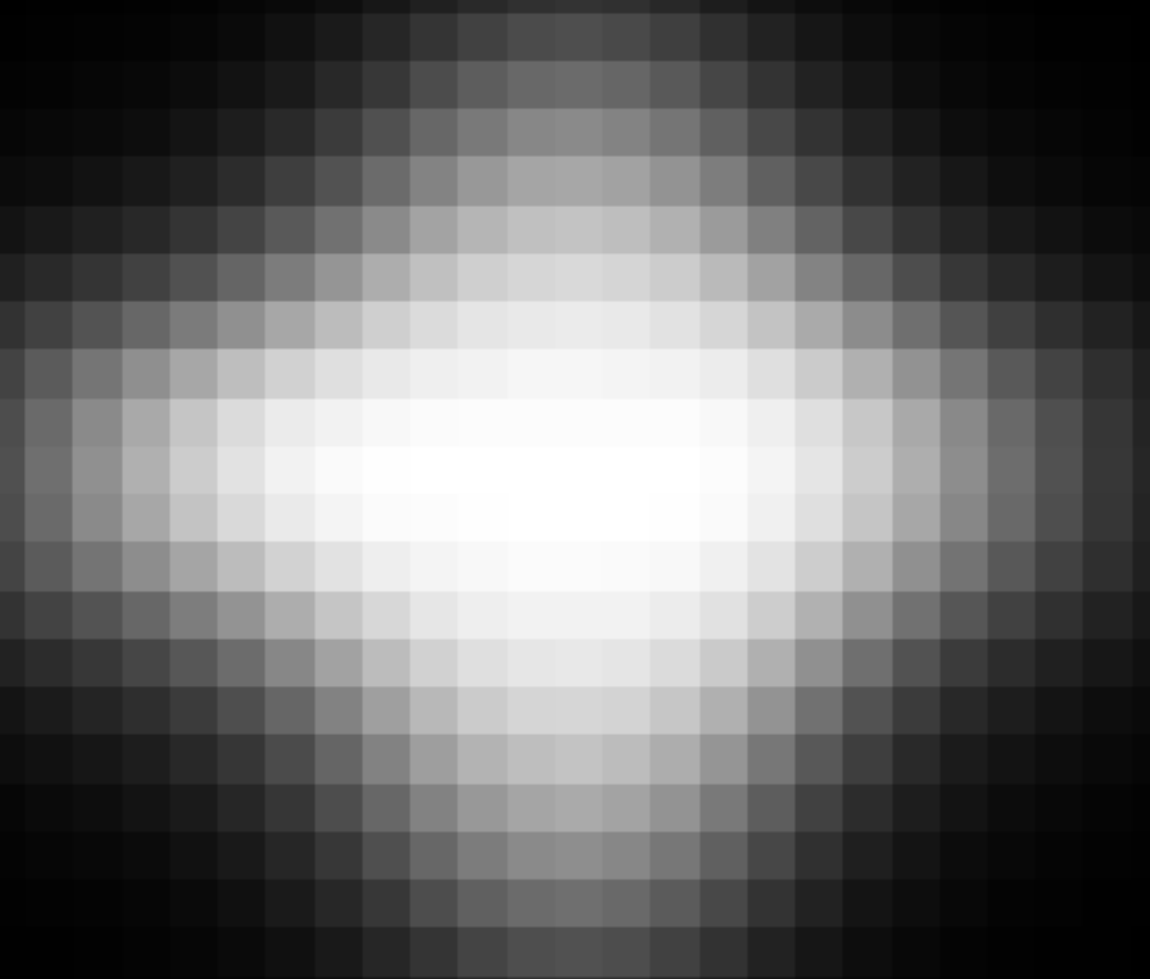}
    }
    \subfigure[]{
        \includegraphics[width=.22\columnwidth,trim=90 70 155 80,clip]{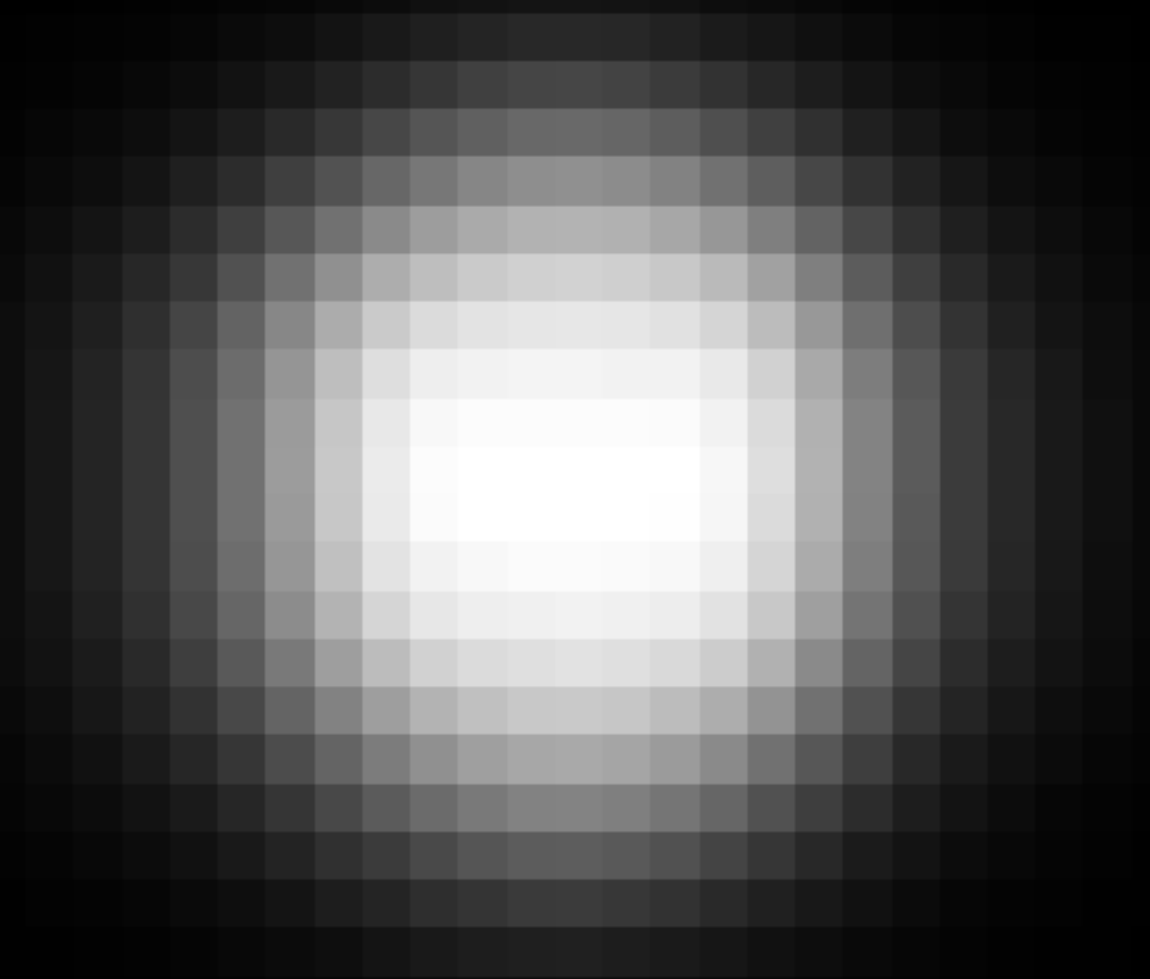}
    }
    \newline
    \subfigure[]{
        \includegraphics[width=.22\columnwidth,trim=90 50 75 50,clip]{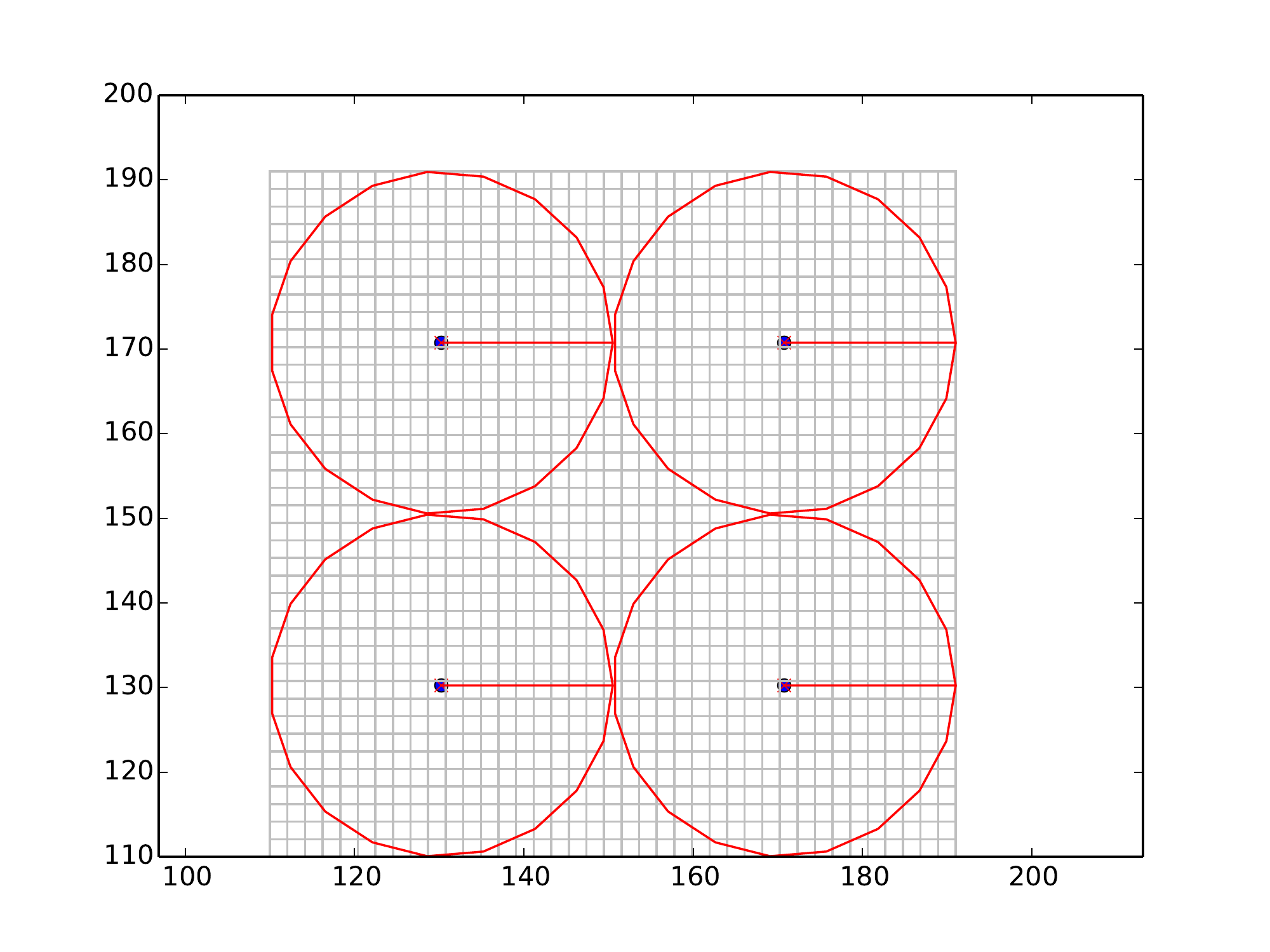}
    }
    \subfigure[]{
        \includegraphics[width=.22\columnwidth,trim=90 50 75 50,clip]{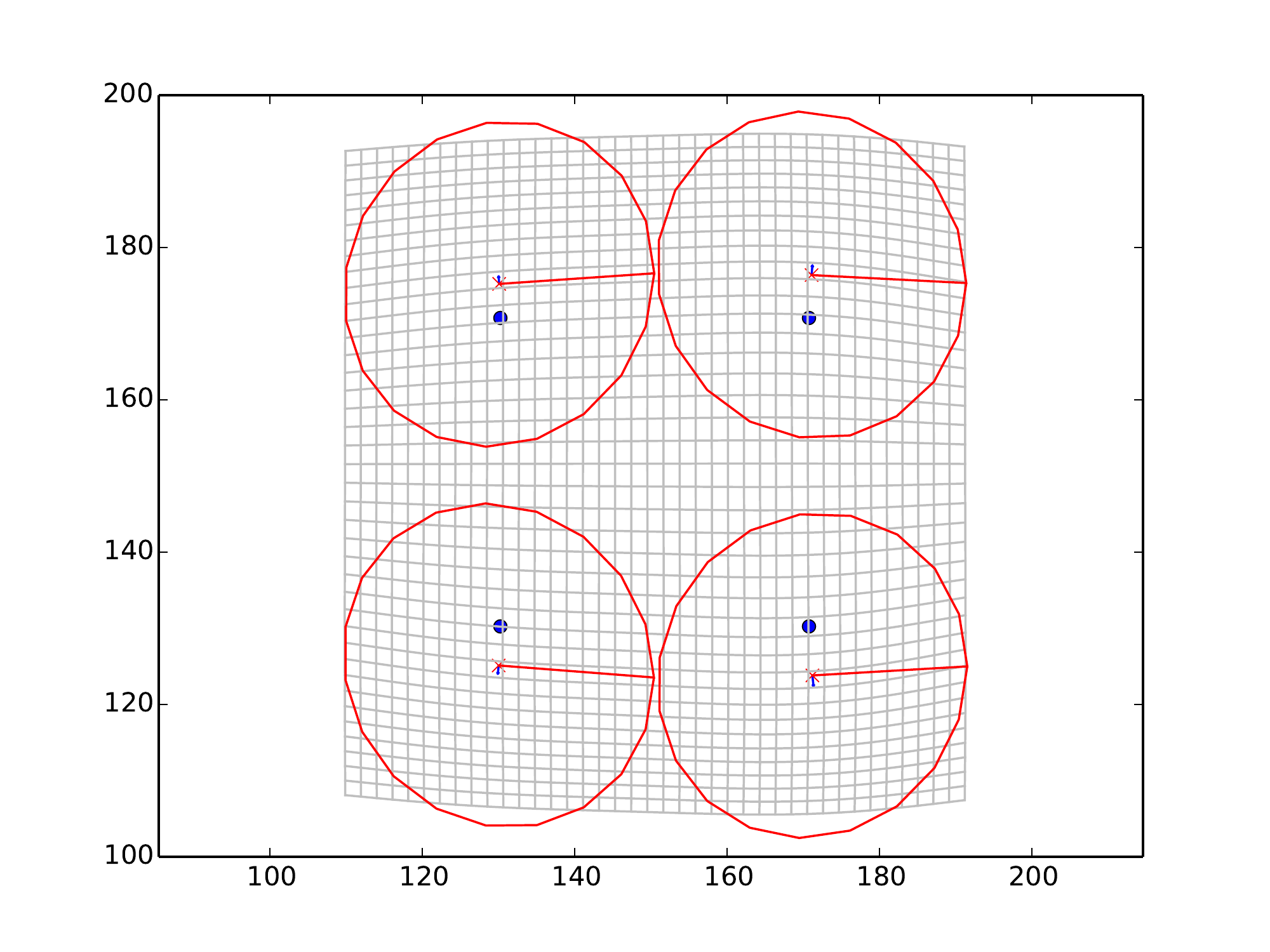}
    }
    \subfigure[]{
        \includegraphics[width=.22\columnwidth,trim=90 50 75 50,clip]{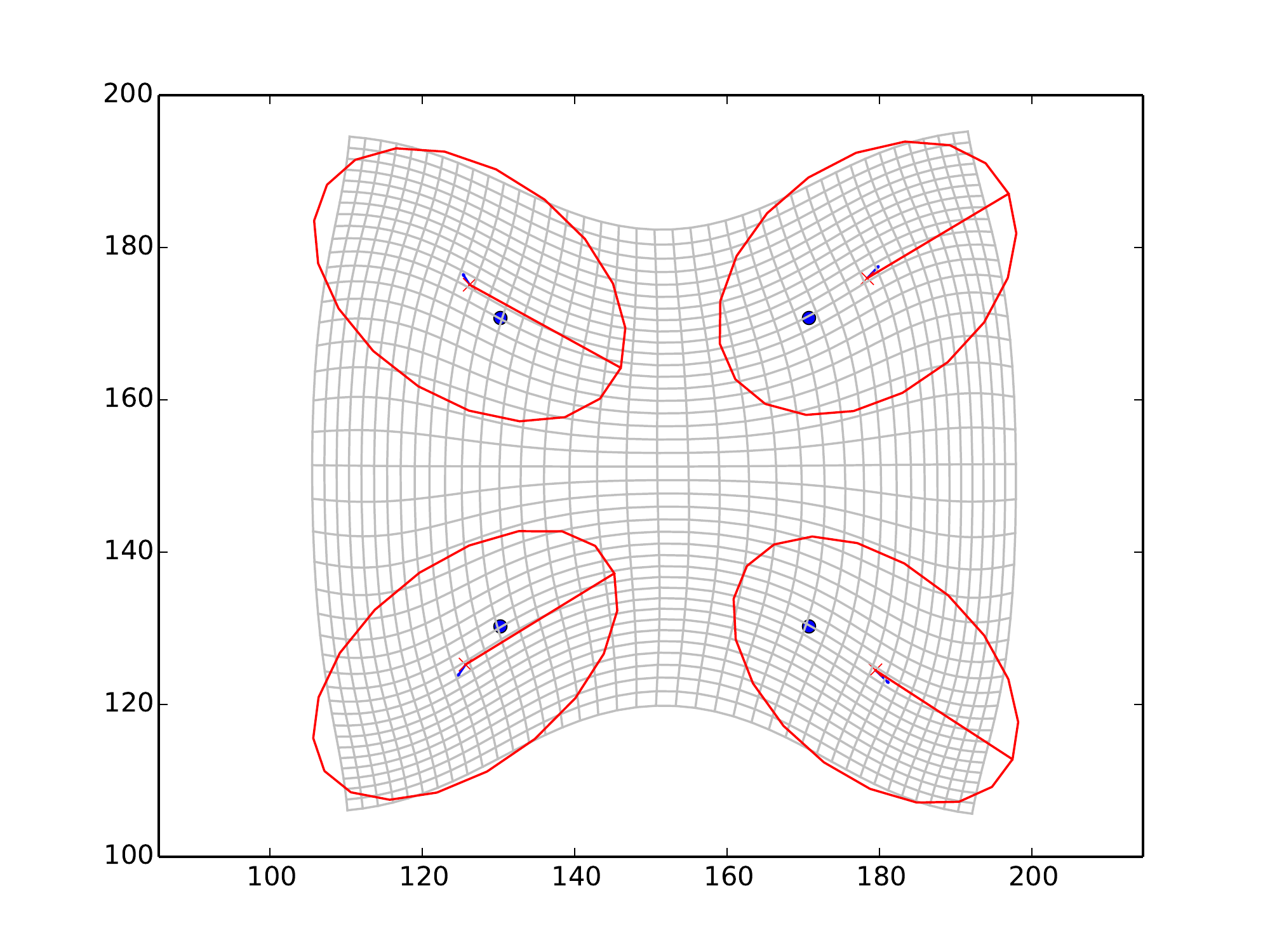}
    }
    \subfigure[]{
        \includegraphics[width=.22\columnwidth,trim=90 50 75 50,clip]{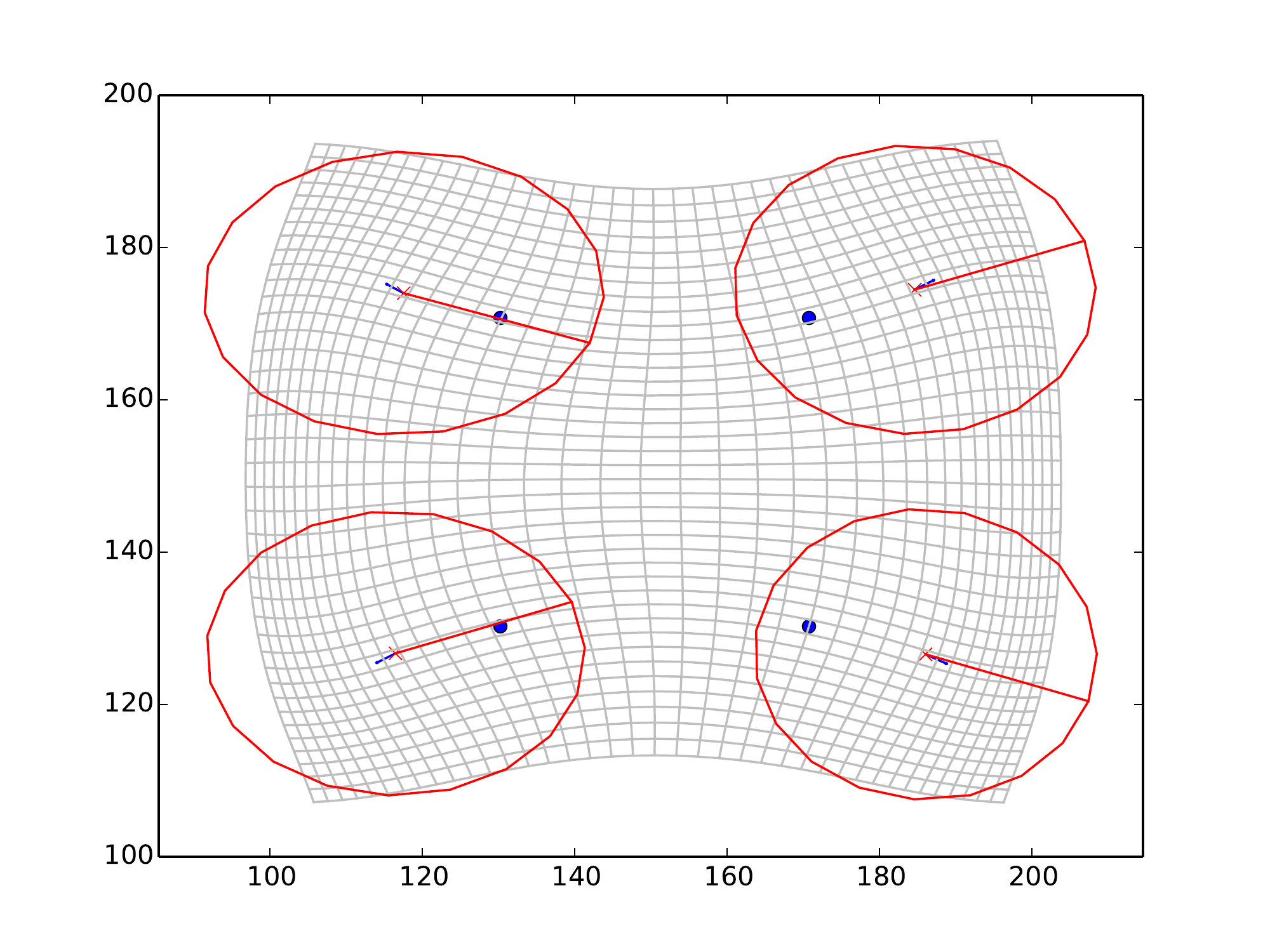}
    }
  \end{center}
  \caption{Matching moving images (b-d) to fixed image (a) using four jet-particles (blue points). Enlarged fixed image and moving images after warping (e-h). Corresponding deformations of an initially square grid (i-l). (b/f/j) Order 0; (c/g/k) order 1; (d/h/l) order 2. Red crosses mark location of jet-particles in moving images after matching, green boxes deformed by the warp Jacobian at the particle positions. Moving images at the red crosses should match fixed image at blue dots; second row images should match the fixed image (a/e). 
  }
  \label{fig:barimages1}
\end{figure}

\begin{figure}[t]
  \begin{center}
    \subfigure[]{
        \includegraphics[width=.22\columnwidth,trim=175 150 240 150,clip,angle=0]{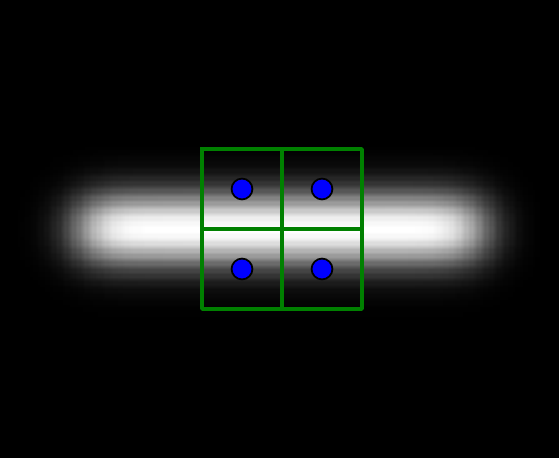}
    }
    \subfigure[]{
        \includegraphics[width=.22\columnwidth,trim=175 150 240 150,clip,angle=0]{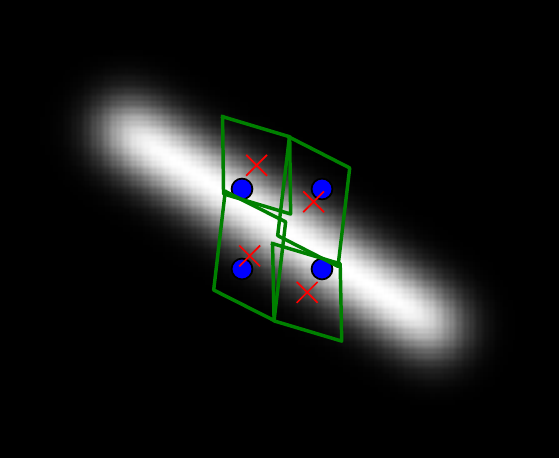}
    }
    \subfigure[]{
        \includegraphics[width=.22\columnwidth,trim=175 150 240 150,clip,angle=0]{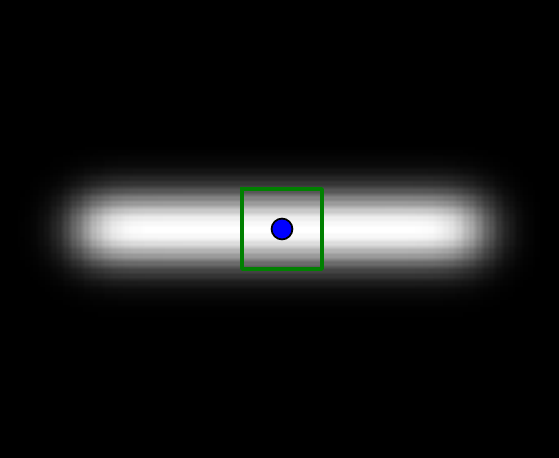}
    }
    \subfigure[]{
        \includegraphics[width=.22\columnwidth,trim=175 150 240 150,clip,angle=0]{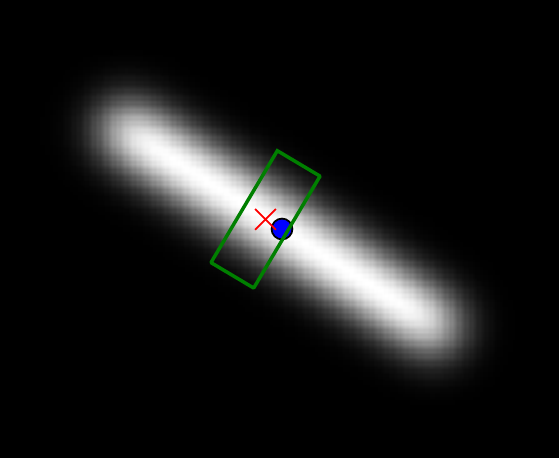}
    }
    \newline
    \subfigure[]{
        \hspace{.20cm}
        \includegraphics[width=.17\columnwidth,trim=135 130 200 130,clip,angle=0]{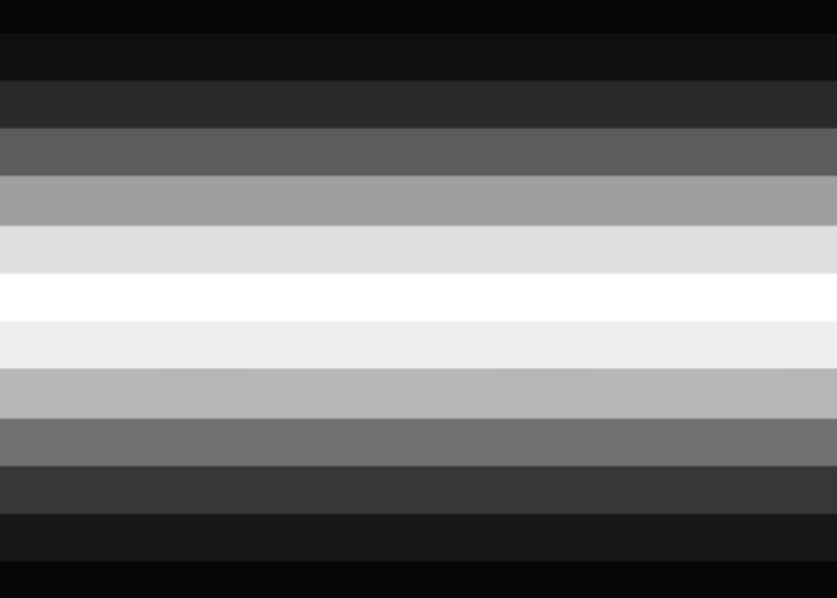}
        \hspace{.20cm}
    }
    \subfigure[]{
        \hspace{.20cm}
        \includegraphics[width=.17\columnwidth,trim=135 130 200 130,clip,angle=0]{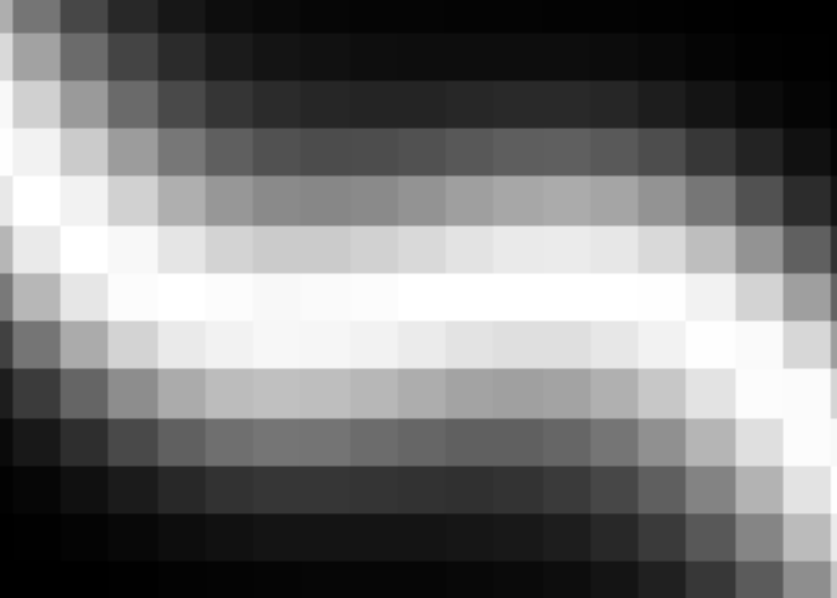}
        \hspace{.20cm}
    }
    \subfigure[]{
        \hspace{.20cm}
        \includegraphics[width=.17\columnwidth,trim=135 130 200 130,clip,angle=0]{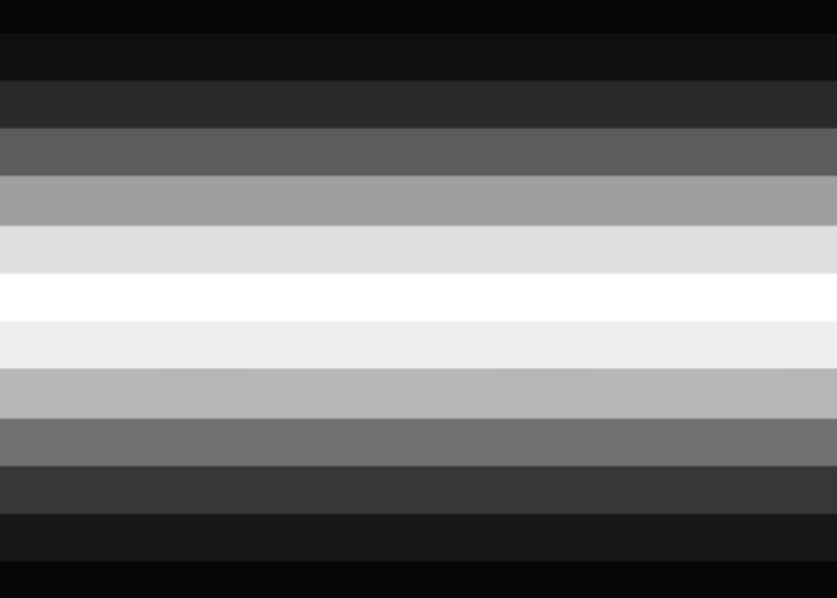}
        \hspace{.20cm}
    }
    \subfigure[]{
        \hspace{.20cm}
        \includegraphics[width=.17\columnwidth,trim=135 130 200 130,clip,angle=0]{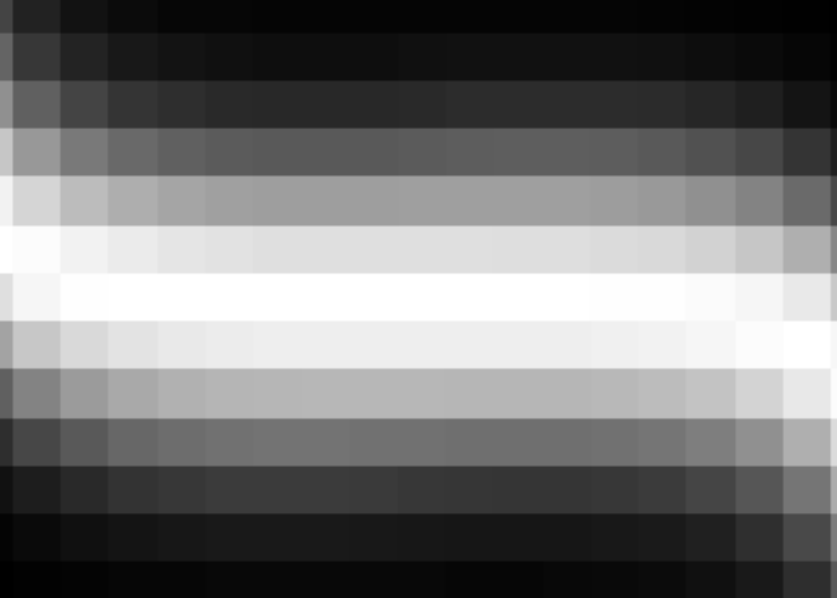}
        \hspace{.20cm}
    }
    \newline
    \subfigure[]{
        \includegraphics[width=.21\columnwidth,trim=220 50 205 50,clip,angle=90]{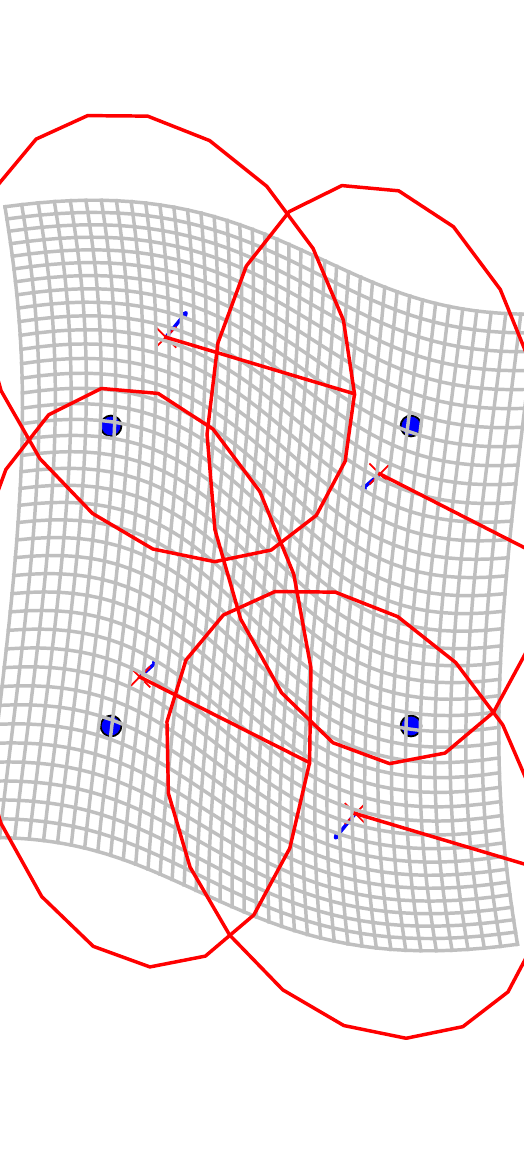}
    }
    \subfigure[]{
        \includegraphics[width=.21\columnwidth,trim=220 50 205 50,clip,angle=90]{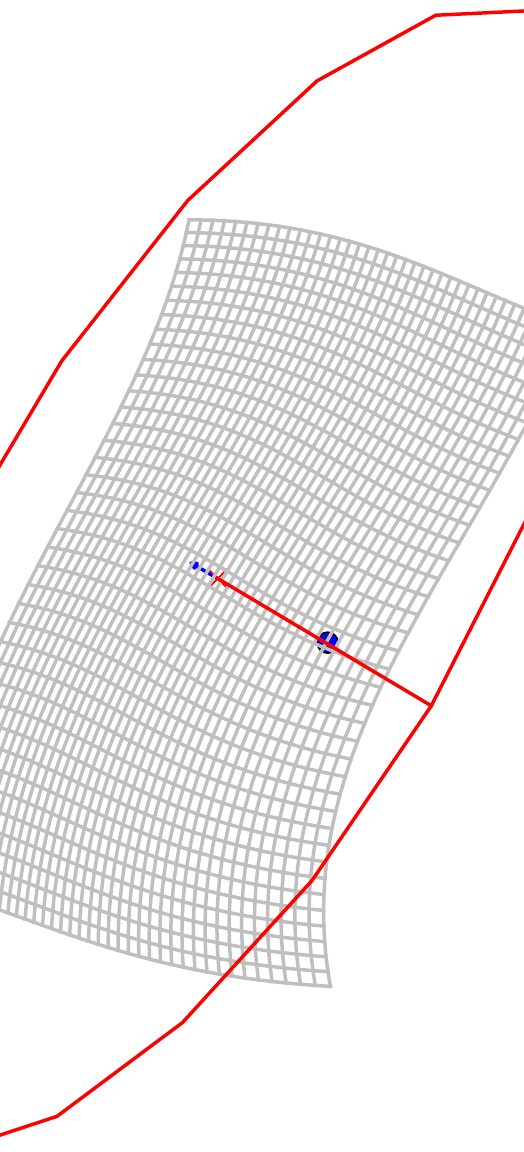}
    }
  \end{center}
  \caption{First order (linear/affine) deformations can be matched with multiple zeroth
    order jet-particles (a,b) or one first order jet-particle (c,d). A rotated bar (b/d) is
    matched to a bar (a/c). The warps that transform the moving images (b/d) 
    to the fixed images (a/c) are applied to initially square grids (i/j)
    (rotated $90^\circ$).
    Red circles are deformed with warp derivative at the particle positions.
    (e-h) shows enlarged fixed and warped moving images. The amount of bending
    at the edges (f/h) is a function of the kernel size.
}
  \label{fig:barimages2}
\end{figure}

\subsection{Matching Simple Structures}
With the following set of examples, we wish to illustrate the effects of
including second order information in the matching term approximation. We
visualize this using simple test images. In all examples, we will employ the
approximations $F_h^{(k)}$ for $k=0,2$. In addition, we will match using
only zeroth and first order information with a matching term that results from
dropping the second order terms from $F_h^{(2)}$. While this approximation does
not arise naturally from a Taylor expansion of $F$, it allows visualization
of the differences between including first and second order image information in
the match.

\begin{figure}[t]
  \begin{center}
    \subfigure[]{
        \includegraphics[width=.17\columnwidth,trim=175 150 240 150,clip]{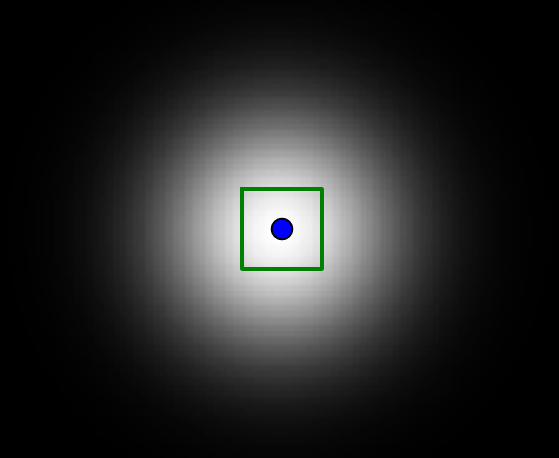}
    }
    \subfigure[]{
        \includegraphics[width=.17\columnwidth,trim=175 150 240 150,clip]{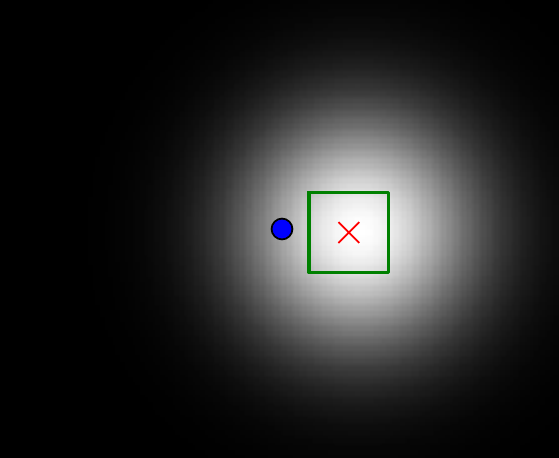}
    }
    \subfigure[]{
        \includegraphics[width=.17\columnwidth,trim=175 150 240 150,clip]{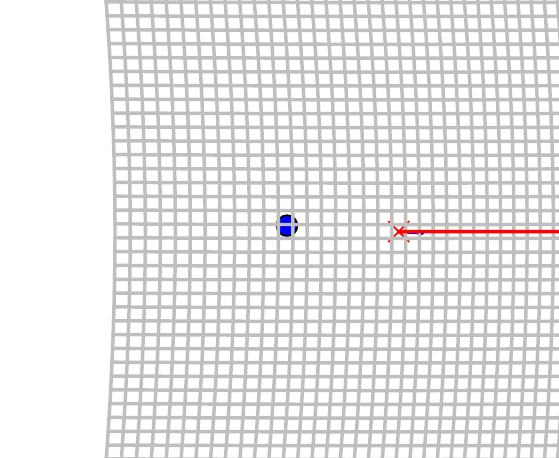}
    }
    \subfigure[]{
        \includegraphics[width=.17\columnwidth,trim=175 150 240 150,clip]{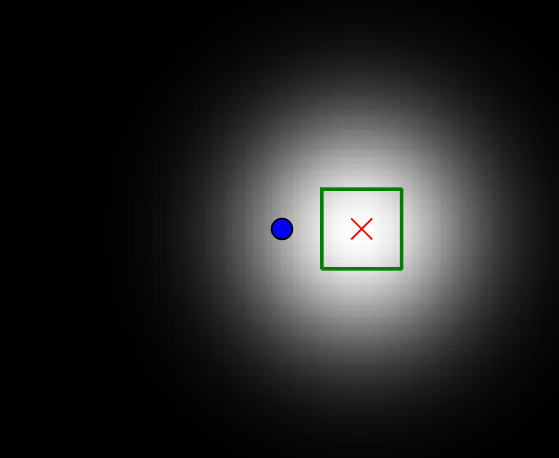}
    }
    \subfigure[]{
        \includegraphics[width=.17\columnwidth,trim=175 150 240 150,clip]{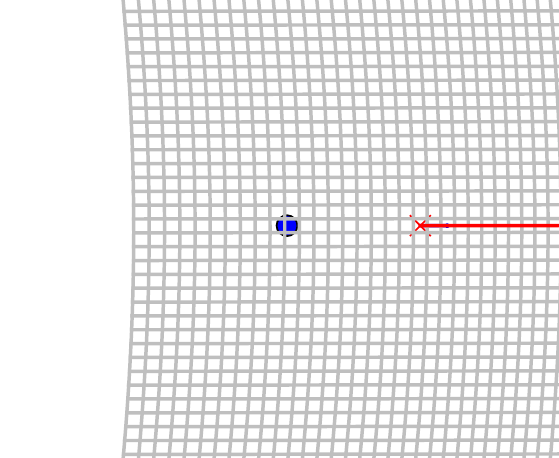}
    }
  \end{center}
  \caption{Without higher order features, 2nd order jet-particles do not change the
  match: A blob (a) is translated and matched in moving images (b,d) with red
crosses marking positions of jet-particles after match. Grids (c,e) illustrate the deformations
that are equivalent for 0th order (b,c) and 2nd order (d,e).}
  \label{fig:translation}
\end{figure}

In Figure~\ref{fig:barimages1}, a bar (moving image) is matched to a square
(fixed image). The figure shows how four jet-particles move from their positions on a
grid in the fixed image (a) to positions in the moving image that contain
features matching
the fixed image up to the order of the approximation. For zeroth order (b), only
pointwise intensity is matched and the jet-particles move vertically (red crosses)
resulting in only a slight deformation. With
first order matching (c), the jet-particles locally rotate the domain (warp Jacobian
matrices shown with green boxes) to account for the image gradient at the
corners of the square. This produces a diamond-like shape.
With second order (d), the corners are matched and the jet-particles move towards the corners of the moving image bar. 
The middle
row shows the warped moving images enlarged. The second order 
match (h) is close to the fixed image (a) while both first and zeroth order
fail to produce satisfying matches.
\begin{figure}[t!]
  \begin{center}
      \subfigure[fixed, full image]{
        \includegraphics[width=.20\columnwidth,trim=80 80 80 80,clip,angle=90]{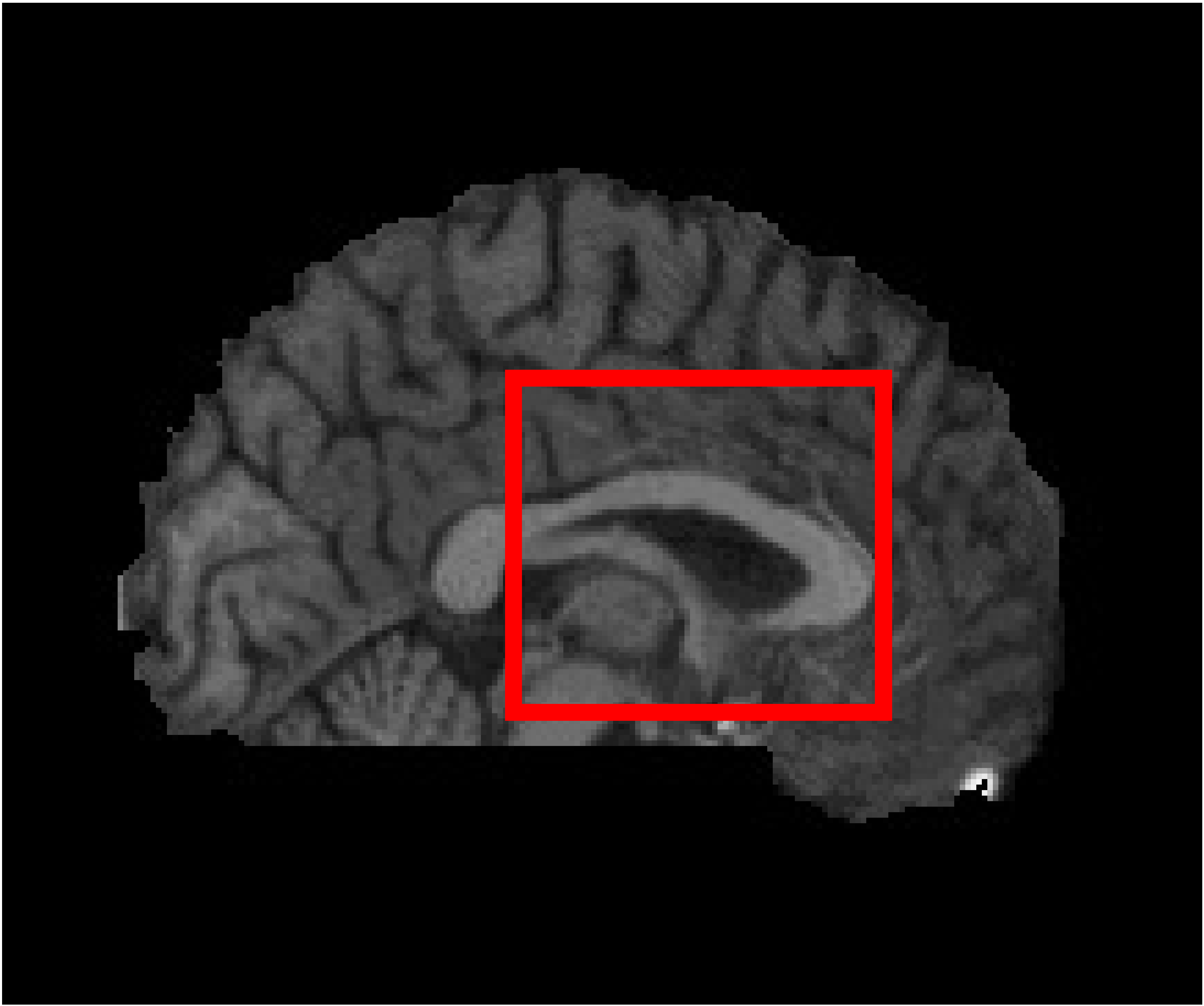}
      }
      \subfigure[fixed, region]{
        \includegraphics[width=.20\columnwidth,trim=90 50 155 50,clip]{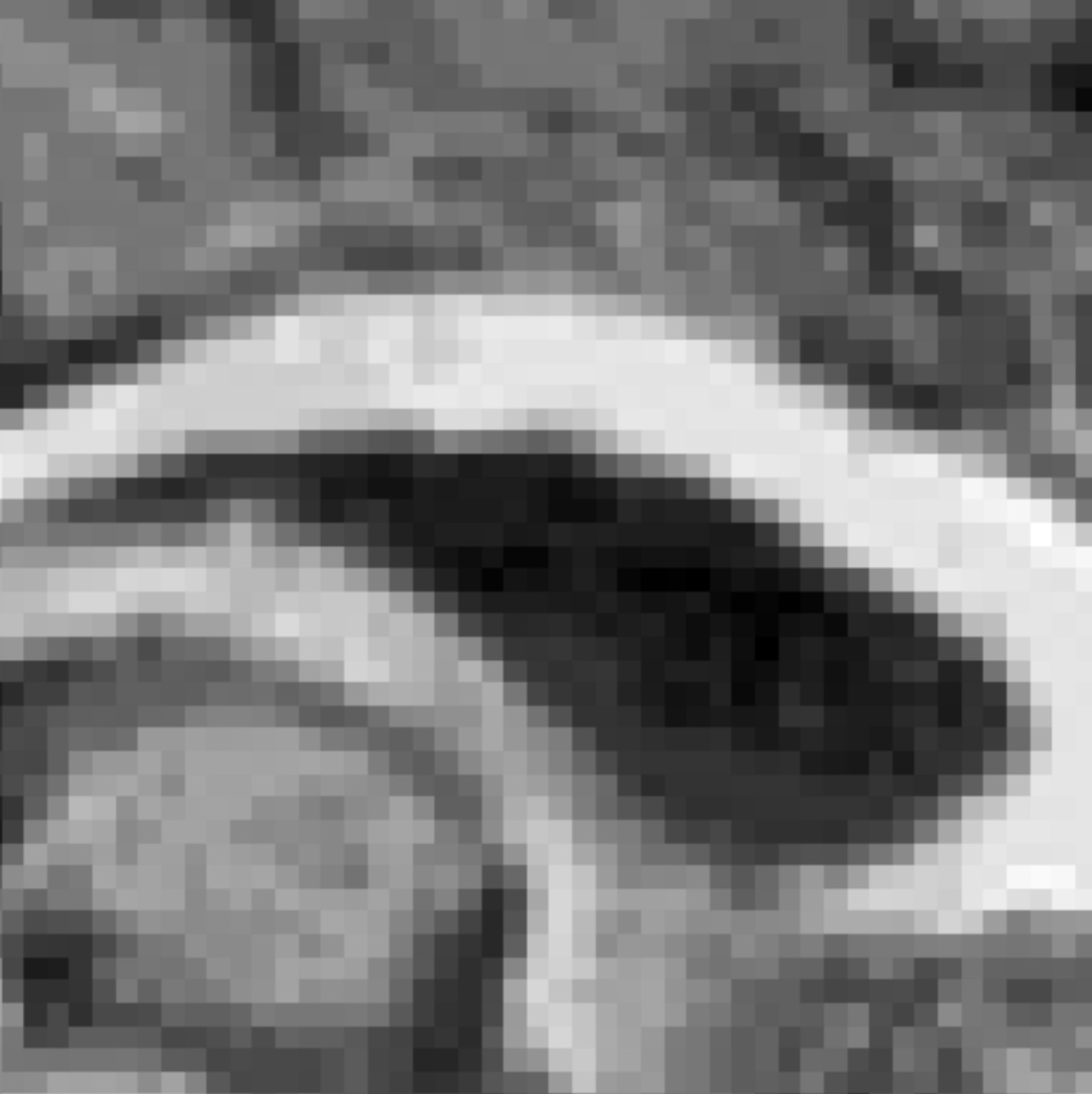}
      }
      \subfigure[moving, full image]{
        \includegraphics[width=.20\columnwidth,trim=80 80 80 80,clip,angle=90]{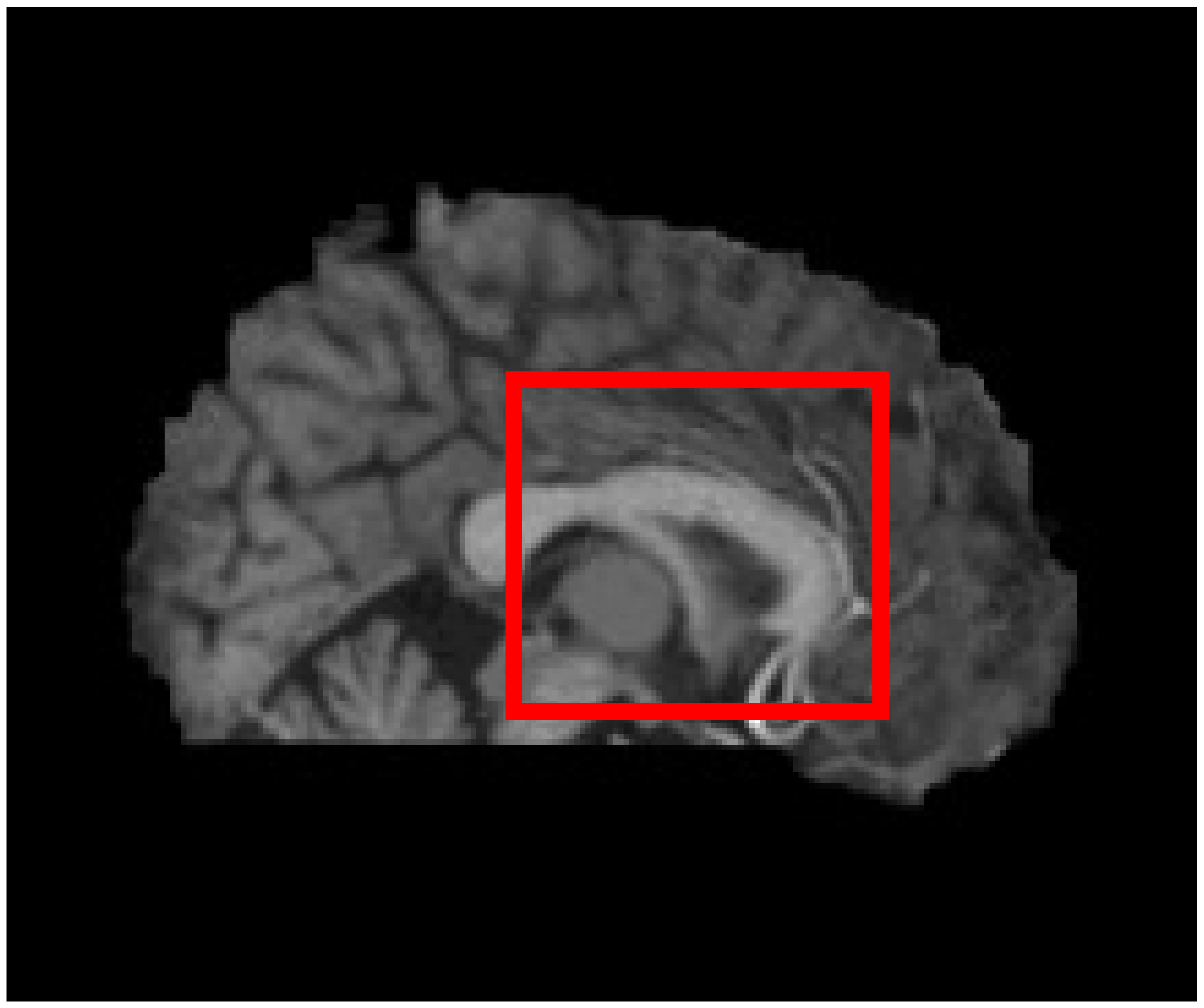}
      }
      \subfigure[mov., region]{
        \includegraphics[width=.20\columnwidth,trim=90 50 155 50,clip]{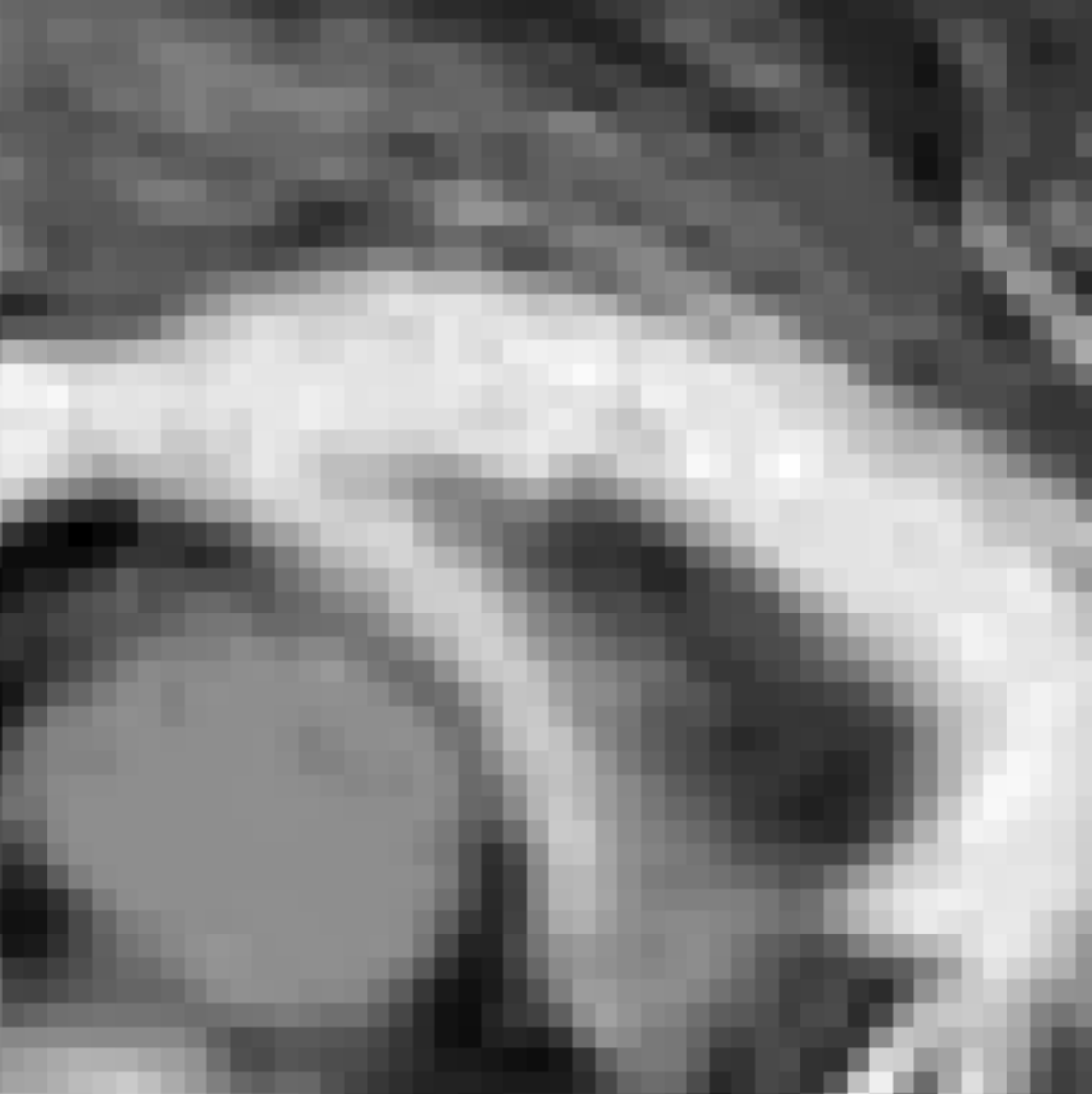}
      }
      \newline
      \subfigure[2nd order, 9 jet-particles]{
        \includegraphics[width=.20\columnwidth,trim=90 50 155 50,clip]{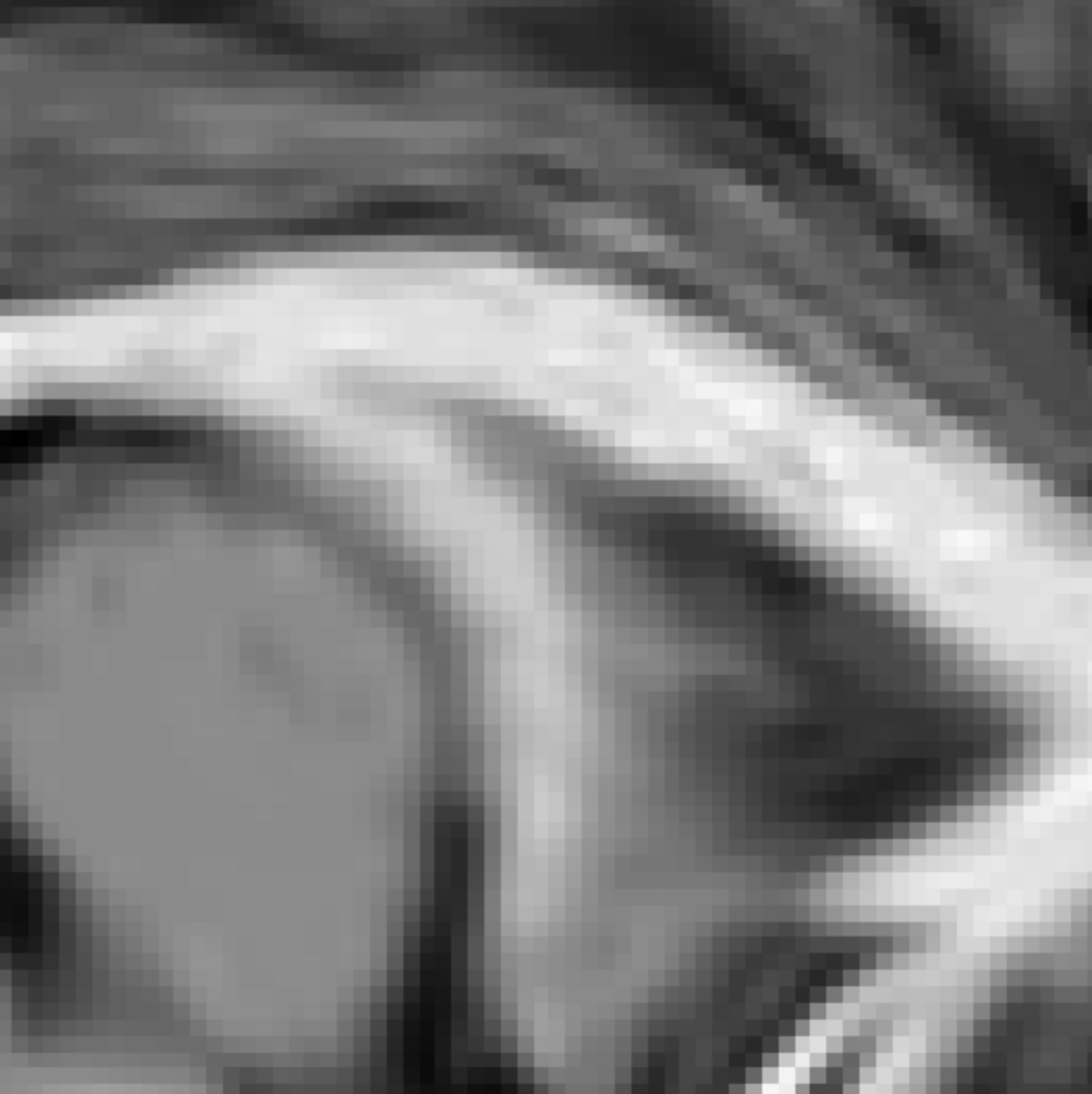}
      }
      \subfigure[2nd order, 16 jet-particles]{
        \includegraphics[width=.20\columnwidth,trim=90 50 155 50,clip]{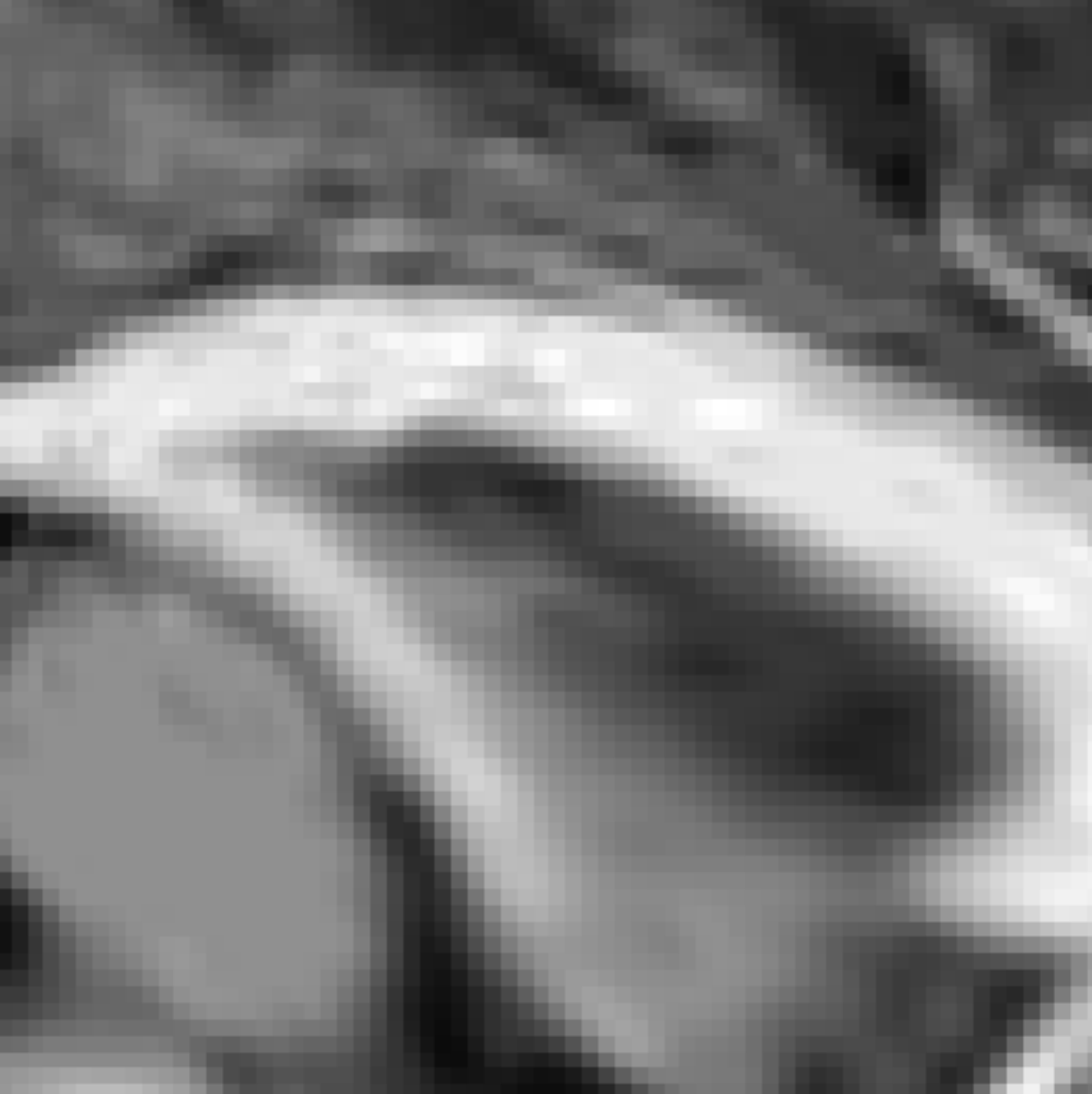}
      }
      \subfigure[2nd order, 64 jet-particles]{
        \includegraphics[width=.20\columnwidth,trim=90 50 155 50,clip]{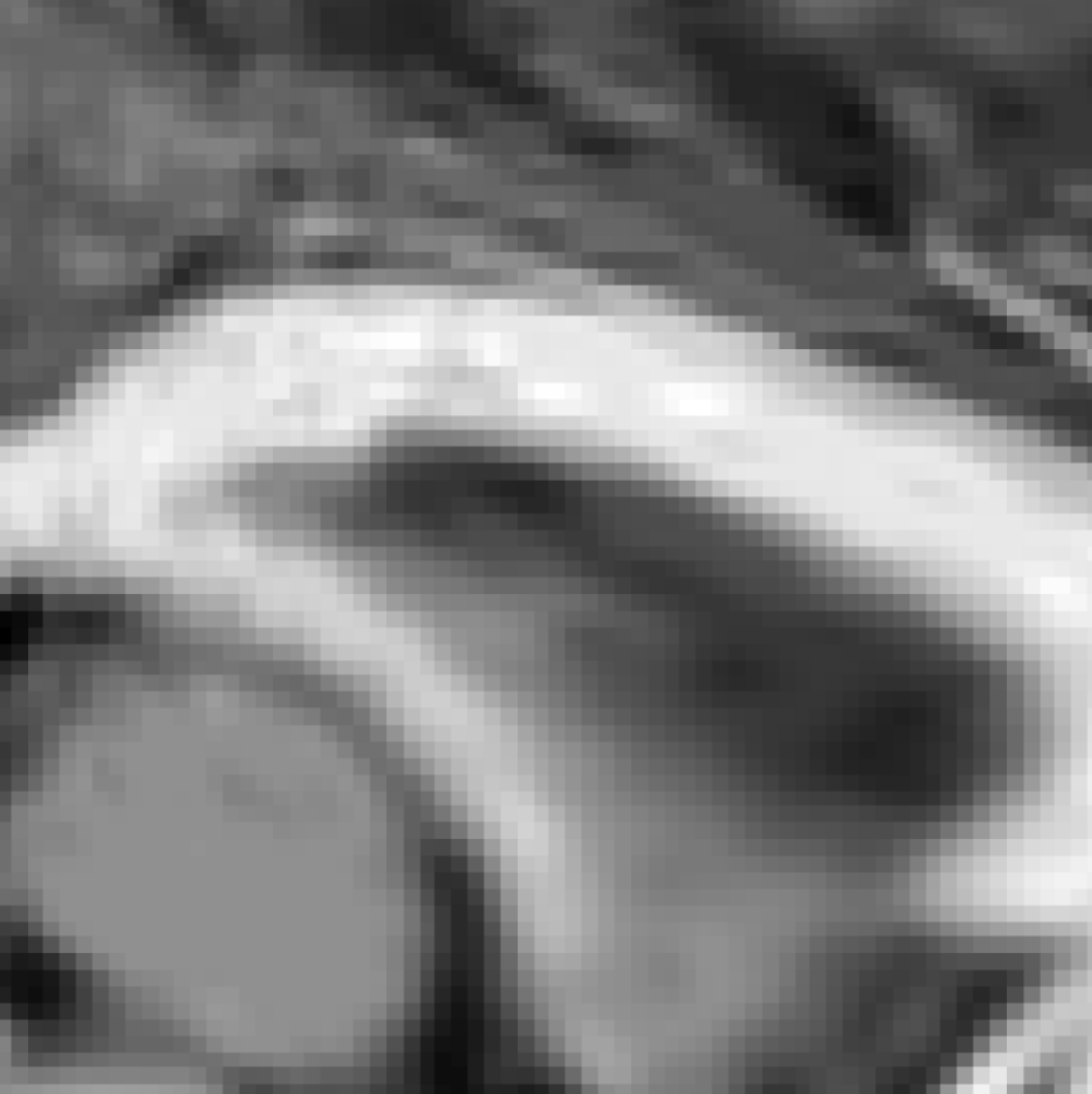}
      }
      \newline
      \subfigure[0th order, 9 jet-particles]{
        \includegraphics[width=.20\columnwidth,trim=90 50 155 50,clip]{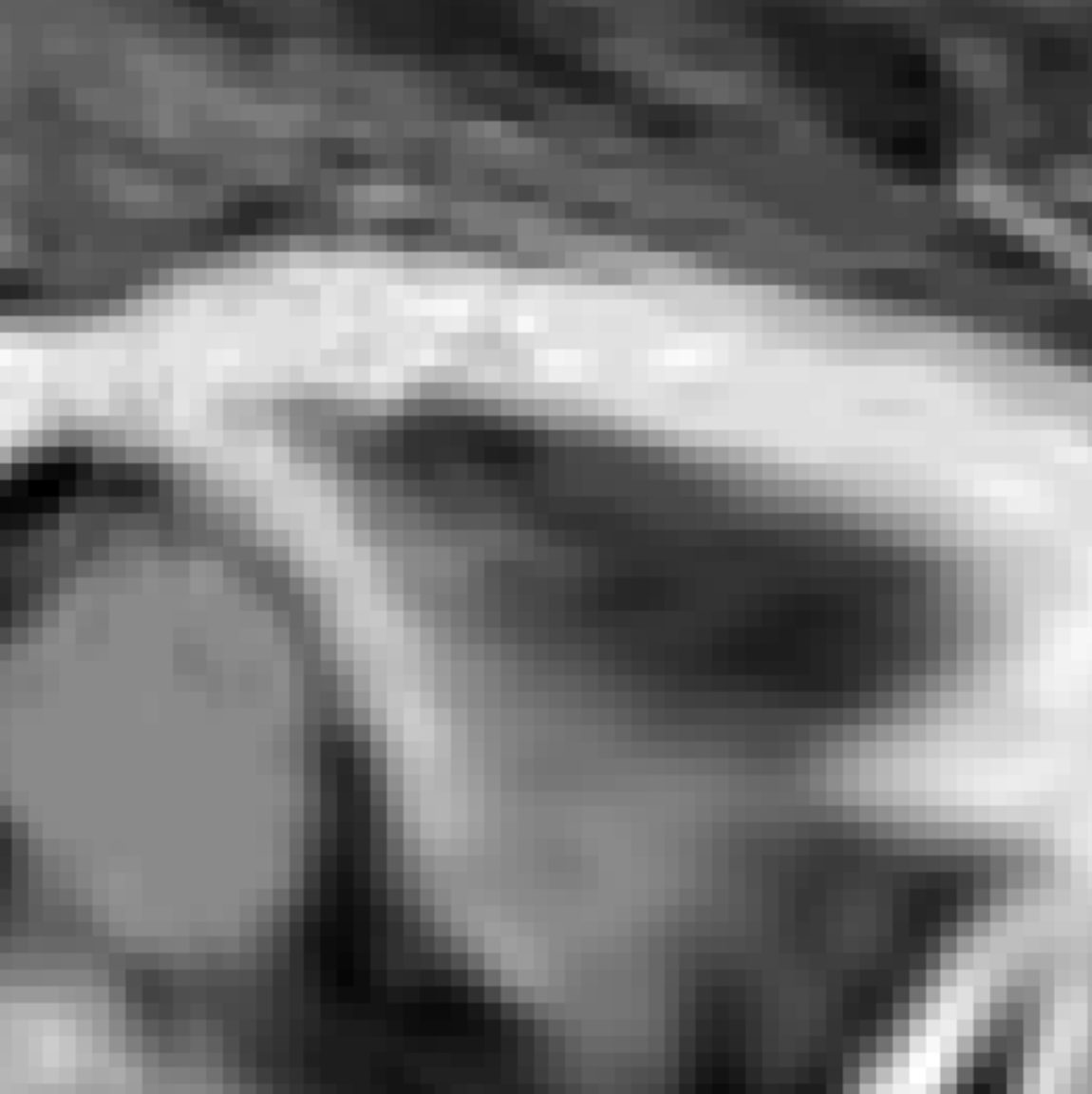}
      }
      \subfigure[0th order, 16 jet-particles]{
        \includegraphics[width=.20\columnwidth,trim=90 50 155 50,clip]{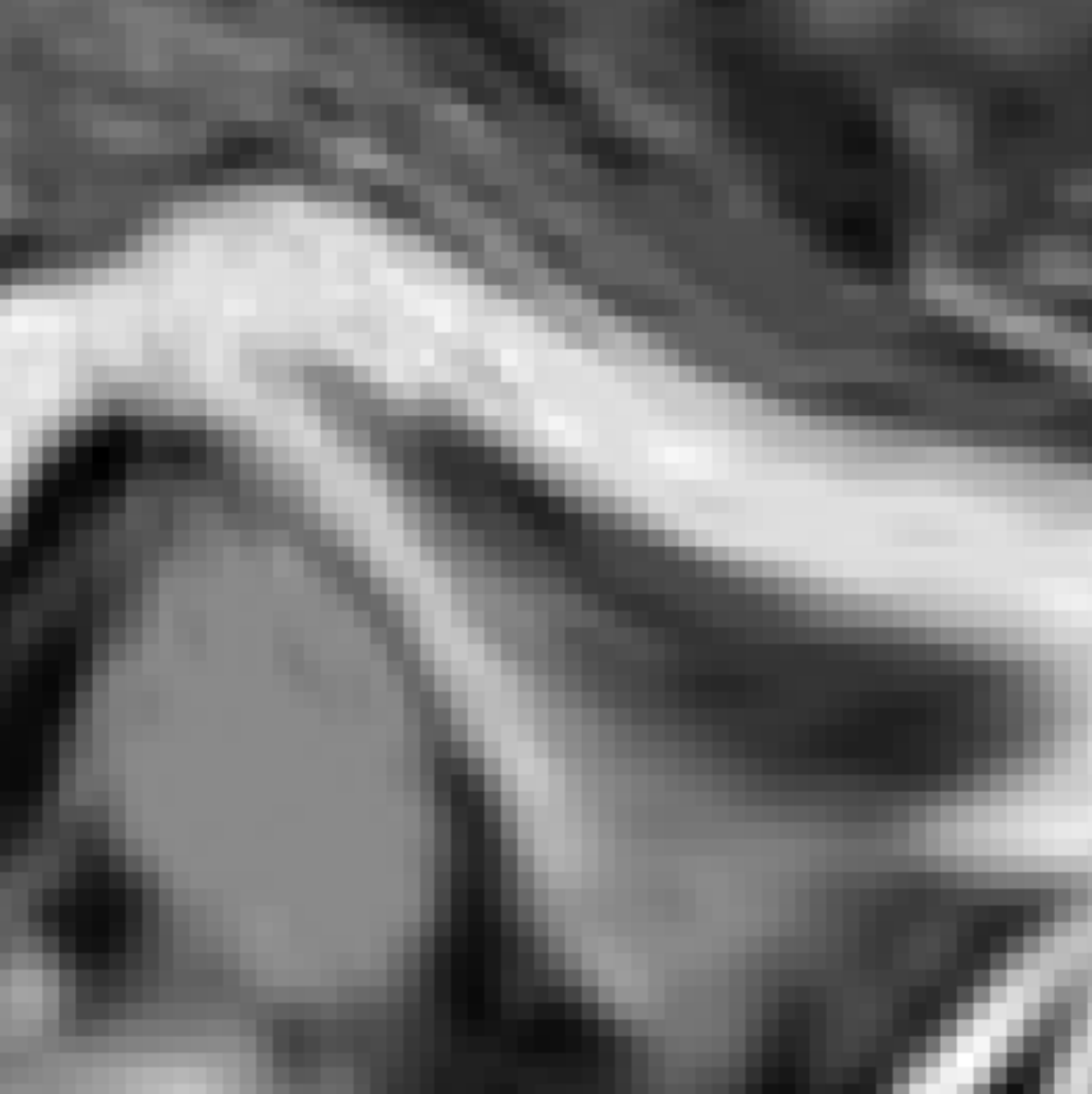}
      }
      \subfigure[0th order, 64 jet-particles]{
        \includegraphics[width=.20\columnwidth,trim=90 50 155 50,clip]{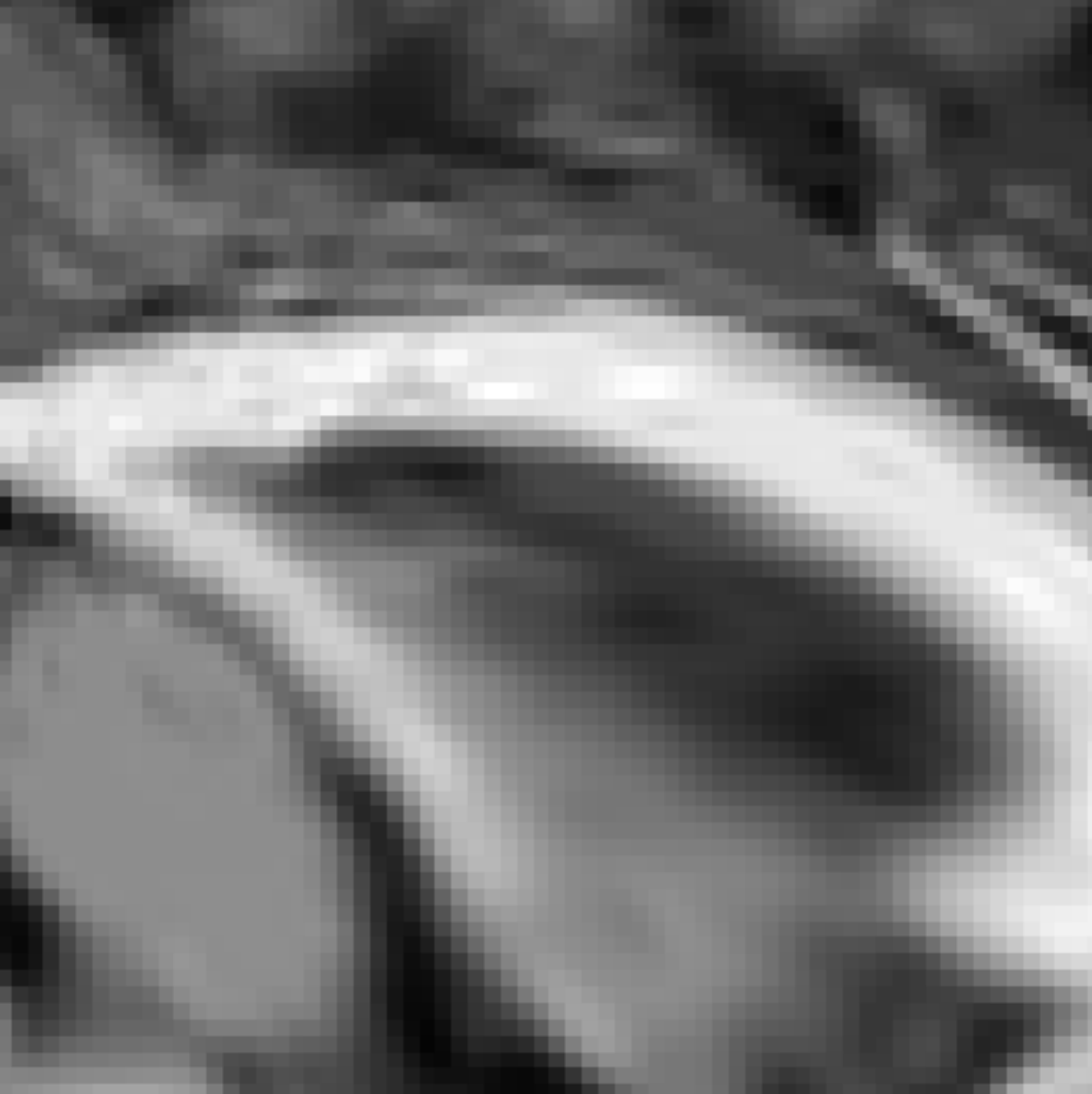}
      }
      \newline
  \end{center}
  \caption{2D registration of MRI slices, (a-b) fixed image, (c-d) moving
    image, red boxes: regions to be matched. Lower rows: matching results using 
    2nd order jet-particles (e-g), 0th order jet-particles (h-j). Images in lower rows should be close to (b).
    With 9 2nd order jet-particles (3 per axis), the moving
    image approaches the fixed. The match is visually good with 16 jet-particles (4 per
    axis). The ventricle region can equivalently be inflated with 64 0th order jet-particles.}

  \label{fig:slices}
\end{figure}

Figure~\ref{fig:barimages2} shows the result of matching images differing by an
affine transformation with either one first order jet or multiple zeroth order
jet-particles. While three zeroth order jet-particles can approximate a first order deformation in
2D, four particles are used to produce a symmetric picture. The warp Jacobians deform 
the initially square green boxes displayed at the jet
positions. The resulting warps in both cases approximate an affine
transformation.

With translation only, including second order information in the match
does not change the result as illustrated in Figure~\ref{fig:translation} where
the match is performed on an image and a translated version of the image.

\subsection{Real image data}
We illustrate the effect of the increased order on real images by matching two mid-sagittal slices of 3D MRI
from the MGH10 dataset \cite{klein_evaluation_2009}. In Figure~\ref{fig:slices}, red boxes mark the ventricle
area of the brain on which the matching is performed. We perform the match with
9 jet-particles (3 per axis), 16 jet-particles (4 per axis) and 64 jet-particles (8 per axis) and $k=0,2$.
With 9 second order jet-particles (e), the moving image (d) approaches the fixed (b). A
visually good match is obtained with 16 or more jet-particles. 9 and 16 zeroth order jet-particles are
not sufficient to correctly encode the expansion of the ventricle. With 64 zeroth order
jet-particles, the transformed image is close to the results of the second order matches.

\section{Conclusion and Future Work}
A priori, the LDDMM framework of image registration poses an optimization
problem on the space of diffeomorphism. Here, we introduced a family of discretized
cost functions on a finite dimensional phase space that can be minimized numerically.
The solutions of the discretized problem can be related to solutions
of the full infinite-dimensional problem with $O(h^{d+k})$ accuracy,
where $h$ is a grid spacing and $k$ is the order of approximation.


We provided numerical examples of deformations parametrized by zeroth, first,
and second order jet-particles, and we showed examples of the higher order convergence of the
similarity measure. The higher-order similarity measure allows matching of
higher order features, and we use this fact to register various shapes and
images with low numbers of jet-particles.

Representing a $C^k$ image requires much less information than
representing a $C^0$ image.  Heuristically, the impact of this for computation is that we may use
different techniques to approximate and advect smooth images with a sparse set of parameters.
The higher-order accuracy schemes here constitutes a particular example of using
reduction by symmetry to remove redundant information, and specialize advection to the
data at hand.
In this case, we reduce the dimensionality from infinite to finite for a given discretization,
and we specialize the discretization to $C^2$ images.

While the applicability of this specialization is limited to images of sufficient regularity,
the bigger point of this article is the notion of tailoring discretizations to data.
This approach is applicable for reducing the dimensionality of data beyond images. 
For example, accurate discretizations of curves with tangents, surfaces
with tangent planes, and higher-order tensors can be derived with corresponding
reduction in dimensionality. The present framework thus points
to a general approach for higher-order accurate discretizations of 
general classes of matching problems.  Future work will constitute testing these 
areas of wider applicability.

\section*{Acknowledgments}
Henry O. Jacobs is supported by the European Research
Council Advanced Grant 267382 FCCA.
Stefan Sommer is supported by the Danish Council for Independent Research with
the project ``Image Based Quantification of Anatomical Change''.
We would like to thank Klas Modin for summarizing the literature on
analysis of the EPDiff equation, especially in regards to smoothness of initial conditions.

\appendix

\section{Equations of motion}
\label{sec:eom}

The equations of motion are expressible as Hamiltonian
equations with respect to a non-canonical Poisson bracket.
If we denote $q^{(0)}$ simply by $q$ and $p^{(0)}$ simply
by $p$ then the Hamiltonian is
\begin{align} \label{eq:Hamiltonian}
  H(q,p,\mu^{(1)},\mu^{(2)}) =& \frac{1}{2} p_{i \alpha} p_{j \beta}
  K^{\alpha \beta}(q_i - q_j) - p_{i \alpha} [\mu^{(1)}_j]
  \indices{_{\beta}^{\gamma}} \partial_\gamma K^{\alpha \beta}( q_i -
  q_j) \\
  &+ p_{i \alpha} [\mu^{(2)}_j] \indices{_\beta^{\gamma
      \delta}} \partial_{\delta \gamma}K^{\alpha \beta}( q_i - q_j) -
  \frac{1}{2} [\mu_i^{(1)}] \indices{_{\alpha}^\delta} [\mu_j^{(1)} ]
  \indices{_\beta^\gamma} \partial_{\delta \gamma} K^{\alpha
    \beta}(q_i -q_j) \\
  &+ [\mu_i^{(1)}]\indices{_\alpha^\epsilon} [\mu_j^{(2)}]
  \indices{_\beta^{\gamma \delta}} \partial_{\epsilon \gamma \delta}
  K^{\alpha \beta}( q_i - q_j) \\
  &+ \frac{1}{2}
  [\mu_i^{(2)}]\indices{_\alpha^{\epsilon \phi}} [\mu_j^{(2)}]
  \indices{_\beta^{\gamma \delta} } \partial_{\gamma \delta \epsilon
    \phi} K^{\alpha \beta}( q_i - q_j )
\end{align}
Where $K^{\alpha \beta}(x) = \delta^{\alpha \beta} e^{- \| x \|^2 / 2 \sigma^2}$.
Hamiton's equations are then given in short by
\begin{align}
  \dot{q} &= \frac{ \partial H}{\partial p} \label{eq:dHdp} \\
  \dot{p} &= - \frac{ \partial H}{\partial q} \label{eq:dHdq} \\
  \xi &= \frac{\partial H}{\partial \mu} \label{eq:dHdmu} \\
  \dot{\mu} &= - \ad^*_{ \xi } ( \mu ) . \label{eq:LiePoisson} 
\end{align}
More explicitly, equation \eqref{eq:dHdp} is given by
\begin{align*}
  \dot{q}_i^\alpha = p_{j \beta} K^{\alpha \beta}( q_i - q_j) -
  [\mu^{(1)}_j] \indices{_\beta^\gamma} \partial_\gamma K^{\alpha
    \beta}(q_i - q_j) + [\mu^{(2)}_j] \indices{_\beta^{\gamma
      \delta}} \partial_{\gamma \delta} K^{\alpha \beta}(q_i - q_j)
\end{align*}
equation \eqref{eq:dHdq} is given by the sum
\begin{align*}
  \dot{p}_{i\alpha} = T^{00}_{i\alpha} + T^{01}_{i\alpha} +
  T^{02}_{i\alpha} + T^{12}_{i\alpha} + T^{11}_{i\alpha} + T^{22}_{i\alpha}
\end{align*}

Where we define the six terms in this sum as
\begin{align*}
  T^{00}_{i\alpha} =& -p_{i\gamma}
  p_{j\beta} \partial_{\alpha}K^{\gamma\beta}(q_i - q_j) \\
  T^{01}_{i\alpha} =& (p_{i\delta} [\mu_j^{(1)}]\indices{_\beta^\gamma} -
  p_{j\delta}[\mu_i^{(1)}]\indices{_\beta^\gamma}) \partial_{\gamma\alpha}K^{\delta\beta}(q_i
  - q_j) \\
  T^{02}_{i\alpha} =& - (p_{i\epsilon}
  [\mu_j^{(2)}]\indices{_\beta^{\gamma\delta}} + p_{j\epsilon}
  [\mu_i^{(2)}] \indices{_\beta^{\gamma\delta}}
  ) \partial_{\gamma\delta\alpha}K^{\epsilon\beta}(q_i - q_j) \\
  T^{12}_{i\alpha} =& -([\mu_i^{(1)}] \indices{_\phi^\epsilon}
  [\mu_j^{(2)}]\indices{_\beta^{\gamma\delta}} -
  [\mu_j^{(1)}]\indices{_\phi^\epsilon} [\mu_i^{(2)}]
  \indices{_\beta^{\gamma\delta}}) \partial_{\epsilon\gamma\delta\alpha}K^{\phi\beta}(q_i
  - q_j) \\
  T^{11}_{i\alpha} =& [\mu_i^{(1)}] \indices{_\epsilon^\delta}
  [\mu_j^{(1)}]\indices{_\beta^\gamma} \partial_{\delta\gamma\alpha}
  K^{\epsilon\beta}(q_i - q_j) \\
  T^{22}_{i\alpha} =& -[\mu_i^{(2)}]\indices{_\zeta^{\epsilon\phi}}
  [\mu_j^{(2)}]\indices{_\beta^{\gamma\delta}} \partial_{\epsilon\delta\gamma\phi\alpha}
  K^{\zeta\beta}(q_i - q_j)
\end{align*}

Next, we calculate the quantities $\xi^{(i)} = \partial H / \partial \mu^{(i)}$
for $i=1,2$ of equation \eqref{eq:dHdmu} to be
\begin{align*}
  [\xi_i^{(1)} ]\indices{^\alpha_\beta} =&
  p_{j,\gamma} \partial_\beta K^{\alpha \gamma}(q_i - q_j) -
  [\mu_j^{(1)}]\indices{_{\delta}^{\gamma}} \partial_{\beta \gamma}
  K^{\alpha \delta}( q_i - q_j) +
  [\mu_j^{(2)}]\indices{_{\epsilon}^{\gamma \delta} }\partial_{\beta
    \gamma \delta} K^{\alpha
    \epsilon}(q_i - q_j) \\
  [\xi_i^{(2)}]\indices{^\alpha_{\beta \gamma}} =& p_{j
    \delta} \partial_{\beta \gamma}K^{\alpha \delta}(q_i - q_j) -
  [\mu_j^{(1)}]
  \indices{_\delta^\epsilon} \partial_{\epsilon\beta\gamma}K^{\alpha\delta}(q_i
  -q_j) + [\mu_j^{(2)}]\indices{_\epsilon^{\phi\delta}}\partial_{\beta\gamma\phi\delta}K^{\alpha\epsilon}(q_i-q_j)
\end{align*}
which allows us to compute $\dot{\mu}^{(i)}$ in equation \eqref{eq:LiePoisson} as
\begin{align*}
 [\dot{\mu}^{(1)}_i ]\indices{_\alpha^\beta} =&
 [\mu_i^{(1)}]\indices{_\alpha^\gamma}
 [\xi_i^{(1)}]\indices{^\beta_\gamma} -
 [\mu_i^{(1)}]\indices{_\gamma^\beta}
 [\xi_i^{(1)}]\indices{^\gamma_\alpha} \\
& + [\mu_i^{(2)}]\indices{_\alpha^{\delta \gamma}}
  [\xi_i^{(2)}]\indices{^\beta_{\delta \gamma}} 
- [\mu_i^{(2)}]\indices{_\delta^{\beta \gamma}}
  [\xi_i^{(2)}]\indices{^\delta_{\alpha \gamma}} 
- [\mu_i^{(2)}]\indices{_\delta^{\gamma \beta}}
  [\xi_i^{(2)}]\indices{^\delta_{\gamma \alpha}} \\
 [\dot{\mu_i}^{(2)}]\indices{_\alpha^{\beta \gamma}} =&
[\mu_i^{(2)}]\indices{_\alpha^{\delta \gamma}}
[\xi_i^{(1)}]\indices{^\beta_\delta} +
[\mu_i^{(2)}]\indices{_\alpha^{\beta \delta}}
[\xi_i^{(1)}]\indices{^\gamma_\delta} - [\mu_i^{(2)}]\indices{_\delta^{\beta \gamma}} [\xi_i^{(1)}]\indices{^\delta_\alpha}
\end{align*}

\subsection{Computing $\dot{q}$ as a function of $\xi$}
\label{sec:computing-dotq-as}

The action of $\xi$ on $q$ is given by $\xi \cdot q$.  We set $\dot{q}
= \xi \cdot q$. We've already calculated $\dot{q}^{(0)}$.
We need only calculate $\dot{q}^{(1)}$ and $\dot{q}^{(2)}$.
Componentwise we calculate these to be
\begin{align*}
  [\dot{q}^{(1)}]\indices{^\alpha_\beta} =&
  [\xi^{(1)}]\indices{^\alpha_\gamma} [q^{(1)}]\indices{^\gamma_\beta} \\
  [\dot{q}^{(2)}]\indices{^\alpha_{\beta \gamma} } =&
  [\xi^{(2)}]\indices{^\alpha_{\delta \epsilon}} \cdot 
  [q^{(1)}]\indices{^\delta_\beta} \cdot [q^{(1)}]\indices{^\epsilon_\gamma} + 
 [\xi^{(1)}]\indices{^\alpha_\delta} \cdot [q^{(2)}] \indices{^\delta_{\beta
     \gamma} }
\end{align*}

\section{First variation equations}
\label{sec:variation-equations}

The first variation equations are equivalent to applying the tangent functor
to our evolutions.  We find the velocites:
\begin{align*}
  \frac{d}{dt} \delta q_i^\alpha =& \delta p_{j\beta} K^{\alpha \beta}(q_i - q_j) 
  + p_{j\beta} (\delta q_i^\gamma - \delta q_j^\gamma) \partial_{\gamma}K^{\alpha \beta}(q_i - q_j) \\
  & - [\delta \mu_j^{(1)}] \indices{_\beta^\gamma} \partial_\gamma K^{\alpha \beta}(q_i -q_j)
  - [\mu_j^{(1)}]\indices{_\beta^\gamma} \partial_{\gamma\delta}K^{\alpha\beta}(q_i-q_j) (\delta q_i^\delta - \delta q_j^\delta) \\
  &+ [\delta \mu_j^{(2)}]\indices{_\beta^{\gamma\delta}} \partial_{\gamma\delta}K^{\alpha\beta}(q_i - q_j)
  + [\mu_j^{(2)}]\indices{_\beta^{\gamma\delta}} \partial_{\gamma\delta\epsilon}K^{\alpha\beta}(q_i - q_j) (\delta q_i^\epsilon - \delta q_j^\epsilon)
\end{align*}

\begin{align*}
[\delta \xi^{(1)}_i]\indices{^\alpha_\beta} =& \delta
p_{j\gamma} \partial_\beta K^{\gamma\alpha}(q_i - q_j) 
   +p_{j\gamma} (\delta q_i^\delta - \delta q_j^\delta) \partial_{\delta\beta}K^{\alpha\gamma}(q_i - q_j) \\
  &- [\delta \mu_j^{(1)}]\indices{_\delta^\gamma} \partial_{\beta\gamma}K^{\alpha\delta}(q_i- q_j) 
  - [\mu_j^{(1)}]\indices{_\delta^\gamma} (\delta q_i^\epsilon -
  \delta
  q_j^\epsilon) \partial_{\beta\gamma\epsilon}K^{\alpha\delta}(q_i - q_j)  \\
  &+ [\delta \mu_j^{(2)}]\indices{_\phi^{\gamma\delta}} \partial_{\beta\gamma\delta}K^{\alpha \phi}(q_i - q_j) 
  + [\mu_j^{(2)}]\indices{_\phi^{\gamma\delta}} (\delta q_i^\epsilon - \delta q_j^\epsilon) \partial_{\beta\gamma\delta\epsilon}K^{\alpha\phi}(q_i-q_j)
\end{align*}

\begin{align*}
  [\delta \xi^{(2)}_i]\indices{^\alpha_{\beta\gamma}} =&
  \delta p_{j\delta} \partial_{\beta\gamma} K^{\alpha\delta}(ij) 
  + p_{j\epsilon} (\delta q_i^\delta - \delta q_j^\delta) \partial_{\beta\gamma\delta}K^{\alpha\epsilon}(ij) \\
  &- [\delta \mu_j^{(1)}]\indices{_\epsilon^\delta} \partial_{\beta\gamma\delta}K^{\alpha\epsilon}(ij)
  -[\mu_j^{(1)}]\indices{_\phi^\delta} (\delta q_i^\epsilon - \delta q_j^\epsilon) \partial_{\beta\gamma\delta\epsilon}K^{\alpha\phi}(ij) \\
  &+[\delta\mu_j^{(2)}]\indices{_\phi^{\delta\epsilon}} \partial_{\beta\gamma\delta\epsilon}K^{\alpha\phi}(ij)
  + [\mu_j^{(2)}]\indices{_\lambda^{\delta\epsilon}} (\delta q_i^\phi - \delta q_j^\phi) \partial_{\beta\gamma\delta\epsilon\phi} K^{\alpha\lambda}(ij)
\end{align*}

and the momenta:
\begin{align*}
  \frac{d}{dt} \delta p_{i\alpha} =& \delta T^{00}_{i\alpha} + \delta
  T^{01}_{i\alpha} + \delta T^{02}_{i\alpha} + \delta T^{12}_{i\alpha}
  + \delta T^{11}_{i\alpha} + \delta T^{22}_{i\alpha}
\end{align*}

The first-variation equation for $\mu^{(1)}$ is 
\begin{align*}
  \frac{d}{dt} [\delta\mu^{(1)} ]\indices{_\alpha^\beta} =&
  [\delta\mu^{(1)}]\indices{_\alpha^\gamma}[\xi^{(1)}]\indices{^\beta_\gamma}
  +
  [\mu^{(1)}]\indices{_\alpha^\gamma}[\delta\xi^{(1)}]\indices{^\beta_\gamma}
  \\
  &-[\delta\mu^{(1)}]\indices{_\gamma^\beta}[\xi^{(1)}]\indices{^\gamma_\alpha}
  -
  [\mu^{(1)}]\indices{_\gamma^\beta}[\delta\xi^{(1)}]\indices{^\gamma_\alpha}\\
  &+[\delta\mu^{(2)}]\indices{_\alpha^{\delta\gamma}}[\xi^{(2)}]\indices{^\beta_{\delta\gamma}}
  +
  [\mu^{(2)}]\indices{_\alpha^{\delta\gamma}}[\delta\xi^{(2)}]\indices{^\beta_{\delta\gamma}}
  \\
  &-[\delta\mu^{(2)}]\indices{_\delta^{\beta\gamma}}[\xi^{(2)}]\indices{^\delta_{\alpha\gamma}}
  -
  [\mu^{(2)}]\indices{_\delta^{\beta\gamma}}[\delta\xi^{(2)}]\indices{^\delta_{\alpha\gamma}}
  \\
  &-[\delta\mu^{(2)}]\indices{_\delta^{\gamma\beta}}[\xi^{(2)}]\indices{^\delta_{\gamma\alpha}}
  -
  [\mu^{(2)}]\indices{_\delta^{\gamma\beta}}[\delta\xi^{(2)}]\indices{^\delta_{\gamma\alpha}}
\end{align*}
and finally
\begin{align*}
  \frac{d}{dt}[\delta\mu^{(2)}]\indices{_\alpha^{\beta\gamma}} =&
  [\delta\mu^{(2)}]\indices{_\alpha^{\delta\gamma}}[\xi^{(1)}]\indices{^\beta_\delta}
  +
  [\mu^{(2)}]\indices{_\alpha^{\delta\gamma}}[\delta\xi^{(1)}]\indices{^\beta_\delta}
  \\
  &+[\delta\mu^{(2)}]\indices{_\alpha^{\beta\delta}}[\xi^{(1)}]\indices{^\gamma_\delta}
  +
  [\mu^{(2)}]\indices{_\alpha^{\beta\delta}}[\delta\xi^{(1)}]\indices{^\gamma_\delta}
  \\
  &-[\delta\mu^{(2)}]\indices{_\delta^{\beta\gamma}}
  [\xi^{(1)}]\indices{^\delta_\alpha} - [\mu^{(2)}]\indices{_\delta^{\beta\gamma}}[\delta\xi^{(1)}]\indices{^\delta_\alpha}
\end{align*}
where the $\delta T$'s are given by
\begin{align*}
  \delta T^{00}_{i\alpha} =&
  -\delta p_{i\gamma}p_{j\beta} \partial_{\alpha}K^{\gamma\beta}(q_i -
  q_j)   - p_{i\gamma} \delta
  p_{j\beta} \partial_{\alpha}K^{\gamma\beta}(q_i - q_j) \\
  &- p_{i\gamma}p_{j\beta} \partial_{\alpha\delta}K^{\gamma\beta}(q_i
  - q_j) (\delta q_i^\delta - \delta q_j^\delta)
\end{align*}

\begin{align*}
  \delta T^{01}_{i\alpha} =& -\delta p_{j\delta}
  [\mu_i^{(1)}]\indices{_\beta^\gamma} \partial_{\gamma\alpha}K^{\delta\beta}(ij)
  -p_{j\delta} [\delta\mu_i^{(1)}]\indices{_\beta^\gamma} \partial_{\gamma\alpha}K^{\delta\beta}(ij)
  \\
  &-p_{j\delta}[\mu_i^{(1)}]\indices{_\beta^\gamma}(\delta
  q_i^\epsilon - \delta
  q_j^\epsilon) \partial_{\epsilon\gamma\alpha}K^{\delta\beta}(ij) \\
  &+\delta
  p_{i\delta}[\mu_j^{(1)}]\indices{_\beta^\gamma}\partial_{\gamma\alpha}K^{\delta\beta}(ij)
  + p_{i\delta}[\delta\mu_j^{(1)}]\indices{_\beta^\gamma}\partial_{\gamma\alpha}K^{\delta\beta}(ij)
  \\
  &+p_{i\delta}[\mu_j^{(1)}]\indices{_\beta^\gamma}(\delta
  q_i^\epsilon - \delta
  q_j^\epsilon) \partial_{\epsilon\gamma\alpha}K^{\delta\beta}(ij) \\
\end{align*}

\begin{align*}
\delta T^{02}_{i\alpha} =& -\delta p_{i\epsilon}
[\mu_j^{(2)}]\indices{_\beta^{\gamma\delta}} \partial_{\gamma\delta\alpha}
K^{\epsilon\beta}(ij) -p_{i\epsilon}
[\delta\mu_j^{(2)}]\indices{_\beta^{\gamma\delta}} \partial_{\gamma\delta\alpha}
K^{\epsilon\beta}(ij) \\
&-p_{i\epsilon}
[\mu_j^{(2)}]\indices{_\beta^{\gamma\delta}} (\delta q_i^\phi -
\delta
q_j^\phi) \partial_{\gamma\delta\phi\alpha}K^{\epsilon\beta}(ij) \\
 &-\delta p_{j\epsilon}
[\mu_i^{(2)}]\indices{_\beta^{\gamma\delta}} \partial_{\gamma\delta\alpha}
K^{\epsilon\beta}(ij)
-p_{j\epsilon}
[\delta\mu_i^{(2)}]\indices{_\beta^{\gamma\delta}} \partial_{\gamma\delta\alpha}
K^{\epsilon\beta}(ij) \\
  &-p_{j\epsilon}
[\mu_i^{(2)}]\indices{_\beta^{\gamma\delta}} (\delta q_i^\phi -
\delta
q_j^\phi) \partial_{\gamma\delta\phi\alpha}K^{\epsilon\beta}(ij) \\
\end{align*}

\begin{align*}
\delta T^{12}_{i\alpha} =& -[\delta
\mu_i^{(1)}]\indices{_\phi^\epsilon}
[\mu_j^{(2)}]\indices{_\beta^{\gamma\delta}} \partial_{\epsilon\gamma\delta\alpha}K^{\phi\beta}(ij)\\
  &-[\mu_i^{(1)}]\indices{_\phi^\epsilon}
  [\delta\mu_j^{(2)}]\indices{_\beta^{\gamma\delta}} \partial_{\epsilon\gamma\delta\alpha}K^{\phi\beta}(ij)\\
  &-[\mu_i^{(1)}]\indices{_\phi^\epsilon}[\mu_j^{(2)}]\indices{_\beta^{\gamma\delta}}
  (\delta q_i^\zeta - \delta
  q_j^\zeta) \partial_{\zeta\epsilon\gamma\delta\alpha}K^{\phi\beta}(ij) \\
  &+[\delta
\mu_j^{(1)}]\indices{_\phi^\epsilon}
[\mu_i^{(2)}]\indices{_\beta^{\gamma\delta}} \partial_{\epsilon\gamma\delta\alpha}K^{\phi\beta}(ij)\\
  &+[\mu_j^{(1)}]\indices{_\phi^\epsilon}
  [\delta\mu_i^{(2)}]\indices{_\beta^{\gamma\delta}} \partial_{\epsilon\gamma\delta\alpha}K^{\phi\beta}(ij)\\
  &+[\mu_j^{(1)}]\indices{_\phi^\epsilon}[\mu_i^{(2)}]\indices{_\beta^{\gamma\delta}}
  (\delta q_i^\zeta - \delta
  q_j^\zeta) \partial_{\zeta\epsilon\gamma\delta\alpha}K^{\phi\beta}(ij) \\
\end{align*}

\begin{align*}
\delta T^{11}_{i\alpha} =& (
[\delta\mu_i^{(1)}]\indices{_\epsilon^\delta}
[\mu_j^{(1)}]\indices{_\beta^\gamma} +
[\mu_i^{(1)}]\indices{_\epsilon^\delta}[\delta\mu_j^{(1)}]\indices{_\beta^\gamma}
) \partial_{\delta\gamma\alpha}K^{\epsilon\beta}(ij) \\
  &+[\mu_i^{(1)}]\indices{_\epsilon^\delta}
  [\mu_j^{(1)}]\indices{_\beta^\gamma} (\delta q_i^\phi - \delta
  q_j^\phi) \partial_{\phi\delta\gamma\alpha}K^{\epsilon\beta}(ij) \\
\end{align*}

\begin{align*}
\delta T^{22}_{i\alpha}
=&- ([\delta\mu_i^{(2)}]\indices{_\zeta^{\epsilon\phi}}
[\mu_j^{(2)}]\indices{_\beta^{\gamma\delta}} +
[\mu_i^{(2)}]\indices{_\zeta^{\epsilon\phi}}
[\delta\mu_j^{(2)}]\indices{_\beta^{\gamma\delta}}) \partial_{\epsilon\delta\gamma\phi\alpha}K^{\zeta\beta}(ij)
\\
&-[\mu_i^{(2)}]\indices{_\zeta^{\epsilon\phi}}
[\mu_j^{(2)}]\indices{_\beta^{\gamma\delta}} (\delta q_i^\lambda -
\delta q_j^\lambda) \partial_{\lambda\epsilon\delta\gamma\phi\alpha}K^{\zeta\beta}(ij)
\end{align*}

Finally, we compute the variation equations for $\delta q^{(1)}$ and
$\delta q^{(2)}$ to be
\begin{align*}
  \delta [\dot{q}_i^{(1)}]\indices{^\alpha_\beta} =& [\delta
  \xi_i^{(1)} ]\indices{^\alpha_\gamma}
  [q_i^{(1)}]\indices{^\gamma_\beta} +
  [\xi_i^{(1)}]\indices{^\alpha_\gamma}
  [\delta q_i^{(1)}]\indices{^\gamma_\beta} \\
  \delta [\dot{q}^{(2)}]\indices{^\alpha_{\beta\gamma}} =&
  [\delta \xi^{(2)}]\indices{^\alpha_{\delta \epsilon}} \cdot 
  [q^{(1)}]\indices{^\delta_\beta} \cdot
  [q^{(1)}]\indices{^\epsilon_\gamma} + 
[\xi^{(2)}]\indices{^\alpha_{\delta \epsilon}} \cdot 
  [\delta q^{(1)}]\indices{^\delta_\beta} \cdot
  [q^{(1)}]\indices{^\epsilon_\gamma} \\
  &+ [\xi^{(2)}]\indices{^\alpha_{\delta \epsilon}} \cdot 
  [q^{(1)}]\indices{^\delta_\beta} \cdot
  [\delta q^{(1)}]\indices{^\epsilon_\gamma} + 
  [\delta \xi^{(1)}]\indices{^\alpha_\delta} \cdot [q^{(2)}] \indices{^\delta_{\beta
     \gamma} } +
  [\xi^{(1)}]\indices{^\alpha_\delta} \cdot [\delta q^{(2)}] \indices{^\delta_{\beta
     \gamma} }
\end{align*}

\section{Computation of the adjoint equations}
\label{sec:adjoint}

Given any ODE $\dot{x} = f(x)$ on $M$, we may consider the equations of motion
for variations $\frac{d}{dt} \delta x = T_xf \cdot \delta
x$.  In particular, $T_x f$ is a linear operator over the point $x$
which has a dual operator.  The adjoint equations are and ODE on $T^*M$ given by
\begin{align*}
  \frac{d \lambda }{dt} = - T_x^*f \cdot \lambda.
\end{align*}
This is useful for us in the following way.
Given an integral curve, $x(t)$, and a variation in the initial condition,
$\delta x_0$, we see that the quantity $\langle \lambda(t) , \delta
x(t) \rangle$ is constant when $\delta x(t)$ satisfies the first
variation equation with initial condition $\delta x_0$ and
$\lambda(t)$ satisfies the adjoint equation.
In our case we are able to compute the gradient of the energy with
respect to varying an initial condition in this way.  More explicitly,
we should be able to express $T_xf$ as a matrix $M(x) \indices{^B_A}$
so that the first variation equations are
\begin{align*}
  \frac{d}{dt} \delta x^A = M(x)\indices{^A_B} \delta x^B
\end{align*}
and the adjoint equations can be written as
\begin{align*}
  \dot{\lambda}_A = - \lambda_B M(x) \indices{^B_A}
\end{align*}
where $\lambda_A$ is the covector associated to the $A$-th coordinate
and $M(x)\indices{^B_A}$ is the coefficient for $\delta x^A$ in the
equation for $\frac{d}{dt} \delta^B$.  More specifically, the elements
of $M^B_A$ is the partial derivative of $\delta \dot{B}$ with respect
to $\delta A$.  So we compute all these (36) quantities below.

\begin{align*}
  \pder{ [ \delta\dot{q}_i^{(0)}]^{\alpha} }{ [\delta q^{(0)}_j]^{\beta}} =&
  \left(
  p_{k\gamma} \partial_{\beta}K^{\alpha\gamma}(jk) 
  -[\mu^{(1)}_k]\indices{_\delta^\gamma} \partial_{\gamma\beta}K^{\alpha\delta}(jk)
  +[\mu^{(2)}_k]\indices{_\epsilon^{\gamma\delta}} \partial_{\gamma\delta\beta}K^{\alpha\epsilon}(jk)
  \right)
  \delta_i^j
  \\&
  -p_{j\gamma} \partial_{\beta}K^{\alpha\gamma}(ij) 
  +[\mu^{(1)}_j]\indices{_\delta^\gamma} \partial_{\gamma\beta}K^{\alpha\delta}(ij)
  -[\mu^{(2)}_j]\indices{_\epsilon^{\gamma\delta}} \partial_{\gamma\delta\beta}K^{\alpha\epsilon}(ij),
  \\ 
  \pder{[\delta \dot{q}_i^{(0)}]^\alpha}{[\delta
    q^{(1)}_j]\indices{^\beta_\gamma}} =& 0
  \quad,\quad 
  \pder{[\delta \dot{q}_i^{(0)}]^\alpha}{[\delta q^{(2)}]\indices{^\beta_{\gamma\delta}}} = 0,
  \\ 
  \pder{[\delta \dot{q}_i^{(0)}]^\alpha]}{[\delta p_j^{(0)}]_\beta} =& K^{\alpha\beta}(ij)
  \quad,\quad 
  \pder{[\delta
    \dot{q}_i^{(0)}]^\alpha }{[\delta\mu^{(1)}_j]\indices{_\beta^\gamma}}
  = -\partial_\gamma K^{\alpha\beta}(ij)
  \quad,\quad 
  \pder{[\delta
    \dot{q}^{(0)}_i]^\alpha}{[\delta\mu_j^{(2)}]\indices{_\beta^{\gamma\delta}}}
  = \partial_{\gamma\delta}K^{\alpha\beta}(ij),
\end{align*}

\begin{align*} 
  \pder{[\delta \dot{q}^{(1)}_i] \indices{^\alpha_\beta}}{[\delta
    q^{(0)}_j]\indices{^\gamma}} =& \pder{ [\delta \xi_i^{(1)}]
    \indices{^\alpha_\delta}}{[\delta q^{(0)}_j]^\gamma} [q_i^{(1)}]\indices{^\delta_\beta}
  \quad,\quad 
  \pder{[\delta \dot{q}^{(1)}_i]\indices{^\alpha_\beta}}{[\delta
    q_j^{(1)}]\indices{^\gamma_\delta}} =
  [\xi_i^{(1)}]\indices{^\alpha_\gamma}
  \delta^\delta_\beta \delta_i^j,
  \quad,\quad 
  \pder{[\delta \dot{q}^{(1)}_i] \indices{^\alpha_\beta}}{[\delta
    q^{(2)}_j]\indices{^\gamma_{\delta\epsilon}}} = 0,
  \\ 
  \pder{[\delta \dot{q}^{(1)}_i] \indices{^\alpha_\beta}}{[\delta
    p_j^{(0)}]_\gamma} =& \pder{[\delta\xi_i^{(1)}]
    \indices{^\alpha_\delta} }{ [\delta p^{(0)}_j]_\gamma}
  [q_i^{(1)}]\indices{^\delta_\beta}
  \quad,\quad 
  \pder{[\delta \dot{q}^{(1)}_i] \indices{^\alpha_\beta}}{[\delta
    \mu_j^{(1)}] \indices{_\gamma^\delta}} = \pder{[\delta\xi_i^{(1)}]
    \indices{^\alpha_\epsilon} }{ [\delta \mu^{(1)}_j] \indices{_\gamma^\delta}}
  [q_i^{(1)}]\indices{^\epsilon_\beta},
  \\ 
  \pder{[\delta \dot{q}^{(1)}_i] \indices{^\alpha_\beta}}{[\delta
    \mu_j^{(2)}] \indices{_\gamma^{\delta\epsilon}} } =& \pder{[\delta\xi_i^{(1)}]
    \indices{^\alpha_\phi} }{ [\delta \mu^{(2)}_j] \indices{_\gamma^{\delta\epsilon}}}
  [q_i^{(1)}]\indices{^\phi_\beta},
\end{align*}

\begin{align*}
  \pder{[\delta\dot{q}_i^{(2)}]\indices{^\alpha_{\beta\gamma}}}{[\delta
    q_j^{(0)}]^\delta} =& \pder{
    [\delta\xi_i^{(2)}]\indices{^\alpha_{\phi\epsilon}}}{[\delta
    q_j^{(0)}]^\delta} [q_i^{(1)}]\indices{^\phi_\beta}
  [q_i^{(1)}]\indices{^\epsilon_\gamma} + \pder{
    [\delta\xi_i^{(1)}]\indices{^\alpha_\epsilon}}{[\delta
    q_j^{(0)}]^\delta} [q_i^{(2)}]\indices{^\epsilon_{\beta\gamma}},
  \\ 
  \pder{[\delta\dot{q}_i^{(2)}]\indices{^\alpha_{\beta\gamma}}}{[\delta
    q_j^{(1)}]\indices{^\delta_\epsilon}}
  =& ( [\xi_i^{(2)}]\indices{^\alpha_{\delta\phi}}
  [q_i^{(1)}]\indices{^\phi_\gamma} \delta^\epsilon_\beta +
  [\xi_i^{(2)}]\indices{^\alpha_{\phi\delta}}
  [q_i^{(1)}]\indices{^\phi_\beta} \delta^\epsilon_\gamma)\delta^j_i,
  \\ 
  \pder{[\delta\dot{q}_i^{(2)}]\indices{^\alpha_{\beta\gamma}}}{[\delta
    q_j^{(2)}]\indices{^\delta_{\epsilon\phi}}} =&
  [\xi_i^{(1)}]\indices{^\alpha_\delta}\delta^\epsilon_\beta
  \delta^\phi_\gamma \delta^j_i,
  \end{align*}
  \begin{align*}
  \pder{[\delta\dot{q}_i^{(2)}]\indices{^\alpha_{\beta\gamma}}}{[\delta
    p_j^{(0)}]_\delta} =& \pder{
    [\delta\xi_i^{(2)}]\indices{^\alpha_{\phi\epsilon}}}{[\delta p_j^{(1)}]\indices{_\delta}}
  [q_i^{(1)}]\indices{^\phi_\beta} [q_i^{(1)}]\indices{^\epsilon_\gamma}
  + \pder{[\delta \xi_i^{(1)}]\indices{^\alpha_\epsilon}}{[\delta
    p_j^{(0)}]_\delta} [q_i^{(2)}]\indices{^\epsilon_{\beta\gamma}},
  \\ 
  \pder{[\delta\dot{q}_i^{(2)}]\indices{^\alpha_{\beta\gamma}}}{[\delta
    \mu_j^{(1)}] \indices{_\delta^\epsilon}} =&
  \pder{[\delta\xi_i^{(2)}]\indices{^\alpha_{\phi\zeta}}}{[\delta\mu_j^{(1)}]\indices{_\delta^\epsilon}}
  [q_i^{(1)}]\indices{^\phi_\beta} [q_i^{(1)}]\indices{^\zeta_\gamma}
  + \pder{
    [\delta\xi_i^{(1)}]\indices{^\alpha_\phi}}{[\delta\mu_j^{(1)}]\indices{_\delta^\epsilon}}
  [q^{(2)}_i]\indices{^\phi_{\beta\gamma}},
  \\ 
  \pder{[\delta\dot{q}_i^{(2)}]\indices{^\alpha_{\beta\gamma}}}{[\delta
    \mu_j^{(2)}] \indices{_\delta^{\epsilon\phi}}} =& \pder{ [\delta
    \xi_i^{(2)}]\indices{^\alpha_{\zeta\lambda}}}{
    [\delta\mu_j^{(2)}]\indices{_\delta^{\epsilon\phi}}}
  [q_i^{(1)}]\indices{^\zeta_\beta}
  [q^{(1)}_i]\indices{^\lambda_\gamma} + \pder{
    [\delta\xi_i^{(1)}]\indices{^\alpha_\zeta}}{[\delta\mu_j^{(2)}]\indices{_\delta^{\epsilon\phi}}} [q_i^{(2)}]\indices{^\zeta_{\beta\gamma}},
\end{align*}

\begin{align*}
  \pder{[\delta \dot{p}^{(0)}_i]_\alpha}{[\delta q_j^{(0)}]^\beta} =&
  \pder{[ \delta T^{00}_{i}]_\alpha }{[\delta q_j^{(0)}]^\beta} + \pder{[
    \delta T^{01}_{i}]_\alpha }{[\delta q_j^{(0)}]^\beta} + \pder{[
    \delta T^{11}_{i}]_\alpha }{[\delta q_j^{(0)}]^\beta} + \pder{[
    \delta T^{12}_{i}]_\alpha }{[\delta q_j^{(0)}]^\beta}  +\pder{[
    \delta T^{02}_{i}]_\alpha }{[\delta q_j^{(0)}]^\beta} + \pder{[
    \delta T^{22}_{i}]_\alpha }{[\delta q_j^{(0)}]^\beta},
    \\ 
    \pder{[\delta \dot{p}^{(0)}_i]_\alpha}{[\delta q_j^{(1)}] \indices{^\beta_\gamma}} =&  0 
    \quad , \quad 
    \pder{[\delta \dot{p}^{(0)}_i]_\alpha}{[\delta q_j^{(2)}] \indices{^\beta_{\gamma\delta}}} = 0 ,
    \end{align*}
    \begin{align*}
    \pder{[\delta \dot{p}^{(0)}_i]_\alpha}{[\delta p_j^{(0)}]
      \indices{_\beta}} =& \pder{[\delta
      T_i^{00}]_\alpha}{[\delta p_j^{(0)}]\indices{_\beta}} +
      \pder{[\delta
      T_i^{01}]_\alpha}{[\delta p_j^{(0)}]\indices{_\beta}} +
      \pder{[\delta
      T_i^{02}]_\alpha}{[\delta p_j^{(0)}]\indices{_\beta}},
    \\ 
    \pder{[\delta \dot{p}^{(0)}_i]_\alpha}{[\delta \mu_j^{(1)}]
      \indices{_\beta^{\gamma}}} =& \pder{[\delta
      T_i^{01}]_\alpha}{[\delta \mu_j^{(1)}]\indices{_\beta^\gamma}} + \pder{[\delta
      T_i^{11}]_\alpha}{[\delta \mu_j^{(1)}]\indices{_\beta^\gamma}} +\pder{[\delta
      T_i^{12}]_\alpha}{[\delta \mu_j^{(1)}]\indices{_\beta^\gamma}},
    \\ 
    \pder{[\delta \dot{p}^{(0)}_i]_\alpha}{[\delta \mu_j^{(2)}]
      \indices{_\beta^{\gamma\delta}}} =& \pder{[\delta
      T_i^{02}]_\alpha}{[\delta \mu_j^{(2)}]\indices{_\beta^{\gamma\delta}}} + \pder{[\delta
      T_i^{12}]_\alpha}{[\delta \mu_j^{(2)}]\indices{_\beta^{\gamma\delta}}} +\pder{[\delta
      T_i^{22}]_\alpha}{[\delta \mu_j^{(2)}]\indices{_\beta^{\gamma\delta}}},
\end{align*}

\begin{align*}
  \pder{[\delta\dot{\mu}_i^{(1)}]\indices{_\alpha^\beta}}{[\delta
    q_j^{(0)}]^\gamma} =& [\mu_i^{(1)}]\indices{_\alpha^\delta}
  \pder{ [\delta\xi_i^{(1)}]\indices{^\beta_\delta}}{ [\delta
    q_j^{(0)}]^\gamma } - [\mu_i^{(1)}]\indices{_\delta^\beta} \pder{
    [\delta \xi_i^{(1)}] \indices{^\delta_{\alpha}}}{ [\delta
    q_j^{(0)}]^\gamma} \\
  &+ [\mu_i^{(2)}]\indices{_\alpha^{\delta\epsilon}}
  \pder{[\delta\xi_i^{(2)}]\indices{^\beta_{\delta\epsilon}}}{ [\delta
    q_j^{(0)}]^\gamma} -
  [\mu_i^{(2)}]\indices{_\delta^{\beta\epsilon}} \pder{ [\delta
    \xi_i^{(2)}]\indices{^\delta_{\alpha\epsilon}}}{ [\delta
    q_j^{(0)}]\indices{^\gamma}} -
  [\mu_i^{(2)}]\indices{_\delta^{\epsilon\beta}} \pder{
    [\delta\xi_i^{(2)}]\indices{^\delta_{\epsilon\alpha}}}{ [\delta
    q_j^{(0)}]^\gamma},
  \\ 
  \pder{[\delta\dot{\mu}_i^{(1)}]\indices{_\alpha^\beta}}{[\delta
    q_j^{(1)}]^\gamma} =& 0,
  \quad,\quad 
  \pder{[\delta\dot{\mu}_i^{(1)}]\indices{_\alpha^\beta}}{[\delta
    q_j^{(2)}]^\gamma} = 0,
  \end{align*}
  \begin{align*} 
  \pder{[\delta\dot{\mu}_i^{(1)}]\indices{_\alpha^\beta}}{[\delta
    p_j^{(0)}]_\gamma} =& [\mu_i^{(1)}]\indices{_\alpha^\delta}
  \pder{ [\delta\xi_i^{(1)}]\indices{^\beta_\delta}}{ [\delta
    p_j^{(0)}]_\gamma } - [\mu_i^{(1)}]\indices{_\delta^\beta} \pder{
    [\delta \xi_i^{(1)}] \indices{^\delta_{\alpha}}}{ [\delta
    p_j^{(0)}]_\gamma} \\
  &+ [\mu_i^{(2)}]\indices{_\alpha^{\delta\epsilon}}
  \pder{[\delta\xi_i^{(2)}]\indices{^\beta_{\delta\epsilon}}}{ [\delta
    p_j^{(0)}]_\gamma} -
  [\mu_i^{(2)}]\indices{_\delta^{\beta\epsilon}} \pder{ [\delta
    \xi_i^{(2)}]\indices{^\delta_{\alpha\epsilon}}}{ [\delta
    p_j^{(0)}]\indices{_\gamma}} -
  [\mu_i^{(2)}]\indices{_\delta^{\epsilon\beta}} \pder{
    [\delta\xi_i^{(2)}]\indices{^\delta_{\epsilon\alpha}}}{ [\delta
    p_j^{(0)}]_\gamma}, 
  \\ 
  \pder{[\delta \dot{\mu}_i^{(1)}]\indices{_\alpha^\beta}}{
    [\delta\mu_j^{(1)}]\indices{_\gamma^\delta}} =&
  \delta_i^j\delta_\alpha^\gamma
  [\xi_i^{(1)}]\indices{^\beta_\delta} +
  [\mu_i^{(1)}]\indices{_\alpha^\epsilon} \pder{ [\delta
    \xi_i^{(1)}]\indices{^\beta_\epsilon}}{
    [\delta\mu_j^{(1)}]\indices{_\gamma^\delta}} -
  \delta_i^j \delta_\delta^\beta
  [\xi_i^{(1)}]\indices{^\gamma_\alpha} -
  [\mu_i^{(1)}]\indices{_\epsilon^\beta}
  \pder{[\delta\xi_i^{(1)}]\indices{^\epsilon_\alpha}}{[\delta\mu_j^{(1)}]\indices{_\gamma^\delta}}
  \\
  &+ [\mu_i^{(2)}]\indices{_\alpha^{\epsilon\phi}}
  \pder{[\delta\xi_i^{(2)}]\indices{^\beta_{\epsilon\phi}}}{
    [\delta\mu_j^{(1)}]\indices{_\gamma^\delta}} -
  [\mu_i^{(2)}]\indices{_\phi^{\beta\epsilon}} \pder{
    [\delta\xi_i^{(2)}]\indices{^\phi_{\alpha\epsilon}}}{
    [\delta\mu_j^{(1)}]\indices{_\gamma^\delta}} -
  [\mu_i^{(2)}]\indices{_\phi^{\epsilon\beta}} \pder{
    [\delta\xi_i^{(2)}]\indices{^\phi_{\epsilon\alpha}}}{
    [\delta\mu_j^{(1)}]\indices{_\gamma^\delta}},
  \\ 
  \pder{[\delta\dot{\mu}_i^{(1)}]\indices{_\alpha^\beta}}{
    [\delta\mu_j^{(2)}]\indices{_\gamma^{\delta\epsilon}}} =&
  [\mu_i^{(1)}]\indices{_\alpha^\phi}
  \pder{[\delta\xi_i^{(1)}]\indices{^\beta_\phi}}{
    [\delta\mu_j^{(2)}]\indices{_\gamma^{\delta\epsilon}}}
  -[\mu_i^{(1)}]\indices{_\phi^\beta} \pder{ [\delta
    \xi_i^{(1)}]\indices{^\phi_\alpha}}{ [\delta
    \mu_j^{(2)}]\indices{_\gamma^{\delta\epsilon}}} + \delta_i^j
  \delta_\alpha^\gamma [\xi_i^{(2)}]\indices{^\beta_{\delta\epsilon}}, \\ 
  &+ [\mu_i^{(2)}]\indices{_\alpha^{\phi\lambda}} \pder{ [\delta
    \xi_i^{(2)}]\indices{^\beta_{\phi\lambda}}}{
    [\delta\mu_j^{(2)}]\indices{_\gamma^{\delta\epsilon}}} -
  \delta_i^j \delta_\delta^\beta [\xi_i^{(2)}]
  \indices{^\gamma_{\alpha\epsilon}} \\
  &-[\mu_i^{(2)}]\indices{_\phi^{\beta\lambda}}
  \pder{[\delta\xi_i^{(2)}]\indices{^\phi_{\alpha\lambda}}}{[\delta\mu_j^{(2)}]\indices{_\gamma^{\delta\epsilon}}}
  - \delta_i^j \delta_\epsilon^\beta
  [\xi_i^{(2)}]\indices{^\gamma_{\delta\lambda}} -
  [\mu_i^{(2)}]\indices{_\phi^{\lambda\beta}}
  \pder{[\delta\xi_i^{(2)}]\indices{^\phi_{\lambda\alpha}}}{[\delta\mu_j^{(2)}]\indices{_\gamma^{\delta\epsilon}}} ,
\end{align*}

\begin{align*}
  \pder{ [\delta \dot{\mu}_i^{(2)}]\indices{_\alpha^{\beta\gamma}}}{ [\delta
    q_j^{(0)}]^\delta} =&
  [\mu_i^{(2)}]\indices{_\alpha^{\epsilon\gamma}} \pder{ [\delta
    \xi_i^{(1)}]\indices{^\beta_\epsilon}}{ [\delta q_j^{(0)}]^\delta}
  + [\mu_i^{(2)}]\indices{_\alpha^{\beta\epsilon}} \pder{
    [\delta\xi^{(1)}_i]\indices{^\gamma_\epsilon}}{ [\delta
    q^{(0)}_j]^\delta} -
  [\mu_i^{(2)}]\indices{_\epsilon^{\beta\gamma}} \pder{ [\delta
    \xi_i^{(1)}]\indices{^\epsilon_\alpha}}{ [\delta
    q_j^{(0)}]^\delta}
  \\ 
  \pder{[\delta\dot{\mu}^{(2)}]\indices{_\alpha^{\beta\gamma}}}{
    [\delta q^{(1)}]\indices{^\delta_\epsilon}} =& 0
  \qquad 
  \pder{[\delta\dot{\mu}^{(2)}]\indices{_\alpha^{\beta\gamma}}}{
    [\delta q^{(2)}]\indices{^\delta_{\epsilon\phi}}} = 0
  \end{align*}
  
  \begin{align*}
  \pder{ [\delta \dot{\mu}_i^{(2)}]\indices{_\alpha^{\beta\gamma}}}{ [\delta
    p_j^{(0)}]_\delta} =&
  [\mu_i^{(2)}]\indices{_\alpha^{\epsilon\gamma}} \pder{ [\delta
    \xi_i^{(1)}]\indices{^\beta_\epsilon}}{ [\delta p_j^{(0)}]_\delta}
  + [\mu_i^{(2)}]\indices{_\alpha^{\beta\epsilon}} \pder{
    [\delta\xi^{(1)}_i]\indices{^\gamma_\epsilon}}{ [\delta
    p^{(0)}_j]_\delta} -
  [\mu_i^{(2)}]\indices{_\epsilon^{\beta\gamma}} \pder{ [\delta
    \xi_i^{(1)}]\indices{^\epsilon_\alpha}}{ [\delta
    p_j^{(0)}]_\delta}
  \\ 
  \pder{ [\delta \dot{\mu}_i^{(2)}]\indices{_\alpha^{\beta\gamma}}}{ [\delta
    \mu_j^{(1)}] \indices{_\delta^\epsilon}} =&
  [\mu_i^{(2)}]\indices{_\alpha^{\phi\gamma}} \pder{ [\delta
    \xi_i^{(1)}]\indices{^\beta_\phi}}{ [\delta
    \mu_j^{(1)}]\indices{_\delta^\epsilon}} +
  [\mu_i^{(2)}]\indices{_\alpha^{\beta\phi}} \pder{ [\delta
    \xi_i^{(1)}]\indices{^\gamma_\phi}}{
    [\delta\mu_j^{(1)}]\indices{_\delta^\epsilon}} -
  [\mu_i^{(2)}]\indices{_\phi^{\beta\gamma}} \pder{
    [\delta\xi_i^{(1)}] \indices{^\phi_\alpha}}{
    [\delta\mu_j^{(1)}]\indices{_\delta^\epsilon}}
  \\ 
  \pder{ [\delta \dot{\mu}_i^{(2)}]\indices{_\alpha^{\beta\gamma}}}{ [\delta
    \mu_j^{(2)}] \indices{_\delta^{\epsilon\phi}}} =&
  \delta^j_i \delta^\gamma_\phi \delta_\alpha^\delta
  [\xi_i^{(1)}]\indices{^\beta_\epsilon} +
  [\mu_i^{(2)}]\indices{_\alpha^{\lambda\gamma}} \pder{
    [\delta\xi_i^{(1)}]\indices{^\beta_\lambda}}{
    [\delta\mu_j^{(2)}]\indices{_\delta^{\epsilon\phi}}} + \delta_i^j
  \delta_\alpha^\delta \delta_\epsilon^\beta
  [\xi_i^{(1)}]\indices{^\gamma_\phi} \\
 &+
  [\mu_i^{(2)}]\indices{_\alpha^{\beta\lambda}} \pder{
    [\delta\xi_i^{(1)}]\indices{^\gamma_\lambda}}{
    [\delta\mu_j^{(2)}]\indices{_\delta^{\epsilon\phi}}} - \delta_i^j
  \delta_\epsilon^\beta \delta_\phi^\gamma
  [\xi_i^{(1)}]\indices{^\delta_\alpha} -
  [\mu_i^{(2)}]\indices{_\lambda^{\beta\gamma}} \pder{
    [\delta\xi_i^{(1)}]\indices{^\lambda_\alpha}}{ [\delta\mu_j^{(2)}]\indices{_\delta^{\epsilon\phi}}}
\end{align*}

\begin{align*}
  \pder{[\delta T_i^{00}]_\alpha}{[\delta q_j^{(0)}]^\beta} =&
  -p_{j\gamma} p_{k\delta} \partial_{\alpha\beta}K^{\gamma\delta}(jk)
  \delta_j^i + p_{i\gamma} p_{j\delta} \partial_{\alpha\beta}K^{\gamma\delta}(ij),
  \\ 
  \pder{[\delta T_i^{01}]_\alpha}{[\delta q_j^{(0)}]^\beta} =&
  \delta_i^j(p_{j\delta} [\mu_k^{(1)}] \indices{_\epsilon^\gamma} - p_{k\delta}
  [\mu_j^{(1)}]\indices{_\epsilon^\gamma}) \partial_{\alpha\beta\gamma}K^{\delta\epsilon}(jk)\\
  &- (p_{i\delta} [\mu_j^{(1)}]\indices{_\epsilon^\gamma} - p_{j\delta}
  [\mu_i^{(1)}]\indices{_\epsilon^\gamma} ) \partial_{\beta\gamma\alpha}K^{\delta\epsilon}(ij)
  \\ 
  \pder{[\delta T_i^{02}]_\alpha}{[\delta q_j^{(0)}]^\beta} =&
  (p_{i\epsilon} [\mu_j^{(2)}]\indices{_\phi^{\gamma\delta}} +
  p_{j\epsilon} [\mu_i^{(2)}] \indices{_\phi^{\gamma\delta}}
  ) \partial_{\gamma\delta\beta\alpha}K^{\epsilon\phi}(ij) \\
  &- \delta_i^j (p_{j\epsilon}
  [\mu_k^{(2)}]\indices{_\phi^{\gamma\delta}} + p_{k\epsilon}
  [\mu_j^{(2)}]\indices{_\phi^{\gamma\delta}}) \partial_{\gamma\delta\beta\alpha}K^{\epsilon\phi}(jk)
  \\ 
  \pder{[\delta T_i^{12}]_\alpha}{[\delta q_j^{(0)}]^\beta} =&
  \delta^i_j ([\mu_k^{(1)}]\indices{_\phi^\epsilon} [\mu_j^{(2)}]
  \indices{_\lambda^{\gamma\delta}} -
  [\mu_j^{(1)}]\indices{_\phi^\epsilon} [\mu_k^{(2)}]
  \indices{_\lambda^{\gamma\delta}}) \partial_{\beta\epsilon\gamma\delta\alpha}K^{\phi\lambda}(jk)\\
  &- ( [\mu_j^{(1)}]\indices{_\phi^\epsilon}
  [\mu_i^{(2)}]\indices{_\lambda^{\gamma\delta}} -
  [\mu_i^{(1)}]\indices{_\phi^\epsilon}
  [\mu_j^{(2)}]\indices{_\lambda^{\gamma\delta}}
  ) \partial_{\beta\epsilon\gamma\delta\alpha} K^{\phi\lambda}(ij)\\
  \pder{[\delta T_i^{11}]_\alpha}{[\delta q_j^{(0)}]^\beta} =&
  \delta_i^j [\mu_j^{(1)}]\indices{_\epsilon^\delta}
  [\mu_k^{(1)}]\indices{_\phi^\gamma} \partial_{\beta\gamma\delta\alpha}
  K^{\epsilon\phi}(jk) - [\mu_i^{(1)}]\indices{_\epsilon^\delta}
  [\mu_j^{(1)}]\indices{_\phi^\gamma} \partial_{\beta\delta\gamma\alpha}K^{\epsilon\phi}(ij)\\
  \pder{[\delta T_i^{22}]_\alpha}{[\delta q_j^{(0)}]^\beta} =&
  -\delta_i^j [\mu_j^{(2)}]\indices{_\lambda^{\epsilon\phi}}
  [\mu_k^{(2)}]\indices{_\zeta^{\gamma\delta}} \partial_{\beta\epsilon\delta\gamma\phi\alpha}K^{\lambda\zeta}(jk)
  + [\mu_i^{(2)}]\indices{_\lambda^{\epsilon\phi}}
  [\mu_j^{(2)}]\indices{_\zeta^{\gamma\delta}} \partial_{\beta\epsilon\delta\gamma\phi\alpha}K^{\lambda\zeta}(ij)
\end{align*}

\begin{align*}
  \pder{ \delta T^{00}_{i\alpha}}{[\delta p_j^{(0)}]_\beta} =&
  -\delta_i^j p_{k\gamma} \partial_\alpha K^{\beta\gamma}(jk) -
  p_{i\gamma} \partial_{\alpha} K^{\gamma\beta}(ij)
  \\ 
  \pder{ \delta T^{01}_{i\alpha}}{[\delta p_j^{(0)}]_\beta} =&
  - [\mu_i^{(1)}]\indices{_\delta^\gamma} \partial_{\gamma\alpha}
  K^{\beta\delta}(ij) + \delta_i^j
  [\mu_k^{(1)}]\indices{_\delta^\gamma} \partial_{\gamma\alpha}K^{\beta\delta}(jk)
  \\ 
  \pder{ \delta T^{02}_{i\alpha}}{[\delta p_j^{(0)}]_\beta} =&  -
  \delta_i^j 
  [\mu_k^{(2)}]\indices{_\epsilon^{\gamma\delta}} \partial_{\gamma\delta\alpha}
  K^{\beta\epsilon}(jk) -
  [\mu_i^{(2)}]\indices{_\epsilon^{\gamma\delta}} \partial_{\gamma\delta\alpha} K^{\beta\epsilon}(ij)
\end{align*}

\begin{align*}
  \pder{[\delta
    T^{01}_i]_\alpha}{[\delta\mu_j^{(1)}]\indices{_\beta^\gamma}} =& -
  \delta_i^j p_{k\delta} \partial_{\gamma\alpha} K^{\delta\beta}(jk) +
  p_{i\delta} \partial_{\gamma\alpha} K^{\delta\beta}(ij)
  \\ 
  \pder{[\delta
    T^{11}_i]_\alpha}{[\delta\mu_j^{(1)}]\indices{_\beta^\gamma}} =&
    (\delta_i^j
    [\mu_k^{(1)}]\indices{_\epsilon^\delta} \partial_{\gamma\delta\alpha}K^{\beta\epsilon}(jk)
    + [\mu_i^{(1)}]\indices{_\epsilon^\delta} \partial_{\delta\gamma\alpha}K^{\epsilon\beta}(ij)
  \\ 
  \pder{[\delta
    T^{12}_i]_\alpha}{[\delta\mu_j^{(1)}]\indices{_\beta^\gamma}} =& 
    -\delta_i^j
    [\mu_k^{(2)}]\indices{_\phi^{\epsilon\delta}} \partial_{\gamma\epsilon\delta\alpha}K^{\beta\phi}(jk)
    + [\mu_i^{(2)}]\indices{_\phi^{\epsilon\delta}} \partial_{\gamma\epsilon\delta\alpha}K^{\beta\phi}(ij)
\end{align*}

\begin{align*}
  \pder{[\delta T_i^{02}]_\alpha}{[\delta
    \mu_j^{(2)}]\indices{_\beta^{\gamma\delta}}} =&
    -(\delta_i^j
    [\mu_k^{(2)}]\indices{_\epsilon^{\gamma\delta}} \partial_{\gamma\delta\alpha}K^{\beta\epsilon}(jk)
    +
    [\mu_i^{(2)}]\indices{_\epsilon^{\gamma\delta}} \partial_{\gamma\delta\alpha}K^{\beta\epsilon}(ij) )
\\ 
  \pder{[\delta T_i^{12}]_\alpha}{[\delta
    \mu_j^{(2)}]\indices{_\beta^{\gamma\delta}}} =&
  \delta_i^j [\mu_k^{(1)}]\indices{_\phi^\epsilon}
  \partial_{\epsilon\gamma\delta\alpha} K^{\phi\beta}(jk)
  -
  [\mu_i^{(1)}]\indices{_\phi^\epsilon} \partial_{\epsilon\gamma\delta\alpha} K^{\phi\beta}(ij)
\\ 
  \pder{[\delta T_i^{22}]_\alpha}{[\delta
    \mu_j^{(2)}]\indices{_\beta^{\gamma\delta}}} =& - \delta_i^j
  [\mu_k^{(2)}]\indices{_\zeta^{\epsilon\phi}}
  \partial_{\gamma\phi\epsilon\delta\alpha} K^{\beta\zeta}(jk)
  - [\mu_i^{(2)}]\indices{_\zeta^{\epsilon\phi}}
  \partial_{\epsilon\delta\gamma\phi\alpha} K^{\zeta\beta}(ij)
\end{align*}

\begin{align*}
  \pder{ [\delta \xi_i^{(1)}]\indices{^\alpha_\beta}}{[\delta
    q_j^{(0)}]^\gamma} =&
  \left(p_{k\delta} \partial_{\beta\gamma}K^{\alpha\delta}(jk) -
  [\mu_k^{(1)}]
  \indices{_\epsilon^\delta} \partial_{\beta\gamma\delta}
  K^{\alpha\epsilon}(jk) +
  [\mu_k^{(2)}]\indices{_\phi^{\epsilon\delta}} \partial_{\beta\gamma\delta\epsilon}K^{\alpha\phi}(jk)
  \right) \delta^i_j
  \\&
  -p_{j\delta} \partial_{\beta\gamma}K^{\alpha\delta}(ij) +
  [\mu_j^{(1)}]
  \indices{_\epsilon^\delta} \partial_{\beta\gamma\delta}
  K^{\alpha\epsilon}(ij) -
  [\mu_j^{(2)}]\indices{_\phi^{\epsilon\delta}} \partial_{\beta\gamma\delta\epsilon}K^{\alpha\phi}(ij)
  \\ 
  \pder{ [\delta \xi_i^{(1)}]\indices{^\alpha_\beta}}{[\delta
    q_j^{(1)}]\indices{^\gamma_\delta}} =& 0 \quad,\quad 
  \pder{ [\delta \xi_i^{(1)}]\indices{^\alpha_\beta}}{[\delta
    q_j^{(2)}]\indices{^\gamma_{\delta\epsilon}}} = 0 \quad,\quad 
  \pder{ [\delta \xi_i^{(1)}]\indices{^\alpha_\beta}}{[\delta
    p_j^{(0)}]\indices{_\gamma}}
  = \partial_{\beta}K^{\gamma\alpha}(ij) \\ 
    \pder{ [\delta \xi_i^{(1)}]\indices{^\alpha_\beta}}{[\delta
    \mu_j^{(1)}]\indices{_\gamma^\delta}} =& - \partial_{\delta\beta}K^{\gamma\alpha}(ij) \quad,\quad 
    \pder{ [\delta \xi_i^{(1)}]\indices{^\alpha_\beta}}{[\delta \mu_j^{(2)}]\indices{_\gamma^{\delta\epsilon}}} = \partial_{\epsilon\delta\beta}K^{\gamma\alpha}(ij)
\end{align*}

\begin{align*}
   \pder{[\xi_i^{(2)}]\indices{^\alpha_{\beta\gamma}}}{
     [q_j^{(0)}]^\delta } =& 
     \left[
     [p_k^{(0)}]_\epsilon \partial_{\beta\gamma\delta}K^{\alpha\epsilon}(jk)
     -
     [\mu_k^{(1)}]\indices{_\phi^\epsilon}\partial_{\beta\gamma\delta\epsilon}K^{\alpha\phi}(jk)
     +
     [\mu_k^{(2)}]\indices{_\lambda^{\epsilon\phi}}\partial_{\beta\gamma\delta\epsilon\phi}K^{\alpha\lambda}(jk)
   \right]\delta^i_j
   \\&
     - [p_j^{(0)}]_\epsilon \partial_{\beta\gamma\delta}K^{\alpha\epsilon}(ij)
     +
     [\mu_j^{(1)}]\indices{_\phi^\epsilon}\partial_{\beta\gamma\delta\epsilon}K^{\alpha\phi}(ij)
     -
     [\mu_j^{(2)}]\indices{_\lambda^{\epsilon\phi}}\partial_{\beta\gamma\delta\epsilon\phi}K^{\alpha\lambda}(ij),
   \\ 
   \pder{[\xi_i^{(2)}]\indices{^\alpha_{\beta\gamma}}}{
     [q_j^{(1)}]\indices{^\delta_{\epsilon}} } =& 0
   \quad,\quad 
   \pder{[\xi_i^{(2)}]\indices{^\alpha_{\beta\gamma}}}{
     [q_j^{(2)}]\indices{^\delta_{\epsilon \phi}} } = 0 
   \quad,\quad 
   \pder{[\xi_i^{(2)}]\indices{^\alpha_{\beta\gamma}}}{
     [p_j^{(0)}]\indices{_\delta} } = \partial_{\beta\gamma}K^{\alpha\delta}(ij)
   \\ 
   \pder{[\xi_i^{(2)}]\indices{^\alpha_{\beta\gamma}}}{
     [\mu_j^{(1)}]\indices{_\delta^{\epsilon}} } =& -\partial_{\beta\gamma\epsilon}K^{\alpha\delta}(ij)
   \quad,\quad 
   \pder{[\xi_i^{(2)}]\indices{^\alpha_{\beta\gamma}}}{
     [\mu_j^{(2)}]\indices{_\delta^{\epsilon\phi}} } = \partial_{\beta\gamma\epsilon\phi}K^{\alpha\delta}(ij)
\end{align*}

The adjoint equation are then given by
\begin{align*}
  \frac{d}{dt} [\lambda_{ q_i^{(0)}}]_\alpha =& - [\lambda_{ q_j^{(0)}}]_\beta \pder{ 
     [ \delta \dot{q}_j^{(0)}]^\beta }{  [ \delta q_i^{(0)}]^\alpha }
   - [\lambda_{ q_j^{(1)} } ] \indices{_\beta^\gamma}  \pder{ 
     [ \delta \dot{q}_j^{(1)}] \indices{^\beta_\gamma} }{  [ \delta q_i^{(0)}]^\alpha }
   - [ \lambda_{ q_j^{(2)} } ] \indices{_\beta^{\gamma\delta} }  \pder{ 
     [ \delta \dot{q}_j^{(2)}] \indices{^\beta_{\gamma\delta}} }{  [
     \delta q_i^{(0)}]^\alpha } \\
    &- [ \lambda_{ p_j^{(0)} } ]^\beta \pder{ 
     [ \delta \dot{p}_j^{(0)}]_\beta }{  [ \delta q_i^{(0)}]^\alpha }
   - [\lambda_{ \mu_j^{(1)} } ] \indices{^\beta_\gamma}  \pder{ 
     [ \delta \dot{\mu}_j^{(1)}] \indices{_\beta^\gamma} }{  [ \delta q_i^{(0)}]^\alpha }
   - [\lambda_{\mu_j^{(2)} } ] \indices{^\beta_{\gamma\delta} }  \pder{ 
     [ \delta \dot{\mu}_j^{(2)}] \indices{_\beta^{\gamma\delta}} }{  [ \delta q_i^{(0)}]^\alpha }
\end{align*}

\bibliographystyle{amsalpha}

\begin{thebibliography}{BGBHR11}

\bibitem[BGBHR11]{BruverisHolmRatiu2011}
M.~Bruveris, F.~Gay-Balmaz, D.~D. Holm, and T.~S. Ratiu, \emph{{The momentum
  map representation of images.}}, Journal of Nonlinear Science \textbf{21}
  (2011), no.~1, 115--150.

\bibitem[BMTY05]{Beg2005}
M.~F. Beg, M.~I. Miller, A.~Trouv\'e, and L.~Younes, \emph{Computing large
  deformation metric mappings via geodesic flows of diffeomorphisms},
  International journal of computer vision \textbf{61} (2005), no.~2, 139--157.

\bibitem[BV04]{BoydVandenberghe2004}
Stephen Boyd and Lieven Vandenberghe, \emph{Convex optimization}, Cambridge
  University Press, Cambridge, 2004. \MR{2061575 (2005d:90002)}

\bibitem[CDTM12]{ChertockDuToitMarsden2012}
Alina Chertock, Philip Du~Toit, and Jerrold~Eldon Marsden, \emph{Integration of
  the {EPD}iff equation by particle methods}, ESAIM Math. Model. Numer. Anal.
  \textbf{46} (2012), no.~3, 515--534. \MR{2877363}

\bibitem[CH93]{CamassaHolm1993}
Roberto Camassa and Darryl~D. Holm, \emph{An integrable shallow water equation
  with peaked solitons}, Phys. Rev. Lett. \textbf{71} (1993), no.~11,
  1661--1664. \MR{1234453 (94f:35121)}

\bibitem[CHJM14]{CotterHolmJacobsMeier2014}
C~J Cotter, D~D Holm, H~O Jacobs, and D~M Meier, \emph{A jetlet hierarchy for
  ideal fluid dynamics}, Journal of Physics A: Mathematical and Theoretical
  \textbf{47} (2014), no.~35, 352001.

\bibitem[CRM96]{christensen_deformable_1996}
GE~Christensen, RD~Rabbitt, and MI~Miller, \emph{Deformable templates using
  large deformation kinematics}, Image Processing, IEEE Transactions on
  \textbf{5} (1996), no.~10.

\bibitem[CS96]{ConstantineSavits1996}
G.~M. Constantine and T.~H. Savits, \emph{A multivariate {F}a\`a di {B}runo
  formula with applications}, Trans. Amer. Math. Soc. \textbf{348} (1996),
  no.~2, 503--520. \MR{1325915 (96g:05008)}

\bibitem[DGM98]{DupuisGrenanderMiller1998}
Paul Dupuis, Ulf Grenander, and Michael~I Miller, \emph{Variational problems on
  flows of diffeomorphisms for image matching}, Q. Appl. Math. \textbf{LVI}
  (1998), no.~3, 587--600.

\bibitem[DJR13]{JacobsRatiuDesbrun2013}
M~Desbrun, H~O Jacobs, and T~S Ratiu, \emph{On the coupling between an ideal
  fluid and immersed particles}, Physica D: Nonlinear Phenomena \textbf{265}
  (2013), no.~0, 40--56.

\bibitem[GM83]{GolubitskyMarsden1983}
Martin Golubitsky and Jerrold Marsden, \emph{The morse lemma in infinite
  dimensions via singularity theory}, SIAM journal on mathematical analysis
  \textbf{14} (1983), no.~6, 1037--1044.

\bibitem[HM05]{HolmMarsden2005}
D~D Holm and J~E Marsden, \emph{Momentum maps and measure-valued solutions
  (peakons, filaments, and sheets) for the {EPD}iff equation}, The breadth of
  symplectic and {P}oisson geometry, Progr. Math., vol. 232, Birkh\"auser
  Boston, Boston, MA, 2005, pp.~203--235. \MR{2103008 (2005g:37144)}

\bibitem[HR06]{HoldenRaynaud2006}
Helge Holden and Xavier Raynaud, \emph{A convergent numerical scheme for the
  {C}amassa-{H}olm equation based on multipeakons}, Discrete Contin. Dyn. Syst.
  \textbf{14} (2006), no.~3, 505--523. \MR{2171724 (2006d:35226)}

\bibitem[Jac13]{Jacobs_MFCA_2013}
H~O Jacobs, \emph{Symmetries of higher-order momentum distributions in the
  {LDDMM} formalism}, Mathematical foundations of computational anatomy, Sep
  2013, a MICCAI workshop.

\bibitem[JM00]{JoshiMiller2000}
Sarang~C Joshi and Michael~I Miller, \emph{Landmark matching via large
  deformation diffeomorphisms}, {IEEE} Transactions on Image Processing
  \textbf{9} (2000), no.~8, 1357--1370.

\bibitem[KAA{\etalchar{+}}09]{klein_evaluation_2009}
Arno Klein, Jesper Andersson, Babak~A. Ardekani, John Ashburner, Brian Avants,
  Ming-Chang Chiang, Gary~E. Christensen, D.~Louis Collins, James Gee, Pierre
  Hellier, Joo~Hyun Song, Mark Jenkinson, Claude Lepage, Daniel Rueckert, Paul
  Thompson, Tom Vercauteren, Roger~P. Woods, J.~John Mann, and Ramin~V. Parsey,
  \emph{Evaluation of 14 nonlinear deformation algorithms applied to human
  brain {MRI} registration}, NeuroImage \textbf{46} (2009), no.~3, 786--802
  (en).

\bibitem[KMS93]{KolarMichorSlovak1993}
Ivan Kol{\'a}{\v{r}}, Peter~W. Michor, and Jan Slov{\'a}k, \emph{Natural
  operations in differential geometry}, Springer-Verlag, Berlin, 1993.
  \MR{1202431 (94a:58004)}

\bibitem[KMS99]{KMS99}
I.~{Kol\'{a}\u{r}}, P.~W. {Michor}, and J.~{Slov\'{a}k}, \emph{Natural
  operations in differential geometry}, Springer Verlag, 1999.

\bibitem[MG14]{MicheliGlaunes2014}
M.~{Micheli} and J.~A. {Glaun{\`e}s}, \emph{Matrix-valued kernels for shape
  deformation analysis}, Geometry, Imaging, and Computing \textbf{1} (2014),
  no.~1, 57--39.

\bibitem[MM13]{MumfordMichor2013}
D~Mumford and P~W Michor, \emph{On {E}uler's equation and `{EPDiff}'}, Journal
  of Geometric Mechanics \textbf{5} (2013), no.~3, 319--344, arXiv:1209.6576
  [math.AP].

\bibitem[MR99]{MandS}
J.~E. Marsden and T.~S. Ratiu, \emph{Introduction to mechanics and symmetry},
  2nd ed., Texts in Applied Mathematics, vol.~17, Springer Verlag, 1999.

\bibitem[MW83]{MarsdenWeinstein1983}
J~E Marsden and A~Weinstein, \emph{Coadjoint orbits, vortices, and clebsch
  variables for incompressible fluids}, Physica D: Nonlinear Phenomena
  \textbf{7} (1983), no.~1--3, 305--323.

\bibitem[SDP13]{sotiras_deformable_2013}
A.~Sotiras, C.~Davatzikos, and N.~Paragios, \emph{Deformable {Medical} {Image}
  {Registration}: {A} {Survey}}, IEEE Transactions on Medical Imaging
  \textbf{32} (2013), no.~7, 1153--1190.

\bibitem[SJ15]{sommer_reduction_2015}
Stefan Sommer and H.~O. Jacobs, \emph{Reduction by {Lie} {Groups} {Symmetries}
  in {Diffeomorphic} {Image} {Registration} and {Deformation} {Modelling}},
  Symmetry, in press (2015).

\bibitem[SNDP13]{Sommer2013}
S.~Sommer, M.~Nielsen, S.~Darkner, and X.~Pennec, \emph{Higher-order momentum
  distributions and locally affine lddmm registration}, SIAM Journal on Imaging
  Sciences \textbf{6} (2013), no.~1, 341--367.

\bibitem[Son98]{Sontag1998}
Eduardo~D. Sontag, \emph{Mathematical control theory}, second ed., Texts in
  Applied Mathematics, vol.~6, Springer-Verlag, New York, 1998, Deterministic
  finite-dimensional systems. \MR{1640001 (99k:93001)}

\bibitem[Tro83]{Tromba1983}
A.~J. Tromba, \emph{A sufficient condition for a critical point of a functional
  to be a minimum and its application to {P}lateau's problem}, Math. Ann.
  \textbf{263} (1983), no.~3, 303--312. \MR{704296 (84i:58033)}

\bibitem[Tro95]{trouve_infinite_1995}
Alain Trouvé, \emph{An {Infinite} {Dimensional} {Group} {Approach} for
  {Physics} {Based} {Models} in {Patterns} {Recognition}}, 1995.

\bibitem[TY05]{TrouveYounes2005}
A.~Trouv\'e and L.~Younes, \emph{Local geometry of deformable templates}, SIAM
  J. Math. Anal. \textbf{37} (2005), no.~1, 17--59.

\bibitem[Wei83]{Weinstein1983}
A~Weinstein, \emph{The local structure of {P}oisson manifolds}, J. Differential
  Geom. \textbf{18} (1983), no.~3, 523--557. \MR{723816 (86i:58059)}

\bibitem[You10]{Younes2010}
Laurent Younes, \emph{Shapes and diffeomorphisms}, vol. 171, Springer, 2010.

\end{thebibliography}
\newcommand{\etalchar}[1]{$^{#1}$}
\providecommand{\bysame}{\leavevmode\hbox to3em{\hrulefill}\thinspace}
\providecommand{\MR}{\relax\ifhmode\unskip\space\fi MR }
\providecommand{\MRhref}[2]{%
  \href{http://www.ams.org/mathscinet-getitem?mr=#1}{#2}
}
\providecommand{\href}[2]{#2}

\end{document}